\def\journal@name{}
\documentclass[aos,noinfoline]{imsart}
\setpkgattr{copyright}{text}{}  

\RequirePackage{amsthm,amsmath,amsfonts,amssymb}
\usepackage{bm}
\RequirePackage[numbers]{natbib}

\startlocaldefs
\usepackage{xspace,tikz,graphicx, subcaption, lmodern, multirow, multicol, array, nicefrac,hyperref}
\usepackage[normalem]{ulem}

\usepackage{amsfonts, dsfont, amssymb, amsthm, amsmath,graphicx,xcolor,algorithm,cleveref,algpseudocode,semantic}

\newcommand{\beq}{\begin{equation}}
\newcommand{\eeq}{\end{equation}}

\newcommand{\ignore}[1]{}

\newcommand{\blue}[1]{\color{black}#1\color{black}}

\definecolor{forestgreen}{rgb}{0.0, 0.27, 0.13}

\newcommand\numberthis{\addtocounter{equation}{1}\tag{\theequation}}
\newcommand{\opt}{^\star}
\newcommand{\iid}{\stackrel{\text{\ i.i.d.}}{\sim}}

\newtheorem{thm}{Theorem}
\newtheorem{cor}{Corollary}
\newtheorem{assum}{Assumption}

\newtheorem{prop}{Proposition}
\usepackage{xspace}
\usepackage{tikz}
\usetikzlibrary{positioning, shapes, arrows.meta, calc}

\newtheorem{remark}{Remark}
\renewcommand\b[1]{\boldsymbol{#1}}

\crefname{lem}{lemma}{lemmas}
\Crefname{lem}{Lemma}{Lemmas}
\crefname{prop}{proposition}{propositions}
\Crefname{prop}{Proposition}{Propositions}
\crefname{thm}{theorem}{theorems}
\Crefname{thm}{Theorem}{Theorems}

\mathlig{==}{\equiv}
\mathlig{=.}{\doteq}
\mathlig{:=}{\triangleq}
\mathlig{<<}{\ll}
\mathlig{>>}{\gg}
\mathlig{<>}{\neq}
\mathlig{<=}{\leq}
\mathlig{>=}{\geq}
\mathlig{<==}{\Leftarrow}
\mathlig{==>}{\Rightarrow}
\mathlig{<=>}{\Leftrightarrow}
\mathlig{<==>}{\iff}
\mathlig{<--}{\leftarrow}
\mathlig{-->}{\rightarrow}
\mathlig{<->}{\leftrightarrow}
\mathlig{+-}{\pm}
\mathlig{-+}{\mp}
\mathlig{...}{\dots}
\mathlig{!=}{\stackrel{!}{=}}

\DeclareMathOperator*{\argmax}{\arg\!\max}

\newif\ifshowanswer    

\newcommand{\isitthree}[1]
{
  \ifnum#1=3
    number #1 is 3
  \else
    number #1 is not 3
  \fi
}

\newcommand{\Var}{\mathrm{Var}}

\newcommand{\be}{\begin{equation}}
\newcommand{\ee}{\end{equation}}

\newcommand\diag{\operatorname{diag}}

\newcommand\R{{\mathbb{R}}}
\newcommand\Z{{\mathbb{Z}}}

\renewcommand\P{{\mathds{P}}}
\newcommand\E{{\mathds{E}}}

\newcommand\eps{{\varepsilon}}

\newcommand\tr{\operatorname{tr}}

\renewcommand\b{\boldsymbol}

\newcommand\By{{\mathbf y}}

\newcommand\CN{{\mathcal N}}

\newcommand\CS{{\mathcal S}}

\newtheorem{lemma}{Lemma}
\Crefname{assum}{Assumption}{Assumptions}
\crefname{assum}{Assumption}{Assumptions}
\crefname{proposition}{proposition}{propositions}
\Crefname{proposition}{Proposition}{Propositions}
\crefalias{prop}{proposition}

\newcommand{\jgz}{\mathcal{J}_{+}}
\newcommand{\pto}{\xrightarrow{\text{p}}}
\newcommand{\ip}{\pto}
\newcommand{\pr}{\mathbb{P}}
\newcommand{\stacksvd}{\texorpdfstring{\texttt{Stack-SVD}\xspace}{Stack-SVD}}
\newcommand{\svdstack}{\texorpdfstring{\texttt{SVD-Stack}\xspace}{SVD-Stack}}

\newcommand{\Vstacksvd}{\hat{V}_{\stacksvd}}
\newcommand{\Vsvdstack}{\hat{V}_{\svdstack}}

\definecolor{forestgreen}{RGB}{34,139,34}
\usetikzlibrary{arrows}

\endlocaldefs

\begin{document}
\addtocontents{toc}{\protect\setcounter{tocdepth}{-1}}

\begin{frontmatter}
\title{Stacked SVD or SVD stacked? 
A Random Matrix Theory perspective on data integration}

\runtitle{Stacked SVD or SVD stacked?}

\begin{aug}

\author[A,B]{\fnms{Tavor Z.}~\snm{Baharav}\thanksref{equal}\ead[label=e1]{baharav@broadinstitute.org}\orcid{0000-0001-8924-0243}},
\author[B,C]{\fnms{Phillip B.}~\snm{Nicol}\thanksref{equal}\ead[label=e2]{phillipnicol@g.harvard.edu}\orcid{0000-0002-8526-5889}},
\author[B,C]{\fnms{Rafael A.}~\snm{Irizarry}\ead[label=e3]{rafael\_irizarry@dfci.harvard.edu}\orcid{0000-0002-3944-4309}} \and
\author[A,B,C]{\fnms{Rong}~\snm{Ma}\thanksref{corr}\ead[label=e4]{rongma@hsph.harvard.edu}\orcid{0000-0002-9248-8503}}
\thankstext{equal}{\textbf{Equal contributions, listed alphabetically.}}
\thankstext{corr}{\textbf{Corresponding author.}}

\address[A]{Eric and Wendy Schmidt Center, Broad Institute, Cambridge, MA, 02142, USA\printead[presep={,\ }]{e1}}
\address[B]{Department of Data Science, Dana Farber Cancer Institute, Boston, MA, 02115, USA\printead[presep={,\ }]{e3}}
\address[C]{Department of Biostatistics, Harvard T.H. Chan School of Public Health, Boston, MA, 02115, USA\printead[presep={,\ }]{e2,e4}}
\end{aug}

\begin{abstract}
Modern data analysis increasingly requires identifying shared latent structure across multiple high-dimensional datasets. 
A commonly used model assumes that the data matrices are noisy observations of low-rank matrices with a shared singular subspace. In this case, two primary methods have emerged for estimating this shared structure, which vary in how they integrate information across datasets. 
The first approach, termed \stacksvd, concatenates all the datasets, and then performs a singular value decomposition (SVD).
The second approach, termed \svdstack, first performs an SVD separately for each dataset, then aggregates the top singular vectors across these datasets, and finally computes a consensus amongst them.
While these methods are widely used, they have not been rigorously studied in the proportional asymptotic regime, which is of great practical relevance in today's world of increasing data size and dimensionality.
Consequently, it remains unclear when one method should be preferred over another.
In this work, we derive exact expressions for the asymptotic performance and phase transitions of these two methods and develop optimal weighting schemes to further improve both methods.
Our analysis reveals that while neither method uniformly dominates the other in the unweighted case, optimally weighted \stacksvd dominates optimally weighted \svdstack {when the low rank signal is fully shared across the datasets. We then extend our analysis to handle multiple, partially shared components per dataset and demonstrate that \svdstack can yield improved performance without requiring estimation of subspace alignment.
Finally, we provide practical algorithms for estimating optimal weights from data, offering theoretical guidance for method selection in practical data integration problems. }
Extensive numerical simulations and semi-synthetic experiments on genomic data corroborate our theoretical findings.
\end{abstract}



\end{frontmatter}

\section{Introduction}

Modern biomedical data analyses often aim to integrate high-dimensional datasets obtained from diverse sets of assays or samples.
For example, single-cell RNA-sequencing (scRNA-seq) data measures the expression of $d$ genes across $n$ individual cells \cite{saliba2014single}.
A scRNA-seq study of $M$ patients, where patient $i$ has $n_i$ cells, can be represented by matrices $X_1, \ldots, X_M$, with $X_i \in \R^{n_i \times d}$.
It is often of interest to identify latent structure shared among the $M$ patients as this could reveal novel mechanisms associated with disease or other biological processes \cite{zhang2024recovery}.
This shared structure is commonly modeled by assuming the matrices $X_m$ are noisy observations of low-rank matrices with a shared right singular subspace $V \in \R^{d \times r}$, with $r \ll d$ \cite{argelaguet2020mofa+, shen2012sparse,de2021bayesian,grabski2023bayesian,zhang2023scmomat}.

In the case of a single data matrix $X \in \R^{n \times d}$, the truncated singular value decomposition (SVD) produces $\hat{U} \in \R^{n \times r}$, $\hat{\Theta} \in \R^{r \times r}$ and $\hat{V} \in \R^{d \times r}$, that minimizes the Frobenius norm reconstruction error with $\hat{X} = \hat{U}\hat{\Theta}\hat{V}^\top$ over all rank $r$ matrices \cite{eckart1936approximation}.
It also produces the maximum likelihood estimates for each entry under a model that assumes that $X = U \Theta V^{\top} + E$, where $E \in \R^{n \times d}$ is a matrix of independent Gaussian noise, and $\Theta\in \R^{r \times r}$ is the diagonal matrix of singular values.

In this paper, we investigate and compare through careful theoretical analyses two commonly used extensions of SVD (depicted graphically in \Cref{fig:tikz_graphic_1}) designed to estimate a subspace $V$ shared among multiple noisy data matrices. The first, \stacksvd, begins by vertically stacking the matrices into one large matrix $X_{\text{stack}} \in \R^{(n_1 + \ldots + n_M) \times d}$ and then performs an SVD to obtain the estimate $\hat{V}_\stacksvd \in \R^{d \times r}$.
The second, termed \svdstack, performs these operations in the opposite order.
A truncated SVD is first applied to each matrix individually to obtain $\hat{V}_1, \ldots, \hat{V}_M \in \R^{d \times r}$.
Then, the matrices $\hat{V}_i^{\top}$ are vertically stacked to produce $\tilde{V} \in \R^{Mr \times d}$.
Finally, the leading $r$ right singular vectors of $\tilde{V}$ are taken as the estimate $\hat{V}_\svdstack \in \R^{d \times r}$. 
\svdstack can be interpreted as ``averaging'' the $\hat{V}_i$, since the estimator can be equivalently defined as the $r$ leading eigenvectors of $\hat{V}_1 \hat{V}_1^{\top} + \ldots + \hat{V}_M \hat{V}_M^{\top}$. We note that there exists a broader literature on data integration methods that are not directly related to \stacksvd and \svdstack.
Examples include canonical correlation analysis (CCA) \cite{stuart2019comprehensive}, partial least squares (PLS) \cite{wold1985partial}, generalized SVD (gSVD) \cite{van1976generalizing}, kernel SVD \cite{ding2024kernel,landa2024entropic} (these four being only applicable in the case of $M=2$), Bayesian methods \cite{de2021bayesian, grabski2023bayesian}, and methods based on the product of individual projection matrices \cite{sergazinov2024spectral}.

\begin{figure}
\vspace{-.3cm}
    \centering
\begin{tikzpicture}[
    font=\sffamily,
    node distance=2.0cm,
    >=Latex, 
    align=center  
]

\node[draw,
      rectangle,
      minimum width=2cm,
      minimum height=0.8cm] (data1) {\(X_1\)};
\node[above=0.25cm of data1] (Data) {Data};
\node[draw,
    rectangle,
    below=0.5cm of data1,
    minimum width=2cm,
    minimum height=0.8cm] (data2) {\(X_2\)};
\node[below=.1cm of data2] (dots) {\(\vdots\)};
\node[draw,
    rectangle,
    below=0.2cm of dots,
    minimum width=2cm,
    minimum height=0.8cm] (dataM) {\(X_M\)};

    \node[
          above right=.6cm and 2cm of dataM,
          ] (svdeach1) {\(\hat{V}_1\)};
     \node[
          below=0.4cm of svdeach1,
          ] (svdeach2) {\(\hat{V}_2\)};
     \node[below=-.1cm of svdeach2] (dots2) {\(\vdots\)};
     \node[
          below=0cm of dots2,
          ] (svdeachM) {\(\hat{V}_M\)};
      
    \node[draw, rectangle,
        right=1cm of svdeach2,
        text width=2cm,
        minimum height=1.2cm] (vstack) 
        {$\hat{V}_1^\top$ \\ \(\vdots\) \\ $\hat{V}_M^\top$};

\node[rectangle,
      right=1.1cm of vstack,
      text width=1.4cm,
      minimum height=1.2cm] (vhatBot) {\(\hat{V}_{\svdstack}\)};

\node[draw, rectangle,
    above right=-1cm and 2.25cm of data1,
    text width=2cm,
    minimum height=1.2cm] (stack) {$X_1$ \\ \(\vdots\) \\ $X_M$};
\node[rectangle,  text width=1.4cm,
      minimum height=1.2cm] (vhatTop) at (vhatBot |- stack) {\(\hat{V}_{\stacksvd}\)};

\draw[->] (data1.east) -- ++(0.2,0) -- (stack.west);
\draw[->] (data2.east) -- ++(0.5,0) -- (stack.west);
\draw[->] (dataM.east) -- ++(0.8,0) -- ++(0,1.5) -- (stack.west);
\draw[->] (stack.east) -- (vhatTop.west) node[midway, above] {SVD};

\draw[->] (data1.east) -- ++(0.2,0) |- (svdeach1.west) node[pos=0.8, above] {SVD};
\draw[->] (data2.east) -- ++(0.5,0) |- (svdeach2.west) node[pos=0.8, above] {SVD};
\draw[->] (dataM.east) -- ++(0.8,0) |- (svdeachM.west) node[pos=0.8, above] {SVD};

\node[right=.5cm of svdeach2] (dummy_node) {};
\draw[-] (svdeach1.east) -| (dummy_node.west);
\draw[-] (svdeach2.east) -- (dummy_node.west);
\draw[-] (svdeachM.east) -| (dummy_node.west);
\draw[->] (dummy_node.west) -- (vstack.west);
\draw[->] (vstack.east) -- (vhatBot.west) node[midway, above] {SVD};

\node[above=3pt of stack] {$X_{\textnormal{stack}}$};
\node[above=3pt of vstack] {$\tilde{V}$};

\end{tikzpicture}
\vspace{-.2cm}
    \caption{Schematic showing the methods \stacksvd and \svdstack.}
    \label{fig:tikz_graphic_1}
    \vspace{-.3cm}
\end{figure}

Our focus on \stacksvd and \svdstack is motivated by their broad practical relevance and growing impact in modern data science.
Despite widespread use across diverse applications, the theoretical properties of these approaches, and variations thereof, have only recently begun to receive systematic attention.
For example, methods based on \stacksvd have been applied to integrate multi-patient single-cell RNA-seq data \cite{hie2019efficient}, multimodal genomic data \cite{shen2012sparse, argelaguet2020mofa+, hie2019efficient}, as well as electronic health record (EHR) data \cite{gan2025arch}.
On the theoretical front, the minimax rate-optimality of \stacksvd under the aforementioned model was recently established \cite{ma2024optimal}, provided that the data matrices are of comparable sizes or their signal strengths are sufficiently large. 

Similarly, methods based on \svdstack have been used across varied applications, including federated unsupervised learning \cite{shi2024personalized}, distributed unsupervised learning \cite{fan2019distributed,zheng2022limit,bhaskara2019distributed,jou2024generalized,he2024distributed}, ensemble learning \cite{ma2023spectral,danning2025lace},  integrative single-cell genomics \cite{ozbay2023navigating, zhang2024recovery}, and reference-free sample clustering in genomics \cite{chaung2023splash,baharav2024oasis}.
Several statistical methods are equivalent or closely related to \svdstack such as angle-based joint and individual variation explained (AJIVE) \cite{yang2025estimating}, distributed PCA \cite{fan2019distributed}, and the common subspace independent edge (COSIE) model \cite{arroyo2021inference}, which also considers weighting higher signal matrices.
Regarding \svdstack, existing theoretical analyses generally focus on obtaining finite sample bounds of performance as well as rates of convergence.
For example, \cite{zheng2022limit} studied the limiting distribution of estimators of this form under the assumption of equal sample sizes across all matrices;
\cite{yang2025estimating} extended these results to the setting of AJIVE, which permits individual-specific components to be present in the data matrices under certain assumptions, yet its minimax optimality result was established in the setting with identical sample sizes.

In line with prior studies, a deeper theoretical understanding of \stacksvd and \svdstack and their tradeoffs would help researchers more effectively leverage information across different datasets. In modern applications both $n_i$ and $d$ are typically large, and $n_i$ may vary substantially across data matrices; for example, scRNA-seq typically measures tens of thousands of genes across thousands to tens of thousands of cells \cite{hie2019efficient,ma2024principled}.
It is thus reasonable and of practical interest to study the performance of \stacksvd and \svdstack under the \textit{proportional} regime where $n_i \to \infty$, $d \to \infty$ in such a way that $n_i/d \to c_i \in (0, \infty)$ for possibly distinct $c_i$.
In this regime, we show that the exact asymptotic performance of these methods can be characterized and compared in a \emph{pointwise} manner. In contrast, the minimax lower and upper bounds obtained by prior works \cite{fan2019distributed,ma2024optimal,yang2025estimating,zheng2022limit}, based on non-asymptotic analyses, mostly contrast different methods in terms of their worst-case performance.
This offers valuable theoretical insights, but limited  guidance on method selection for specific practical scenarios.
While a large body of work has studied the spectral properties of a single data matrix under the proportional regime within the Random Matrix Theory framework \cite{tao2012topics, livan2018introduction}, and used them in genomic contexts \cite{price2006principal,chardes2025random}, the behavior of \svdstack, \stacksvd, and their variations under this setting remains largely unexplored.
The only related work we are aware of is \citep{hong2023optimally}, which analyzes optimal weighted PCA under a heteroscedastic factor model, a setting that resembles a weighted version of \stacksvd. To our knowledge, no analogous theoretical results exist for \svdstack.

This paper provides theoretical insights into these two families of data integration methods through a systematic analysis using recent advances in Random Matrix Theory.
Our major contributions are summarized in the following, with additional details in \Cref{tab:results_summary} below.

\begin{itemize}
    \item We derive the limit of $\| \hat{V}^{\top} V\|_F$ for both \stacksvd and \svdstack under a joint signal-plus-noise model where $n_i / d \to c_i$. We establish convergence in probability to the limiting forms, which we term the \emph{asymptotic performance} of a method. 
    Our framework provides a precise asymptotic characterization of these methods under possibly unbalanced sample sizes and varying signal strengths across data matrices, allowing for general noise distributions.
    Our analyses reveal phase transitions in performance as a function of the per matrix signal strength, and its interaction with the dimensionality and sample sizes. 
    \item To better account for the varying signal strengths across different data matrices, we consider \textit{weighted} variants of \stacksvd and \svdstack where each matrix is scaled prior to performing the SVD.
    We derive the asymptotically optimal weightings, along with the associated performance limits and phase transition behavior for both weighted \stacksvd and weighted \svdstack.
    \item {We characterize the relationships between the asymptotic performances of these different methods in a pointwise sense (see \Cref{tab:tikz_comparison_summary} for the rank 1 fully aligned case), offering a more nuanced perspective that complements prior minimax analyses \cite{yang2025estimating,ma2024optimal,shi2024personalized}. We show that without weights, either \svdstack or \stacksvd can outperform the other, depending on the problem instance. In the weighted case, however, we show that optimally weighted \stacksvd uniformly outperforms weighted \svdstack. We also show that even optimally binary-weighted \stacksvd can be far from optimal, and that there exists a sequence of instances where optimally weighted \stacksvd has performance going to 1, but optimally binary-weighted \stacksvd, optimally weighted \svdstack, unweighted \stacksvd, and unweighted \svdstack all have a performance of 0, failing to obtain nontrivial performance.}
    \item {Although our primary motivation considers the singular subspaces to be exactly shared between datasets, in many practical settings it is reasonable to instead expect \textit{partial} alignment. 
    In \Cref{sec:general_rankr_unaligned}, we extend our asymptotic analysis to a more general model in which the right singular subspace of each matrix may differ but is contained within a common low-dimensional ambient subspace. From this general setting, a more nuanced comparison between \svdstack and \stacksvd emerges. 
    In particular, we demonstrate that \svdstack generally outperforms \stacksvd when the subspaces are strongly misaligned. Moreover, \svdstack has the surprising property that its optimal weighting depends only on the (estimable) signal strengths and does not depend on the degree of misalignment between subspaces.}
\end{itemize}
{To establish these statistical insights, we develop novel theoretical techniques that are broadly applicable. As \svdstack is a two-stage spectral estimator, its performance does not immediately follow from existing RMT results. 
In particular, the analysis of  \svdstack  draws on recent results concerning the {delocalization} of eigenvector residuals for matrices under general noise conditions. Our analysis of \stacksvd  builds on recent advances in studying spiked eigenvalue problems with heteroscedastic noise. See \Cref{app:weightedSVDStackProof,app:weightedStackSVDProof,sec:unaligned_rank_r} for more details.
We additionally provide a formalization of our main results in Lean 4 \cite{moura2021lean} on GitHub \url{https://github.com/phillipnicol/stackedSVD}.
In order to do so, we prove in Lean 4 the random matrix theory foundations our results rely on, which may be of independent utility to those looking to formalize RMT results.
}

{
The rest of this paper is structured as follows.
In \Cref{sec:model_setup} we introduce the joint signal-plus-noise model and formally define the (weighted) \stacksvd and \svdstack estimators.
In \Cref{sec:unweighted_analysis} we compute the performance for \svdstack and \stacksvd in the setting with rank one signal, establishing convergence in probability to the stated limit.
In \Cref{sec:optimal_weighting} we analyze the weighted estimators and derive the asymptotically optimal weightings and their associated performances.
We then compare the relative performances of our estimators in \Cref{sec:relative_performance}, and identify the specific settings in which each method exhibits superior performance. 
In \Cref{sec:experiments} we perform an extensive synthetic simulation study that corroborates our theoretical findings.
\Cref{sec:general_rankr_unaligned} generalizes our findings to arbitrary finite signal rank, leveraging the rank one results.
We develop a guide for practitioners in \Cref{sec:practitioners_guide} to enable the practical usage of our theoretical results.
\Cref{sec:single_cell} demonstrates the practical utility of the optimally weighted integration methods for single-cell data integration on a semi-synthetic scRNA-seq dataset. 
Finally, we conclude in \Cref{sec:discussion_future_work}, and discuss extensions and possible future work.}

{\everymath{\displaystyle}
\begin{table}[t]
\centering
\renewcommand{\arraystretch}{1.5}  
\begin{tabular}{|>{\centering\arraybackslash}m{1.5cm}|>{\centering\arraybackslash}m{1.8cm}|>{\centering\arraybackslash}m{2.5cm}|>{\centering\arraybackslash}m{6.8cm}|}
\hline
\textbf{Method} & \textbf{Component} & \textbf{Unweighted} & \textbf{Weighted} \\ \hline
\multirow{4}{*}[-6ex]{\centering \textbf{\stacksvd}}  
& Result & \Cref{prop:stacksvd_general} & \Cref{thm:stacksvd_weighted} \\ \cline{2-4}
& Weights & N/A & \vspace{5pt}$w_i\opt = \frac{\theta_i}{\sqrt{\theta_i^2 + c_i}}$ \vspace{5pt}\\ \cline{2-4}
& Detectability Threshold
& $\|\theta\|_2^4/ \sum_i c_i  > 1$ 
& $\sum_i \theta_i^4/c_i > 1$ \\ \cline{2-4}
& Performance & \vspace{5pt} $\frac{\|\theta\|_2^4-\sum_i c_i}{ \|\theta\|_2^2(\|\theta\|_2^2 + 1)}$ \vspace{5pt} &\vspace{5pt} $\displaystyle x\opt \in (0,1) \text{ s.t. }\sum_{i} \theta_i^4 \frac{1-x\opt}{c_i + \theta_i^{2} x\opt} = 1$ \vspace{5pt}\\ 
\hline
\multirow{4}{*}[-6ex]{\textbf{\svdstack}}  
& Result & \Cref{thm:svd_stack_general} & \Cref{thm:svdstack_weighted} \\ \cline{2-4}
& Weights & N/A & \vspace{5pt} $w_i\opt = 
\theta_i\sqrt{\frac{\theta_i^2+1}{\theta_i^2 + c_i}} \mathds{1} \left\{ \theta_i^4 > c_i\right\}
$ \vspace{5pt}\\ \cline{2-4}
& Detectability Threshold & $\beta_2 > 0$ & $\beta_1>0 \quad \iff \quad \max_i \, \theta_i^4/c_i > 1$ \\ \cline{2-4}
& Performance & \vspace{5pt} $\frac{\left( \beta^\top v_{\text{max}}\left(A_\beta\right) \right)^2}{\lambda_\text{max}\left(A_\beta\right)}$ \vspace{1pt}& \vspace{5pt}$\frac{S}{S+1}=1-\left(1+\sum_i \frac{\theta_i^4 - c_i}{\theta_i^2 + c_i} \mathds{1}\left\{ \theta_i^4 > c_i \right\} \right)^{-1}$ \vspace{5pt}\\ \hline
\end{tabular}
\vspace{0.2cm}  
\caption{Results for \stacksvd and \svdstack specialized to the rank-one setting where $X_i = \theta_i u_i v^\top + E_i$, under noise conditions (\Cref{assum:general_noise}) and RMT scaling (\Cref{assum:main}).
$\beta_i$ is defined in \Cref{prop:single_table}, $A_\beta$ in \eqref{eq:A_beta_main_text}, and $S$ in \Cref{thm:svdstack_weighted}.
$\theta$ and $\beta$ denote vectors with entries $\theta_i$ and $\beta_i$ resp.
}
\label{tab:results_summary}
\end{table}
}

\section{Problem formulation} \label{sec:model_setup}
{
Before detailing our problem formulation, we begin by setting up some helpful notation.
We define the set $[M] := \{1,2,\hdots,M\}$.
Throughout, we use ``table'' and ``data matrix'' interchangeably.
The unit sphere in $d$ dimensions is defined as $S^{d-1}$.
We use the standard notation that $v_{\text{max}}(X)$ is a principal right singular vector of the matrix $X$.
{The $i$-th largest singular value and eigenvalue of a matrix $X$ are denoted by $\sigma_i(X)$ and $\lambda_i(X)$ respectively}, and $\sigma_{\text{max}}(X) := \sigma_1(X)$ and $\lambda_\text{max}(X) := \lambda_1(X)$.
We use $\mathds{1}$ to denote the indicator function.
We use standard order-statistic notation, where $x_{(i)}$ denotes the $i$-th largest element of a vector $x$.
The $\ell_p$ norm of a vector $x$ is given by $\|x\|_p$, with $\|x\|=\|x\|_2$ if not otherwise specified.
We use $\pto$ to denote convergence in probability of a sequence of random variables. $\mathbb{O}(a,b)$ denotes the set of $a \times b$ matrices with orthonormal columns.

\subsection{Model setup}

We assume that each data matrix $X_i \in \R^{n_i \times d}$ is a noisy observation of a low-rank ($r_i \ll d$) signal matrix. That is,
\begin{equation} \label{eq:model_rankr}
    X_i = U_i \Theta_i V_i^{\top} + E_i,
\end{equation}
where our signal matrix has singular value decomposition $U_i \in \mathbb{O}(n_i, r_i)$, $\Theta_i  = \text{diag}(\theta_{i1}, \ldots, \theta_{ir_i}) \in \R^{r_i \times r_i}_{\ge 0}$, and $V_i \in \mathbb{O}(d, r_i)$, with noise matrix $E_i$.
Throughout this paper, we will assume that entries of $E_i \in \R^{n_i \times d}$ are independent and identically distributed, with $\sqrt{d} z_{i,j,k}$ having 0 mean, variance $1$, and bounded fourth moment.

\begin{assum} \label{assum:general_noise}
    We assume that the noise matrices $E_i \in \R^{n_i \times d}$ have independent and identically distributed entries $z_{i,j,k}$ for $i \in [M], j \in [n_i]$, $k \in [d]$ such that for some constant $C$: 
    \begin{equation*}
        \E z_{i,j,k} = 0 \quad , \quad 
        \E \left(\sqrt{d} z_{i,j,k}\right)^2 = 1 \quad , \quad 
        \E \left(\sqrt{d} z_{i,j,k}\right)^4 \le C.
    \end{equation*}
\end{assum}

In this work, we consider several different models of \textbf{shared} structure between the singular subspaces $V_i$, relating to a global subspace $V$.
In Sections 3-6, we consider the \textbf{exactly aligned} scenario, where $r_i = r$ and $V_i = V \in \mathbb{O}(d, r)$. 
Later, in \Cref{sec:general_rankr_unaligned} we extend our results to the \textbf{partially aligned} setting, where $V_i=V R_i$ for $R_i \in \mathbb{O}(r,r_i)$.

We study the performance of different estimators of $V$ under the proportional asymptotic regime: 
\begin{equation} \label{eq:RMT_limit}
    n_1,\ldots,n_M,d \to \infty  \quad \text{ and } \quad n_i/d \to c_i \in (0,\infty) \ \forall i.
\end{equation}
For ease of exposition, we present the main results for the exactly aligned scenario in the rank $1$ setting $(r = 1)$, and extend these results to $r > 1$ in \Cref{sec:rank_r}. In this case, Equation \eqref{eq:model_rankr} may be simplified as 
\begin{equation}\label{eq:model_rank1}
    X_i = \theta_i u_i v^{\top} + E_i.
\end{equation}
In this rank one setting, we will write $\theta = (\theta_1, \ldots, \theta_M) \in \R^M$ to denote the vector of signal strengths and $c = (c_1, \ldots, c_M) \in \R^M$ to denote the vector of aspect ratios. 

\begin{assum}\label{assum:main}
   The matrices $\{X_i\}$ are generated according to \Cref{eq:model_rank1}, and satisfy the scaling limits in \Cref{eq:RMT_limit}.
\end{assum}

Under these assumptions, the behavior of the leading eigenvector/ singular vector has been characterized by existing work \cite{Baik2005, paul2007asymptotics, benaych2012singular, liu2023asymptotic, johnstone2018pca,10.3150/19-BEJ1129}. We summarize a key result about $\hat{v}_i := v_{\text{max}}(X_i)$ below, specializing Theorem 1 of \cite{liu2023asymptotic} to our setting.

\begin{prop}\label{prop:single_table}
    Under \Cref{assum:general_noise,assum:main}, we have
    \begin{equation*}
        |\langle \hat{v}_i, v \rangle|^2 \ip \beta_i^2 :=  \begin{cases}
        \frac{\theta_i^4 - c_i}{\theta_i^4 + \theta_i^2} & \text{ if } \theta_i \geq c_i^{1/4}, \\
        0 & \text{ otherwise.}
        \end{cases}
    \end{equation*}

\noindent Furthermore, for any deterministic sequence of unit vectors $w^{(d)}$ orthogonal to $v$:
    \begin{equation*}
        |\langle \hat{v}_i, w^{(d)} \rangle|^2 \ip 0.
    \end{equation*}
\end{prop}

From the above proposition, for each data matrix $X_i$, the asymptotic performance $\beta_i^2$ of $\hat v_i$ undergoes a phase transition as the signal strength $\theta_i$ increases. Specifically, $c_i^{1/4}$ serves as the \textit{detectability threshold} for $\theta_i$: when $\theta_i< c_i^{1/4}$, $\hat v_i$ is asymptotically orthogonal to $v$. The second result states that $\hat{v}_i$ will have a limiting inner product of $0$ with any fixed vector orthogonal to $v$.

\subsection{Estimators}

As discussed, two classes of methods have been popularly studied in this setting (\Cref{fig:tikz_graphic_1}).
We now formally define them for the exactly aligned scenario.
The first, termed \stacksvd, concatenates all the matrices, and performs an SVD on the concatenated matrix $X_\text{stack} \in \R^{(n_1+...+n_M)\times d}$.
This method can optionally be weighted during the integration step, as is defined below for a weight vector $w \in \R^M_{\ge 0}$:
\begin{equation}\label{stacksvd.def}
    \stacksvd(X_1, \ldots, X_M ; w) = v_{\text{max}} \left( \begin{bmatrix} w_1X_1 \\ \vdots \\ w_M X_M \end{bmatrix} \right) := \hat{v}_\stacksvd(w).
\end{equation}

The second method, termed \svdstack, performs an SVD on each matrix separately to compute $\hat{v}_i = v_{\text{max}}(X_i)$.
It then stacks these singular vectors into a matrix $\tilde{V}\in\R^{M\times d}$, and performs an SVD to obtain $\hat{v}_\svdstack$.
This method can also optionally be weighted:
\begin{equation}\label{svdstack.def}
    \svdstack(X_1, \ldots, X_M ; w) = v_{\text{max}} \left( \begin{bmatrix} w_1\hat{v}_1^\top \\ \vdots \\ w_M \hat{v}_M^\top \end{bmatrix} \right) := \hat{v}_\svdstack(w).
\end{equation}

Unless otherwise specified, we assume the unweighted versions of each method, i.e. $w = (1, \ldots, 1) \in \R^M$. Furthermore, we will refer to the estimators simply as $\hat{v}_\stacksvd$ and $\hat{v}_\svdstack$ when the choice of $w$ is clear from context.
}

\section{Analysis of unweighted \stacksvd and \svdstack} \label{sec:unweighted_analysis}

To build intuition, we begin by analyzing the performance of the two methods in the simple setting where the matrices have the same signal strength and the same asymptotic size.

\begin{cor}\label{thm:simple_thm1}
    Suppose $c_i=c_0$ and $\theta_i=\theta_0$ for all $i$.
    Under \Cref{assum:general_noise,assum:main}, the performance of \stacksvd and \svdstack is given by:

\begin{align*}
    | \langle \hat{v}_\svdstack, v \rangle |^2 &\pto \begin{cases}
              1 - \frac{c_0 +  \theta_0^2}{M \theta_0^4 + \theta_0^2 - (M-1)c_0} & \text{ if } \theta_0 \geq c_0^{1/4}, \\
              0 & \text{ otherwise.}
             \end{cases}\\
    | \langle \hat{v}_\stacksvd, v \rangle |^2 &\pto \begin{cases}
                1 - \frac{c_0 +  \theta_0^2}{M \theta_0^4 + \theta_0^2} & \text{ if } \theta_0 \geq M^{-1/4} c_0^{1/4}, \\
                0 & \text{ otherwise.}
            \end{cases}
\end{align*}
\end{cor}

These results are special cases of \Cref{thm:svd_stack_general,prop:stacksvd_general} respectively, the general results we provide in the next section.
Both methods display the expected consistency: the performance of each approaches $1$ as $\theta_0$ or $M$ increase.
Considering fixed $M>1$, both methods have a performance of 0 below \stacksvd's recovery threshold of $\theta_0 > M^{-1/4} c_0^{1/4}$.
However, above this threshold, \stacksvd performs strictly better.
Note that for $\theta_0 \in (M^{-1/4} c_0^{1/4}, c_0^{1/4})$, \svdstack continues to fall below the recovery threshold, while \stacksvd attains positive performance (see \Cref{fig:cor_1_plot}).
This is valid for every $M\ge 1$, and the two methods agree at the threshold.

\subsection{Unweighted \svdstack}
We now consider the case where $\theta_i$ can vary across matrices. 
Without loss of generality, we assume that our per table performances $\beta_i$ (defined in \Cref{prop:single_table}) are in sorted order, i.e. \ $\beta_1 \geq \ldots \geq \beta_M$, enabling us to state our detectability condition as $\beta_2>0$.

To study \svdstack, in addition to the inner product between $\hat{v}_i$ and the ground truth $v$, we need to characterize the behavior of $\langle \hat{v}_i,\hat{v}_j\rangle$.
If the direction of the component of $\hat{v}_i$ orthogonal to $v$ were drawn uniformly at random from the unit sphere, independently across $i$, then $|\langle \hat{v}_i,\hat{v}_j\rangle|^2 \pto \beta_i^2\beta_j^2$. This result holds, for example, in the special case of i.i.d. Gaussian noise, as $(I - vv^{\top}) \hat{v}_i$ is essentially Haar distributed (after proper normalization, see Lemma 4.4 of \citep{loffler2021optimality}). In \Cref{lem:delocalization}, we leverage the results presented in \Cref{prop:single_table} to prove that this delocalization result also holds under our more general noise assumptions. In particular, this means that the estimated singular vectors from independent matrices are asymptotically uncorrelated except through their shared alignment with the true signal $v$.

Our analysis reduces the study of the $M \times d$ matrix $\tilde{V}$ to the $M \times M$ matrix $\tilde{V}\tilde{V}^\top \pto A_\beta$ (\Cref{lem:delocalization}).
In particular, with $\beta=(\beta_1,...,\beta_M)$:
\begin{align*} 
    A_\beta &:= \beta \beta^{\top} + \text{diag}(1-\beta_1^2, \ldots, 1-\beta_M^2) \\
    &= \begin{bmatrix}
        1 & \beta_1\beta_2 & \ldots\\
        \beta_1 \beta_2 & 1 & \ldots\\
        \vdots & \vdots & \ddots
    \end{bmatrix} \numberthis \label{eq:A_beta_main_text}.
\end{align*}

\begin{prop}
\label{thm:svd_stack_general}
    Under \Cref{assum:general_noise,assum:main}, when $\beta_2 > 0$, \svdstack satisfies:
    \begin{equation*}
        |\langle \hat{v}_\svdstack, v \rangle|^2 \pto 
        \frac{\left(\beta^\top v_{\text{max}}(A_\beta)\right)^2}{\lambda_\text{max}\left(A_\beta\right)}.
    \end{equation*}
    When $\beta_1 = 0$, this inner product converges in probability to 0.
\end{prop}

We defer the proof of this result to \Cref{app:SVD_stack_proofs}. The requirement that $\beta_2 > 0$ arises as a necessary and sufficient condition to ensure the largest eigenvalue of $A_{\beta}$ is unique when $M>1$.
See \Cref{sec:svdstack_threshold} for additional discussion of the case where $\beta_1>0$ but $\beta_2=0$ yielding $A_\beta = I$.
To interpret our performance guarantee, note that:
\begin{align*}
    \max\left(1,\|\beta\|^2\right) \le \lambda_\text{max}\left(A_\beta\right) &\le \|\beta\|^2 + 1 - \min_i \beta_i^2, \\
     \left(\beta^\top v_{\text{max}}(A_\beta)\right)^2 &\le \|\beta\|^2.
\end{align*}
When all $\beta_i$ are equal, $v_{\text{max}}$ will be directly proportional to the all ones vector, and so the upper bounds for both $\left(\beta^\top v_{\text{max}}(A_\beta)\right)^2$ and $\lambda_\text{max}$ will be attained, recovering the result in \Cref{thm:simple_thm1}.
On the other extreme, if only two $\beta_i$ are nonzero and are equal to $\beta_0$, then $v_{\text{max}}(A_\beta)$ will simply have $1/\sqrt{2}$ on its first two entries and zero otherwise. This yields $\lambda_\text{max}(A_{\beta})=1+ \beta_0^2$ in the denominator, and $2\beta_0^2$ in the numerator, with an overall performance of $2\beta_0^2/ (1+\beta_0^2)$, which is independent of $M$.
As expected, \svdstack's performance can be expressed solely as a function of the $\beta_i$, with no additional dependence on $\theta$ or $c$.

\subsection{\stacksvd unweighted analysis}

Extending the analysis of \stacksvd to the case with nonuniform signal strength is comparatively straightforward, as the concatenated matrix with equal noise variances is now itself simply a matrix that can be easily studied through the Random Matrix Theory framework.

\begin{prop}
\label{prop:stacksvd_general}
    Under \Cref{assum:general_noise,assum:main}, the performance of \stacksvd is:
    \begin{equation*}
        | \langle \hat{v}_\stacksvd, v \rangle |^2 \pto \begin{cases}
            \frac{\|\theta\|_2^4-\|c\|_1}{ \|\theta\|_2^2(\|\theta\|_2^2 + 1)} & \text{ if } \|\theta\|_2^4 > \|c\|_1, \\
                   0 & \text{otherwise.}
               \end{cases}
    \end{equation*}
\end{prop}

The proof of this proposition essentially follows that of \cite{liu2023asymptotic} concerning the asymptotic behavior of spiked eigenvectors of signal-plus-noise matrices with heteroscedastic noise; see Appendix \ref{app:weightedStackSVDProof} for details.
Observe that the performance of \stacksvd depends on the vectors $\theta$ and $c$ only through their norms ($\|\theta\|_2$ and $\|c\|_1$).
This exchangeability is due to the lack of weighting, and can lead to undesirable edge cases.
While this approach draws strength from tables with $\theta_i$ below the threshold of detectability, it can also be corrupted by large $c_i$ from a single table.
Note that \Cref{prop:stacksvd_general} holds even when $M\to\infty$, as long as 
the limiting quantities $\bar{c} = \|c\|_1$ and $\bar{\theta}^2=\|\theta\|_2^2$ exist and are finite:
\begin{equation*}
    \bar{c}=\lim_{M,n_m,d\to\infty}
{\sum_{m=1}^M n_m}/d \quad \text{and} \quad
\bar{\theta}^2 =\lim_{M\to\infty} \sum_{m=1}^M\theta_m^2.
\end{equation*}
This highlights why $\bar{c} = \sum_i c_i$, as the number of rows sum linearly, while the signal strength $\theta$ arises in $X^\top X$ and so the squared singular value of the stacked signal matrix is $\bar{\theta}^2=\sum_i \theta_i^2$.
While \svdstack essentially takes the best of what each table has to offer, \stacksvd instead depends on an ``average'' quantity across the tables.
For example, if two tables have large $\theta_i$, as more low SNR tables with large $c_i$ and small $\theta_i$ are included, the performance of \svdstack will stay the same while \stacksvd suffers and can fall below the recovery threshold.
We elaborate in \Cref{remark:svd_outperform_stack} below, and show this effect in simulations in \Cref{fig:fig3_interesting}b.
This is because the spectrum of $A_\beta$ will not change as uninformative $X_i$ are added, simply adding a new row and column with all zeros except for the diagonal entry.
Thus, $v_{\text{max}}$ will not change (except adding a 0 at the end), and $\lambda_\text{max}$ will be unaffected.
Essentially, high signal matrices can be diluted by low signal ones for \stacksvd, which is not true for \svdstack.
However, this can be avoided by a simple binary weighting scheme, discussed in the following subsection.

\subsubsection{Binary-weighted \stacksvd}
As discussed \stacksvd can be very brittle: if one table has a very large $c_i$ and relatively small $\theta_i$, it can singlehandedly drive the performance of \stacksvd to 0.
To address this, practical implementations of \stacksvd often use a binary weighting scheme, where tables that are judged to be low signal are given a weight of 0 and discarded from the analysis \cite{ma2024optimal}.
This is a simple yet effective way to improve the performance of \stacksvd, and we provide a brief result characterizing the performance of this method.

\begin{cor}\label{cor.2}
    Under \Cref{assum:general_noise,assum:main} the performance of binary-weighted \stacksvd, where tables with $\theta_i^4 \le c_i$ are discarded, is given by (assuming at least one table is above the threshold of detectability):
    \begin{align*}
        | \langle \hat{v}_\stacksvd, v \rangle |^2 \pto \frac{\left(\sum_{i \in \CS} \theta_i^2\right)^2 - \sum_{i \in \CS} c_i}{\left(\sum_{i \in \CS} \theta_i^2\right)^2 + \sum_{i \in \CS} \theta_i^2},\quad\text{where}\quad \CS := \{ i : \theta_i^4 > c_i \}.
    \end{align*}
\end{cor}

Here the threshold used is $\theta_i^4 > c_i$, which is simply the threshold of detection for this table marginally, as indicated in Proposition \ref{prop:single_table}.
One can consider other thresholds on $\theta_i^4/c_i$, or more generally consider any subset of the $M$ tables, in which case this method will trivially always perform at least as well as unweighted \stacksvd.
    Optimizing over all possible subsets $\CS$ to generate $\hat{v}_\stacksvd$,  we obtain:
    \begin{equation*}
        | \langle \hat{v}_\stacksvd, v \rangle |^2 \pto \begin{cases}
            \underset{\CS \in 2^{[M]}\setminus \emptyset}{\max}
 \frac{\left(\sum_{i \in \CS} \theta_i^2\right)^2 - \sum_{i \in \CS} c_i}{\left(\sum_{i \in \CS} \theta_i^2\right)^2 + \sum_{i \in \CS} \theta_i^2} & \text{ if } \underset{\CS \in 2^{[M]} \setminus \emptyset}{\max} \left(\sum_{i \in \CS} \theta_i^2\right)^2 -\sum_{i \in \CS} c_i > 0,\\
            0 & \text{ otherwise.}
        \end{cases}
    \end{equation*}
Note that this improves not only the performance of \stacksvd, but also the detectability threshold.
As we show in Section \ref{sec:relative_performance}, this simple weighting scheme of only selecting tables where $\theta_i^4 > c_i$ outperforms \svdstack when all $c_i\le 1$ in \Cref{thm:stacksvd_binary_optimal_svd_stack}, i.e. when none of the tables are excessively ``tall''.

\section{Optimally weighted estimators} \label{sec:optimal_weighting}
By naively aggregating the information across matrices, we lose a key degree of freedom, the weighting of the different information sources, resulting in limited integration performance.
If different matrices have different signal strengths, a nonuniform weighting $w \in \R^M$ (\Cref{stacksvd.def,svdstack.def}) can dramatically improve performance.
In this section, we derive weights $w$ for both \svdstack and \stacksvd that maximize the limiting value of $|\langle \hat{v}(w), v \rangle|^2$.
We will refer to such a $w$ as an \textit{optimal weighting}.
Throughout, we assume $\{\theta_i\}$ are known, and discuss how to consistently estimate them in \Cref{sec:estimating_theta} (\Cref{thm:theta_est}).

\subsection{\stacksvd weighted analysis}\label{sec:main_text_stacksvd_weighted}
Weighted data integration methods have received renewed attention with the large numbers of datasets available today, including \cite{hong2023optimally} for weighted PCA, \cite{hong2025optimal} for sample allocation in weighted PCA, and \cite{liu2023asymptotic} for general low-rank and structured covariance matrices.
We prove the following theorem regarding the optimal weights for \stacksvd, defined in \Cref{stacksvd.def}, as well as its asymptotic performance, leveraging the results of \cite{liu2023asymptotic,hong2023optimally}. 
\begin{thm}
    \label{thm:stacksvd_weighted} Under \Cref{assum:general_noise,assum:main}, \stacksvd is optimally weighted as 
    \begin{equation*}
    w_i\opt \propto \frac{\theta_i}{\sqrt{\theta_i^2 + c_i}},
\end{equation*}
which yields performance
\begin{equation*}
    (v^\top \hat{v}_\stacksvd)^2 \pto \gamma\opt \quad \text{ the unique solution $x \in (0,1)$ of } \sum_{i=1}^M \theta_i^4 \frac{1-x}{c_i + x\theta_i^{2}} = 1,
\end{equation*}
as long as $\sum_i \theta_i^4/ c_i > 1$, otherwise $\gamma\opt = 0$.
\end{thm}

To prove this claim, we characterize the limiting spectral distribution of $R = \E[X_\text{stack}X_\text{stack}^\top]$ for a given weighting $w$, along with its top eigenvector.
We then compute the inner product between $\hat{v}_\stacksvd$ with $v$, mirroring the analysis of \cite{liu2023asymptotic}, to obtain the asymptotic performance as a function of $w,\theta,c$.
We then optimize this expression over all choices of $w$ to obtain the result in \Cref{thm:stacksvd_weighted}. Full details are deferred to \Cref{app:weightedStackSVDProof}.

Not only does weighting \stacksvd improve its asymptotic performance, it also improves its detectability threshold.
With the optimal weighting, the detectability threshold is $\sum_i \theta_i^4/ c_i > 1$, compared to $\left(\sum_i \theta_i^2\right)^2/\left(\sum_i c_i\right) > 1$ in the unweighted case.
By Cauchy-Schwarz, weighting improves performance: 
$\left(\sum_i \theta_i^2\right)^2
\le \left( \sum_i \theta_i^4/c_i\right) \left( \sum_i c_i\right)$, making it so that whenever \stacksvd detects the signal, so will optimally-weighted \stacksvd.
This improvement is strict when $\theta_i^2/c_i$ is not constant across $i$. When $\theta_i^2/c_i$ is constant across $i$, the optimal weights are uniform, recovering the unweighted case.
In \Cref{sec:experiments}, we provide empirical evidence supporting this theoretical insight (\Cref{fig:fig3_interesting}b,c). 

To interpret these weights, we consider the setting where $M=1$, with no assumptions on $u$ and $v$, and $E$ as i.i.d. Gaussian noise. Then, the MLE for $(u,v)$ is the top singular vectors of the observed matrix.
A natural question is whether there is a similar MLE interpretation for weighted \stacksvd when $M>1$, i.e. under what conditions is weighted \stacksvd the MLE for $v$?
We show that weighted \stacksvd computes the marginalized maximum likelihood estimate when the $u_i$ are assumed to follow a Gaussian distribution.
Concretely, when $u_i\sim \mathcal{N}(0,\frac{1}{n_i} I_{n_i})$, our setting reduces to the spiked covariance model \cite{paul2007asymptotics} where for $X_i$, its rows have covariance $\Sigma_i = \frac{1}{d} I_d + \frac{\theta_i^2}{n_i}v v^\top$, after marginalizing out $u_i$. 
Deriving the MLE for $v$ in this context, for fixed $\{n_i\},d$, with $c_i = n_i/d$ gives $\hat{v} \propto \argmax_{\|v\|_2=1} v^\top \left(\sum_i \frac{\theta_i^2}{\theta_i^2+c_i} X_i^\top X_i \right)v$, exactly recovering weighted \stacksvd.
We detail this derivation in \Cref{app:wstacksvd_mle}.

\subsection{\svdstack weighted analysis}

In contrast to \stacksvd, the optimal weighting of \svdstack, defined in \Cref{svdstack.def}, has not to our knowledge been studied.
Here, we derive the optimal weights for \svdstack, and its resulting performance.

\begin{thm}
    \label{thm:svdstack_weighted}
    Under \Cref{assum:general_noise,assum:main}, an optimal weighting of \svdstack is
    \begin{equation}\label{eq:svdstack.weight}
        w_i \opt = \theta_i \sqrt{\frac{\theta_i^2+1}{\theta_i^2+c_i}} \mathds{1} \left\{ \theta_i^4 > c_i\right\},
    \end{equation}
    which when $\beta_{1}>0$ yields performance
    \begin{equation*}
        |\langle v, \hat{v}_\svdstack \rangle|^2 \pto \frac{S}{S+1}, \quad \text{ where } \quad S = \sum_i \frac{\beta_i^2}{1-\beta_i^2}.
    \end{equation*}
\end{thm}

We defer the detailed proof to \Cref{app:weightedSVDStackProof}, and provide a brief sketch below. 
As in the proof of unweighted \svdstack, we show that the performance is determined by the limiting spectrum of $\hat{V}_w \hat{V}_w^{\top}$ which converges in probability to a weighted version of $A_\beta$, where $\hat{V}_w = \diag{(w)} \tilde{V}$. 
We leverage the fact that $\hat{v}_\svdstack \in \text{span}(\{\hat{v}_i\})$, which allows us to formulate the problem of identifying the optimal coefficients for $\hat{v}_\svdstack$ as a quadratic program.
We then show that these coefficients are achievable by selecting appropriate weights $w$ for $\tilde{V}_w$, and performing an SVD.
\blue{We state optimality within the class of weightings for which $\diag(w) A_{\beta} \diag(w)$ has a simple
largest eigenvalue. A Rayleigh quotient argument extends this to every nonzero
weighting, deterministic or data-dependent: for all $\epsilon > 0$,
$\P\bigl( \sup_{w \neq 0} |\langle v, \hat{v}_\svdstack(w) \rangle|^2 > S/(S+1) + \epsilon
\bigr) \to 0$. We present the restricted form for simplicity, but sketch this extension in \Cref{app:weightedSVDStackProof}.}

To provide some intuition regarding this optimal weighting (\ref{eq:svdstack.weight}), we consider the case where all tables have $\beta_i>0$, and our noise matrices $E_i$ have entries drawn i.i.d. from $\CN(0,1/d)$.
In this case, we show that the optimal weights equalize (stabilize) the noise variance. Note that when above the detectability threshold the optimal weights (\ref{eq:svdstack.weight}) can be rewritten as $w_i\opt = 1/\sqrt{1-\beta_i^2}$.
By Lemma 4.4 of \cite{loffler2021optimality}, with Gaussian noise the component of $\hat{v}_i$ not aligned with $v$ is Haar distributed (up to normalization).
As a result, defining $\tilde{E} \in \R^{M \times d}$ with entries drawn i.i.d. from $\CN(0,1/d)$, the matrix $\tilde{V}$ and its weighted variant $\tilde{V}_w$ can be written as:
\begin{align*}
    \tilde{V} &\approx \beta v^\top + \diag{\left(\sqrt{1-\beta^2}\right)} \tilde E ,\\
    \tilde{V}_w &= \diag{(w)} \tilde{V}  \approx \left(\frac{\beta}{\sqrt{1-\beta^2}}\right) v^\top  + \tilde E,
\end{align*}
\noindent where the approximation emphasizes that the result from \cite{loffler2021optimality} applies to $(I-vv^\top)\hat{v}_1$, rather than $\hat{v}_1-\beta_1 v$, which are asymptotically equivalent.

While the above discussion motivates this specific choice as noise-variance-equalizing weights, it is not a priori clear that this weighting will yield the optimal performance, necessitating our more general analysis culminating in Theorem \ref{thm:svdstack_weighted}.
We see that these weights share similarities with those from \stacksvd, with an extra multiplicative factor of $\sqrt{\theta_i^2+1}$.
Additionally, note that when $c_i=1$ for all $i$ the optimal weights are $w_i\opt \propto \theta_i$, an intuitive choice.

\section{Relative performance of methods} \label{sec:relative_performance}
Having characterized the performance of \stacksvd and \svdstack, both weighted and unweighted, we now compare them in an instance-dependent context, to provide a more nuanced understanding of their relative strengths and weaknesses.
We list these performance comparisons in \Cref{tab:tikz_comparison_summary}, and display them pictorially in \Cref{fig:tikz_comparison}.
Throughout this section, we say that one method dominates another if it has uniformly higher performance over a class of $\theta$ and $c$. \looseness=-1

\begin{table}
\vspace{-.2cm}
\centering
\begin{tabular}{|l|l|l|}
\hline
\textbf{Result} & \textbf{Summary} & \textbf{Setting} \\
\hline
\Cref{prop:dominance} & Weighted \stacksvd dominates weighted \svdstack & For all $\theta,c$ \\
\hline
Prop. \ref{thm:stacksvd_binary_optimal_svd_stack} 
 & Binary \stacksvd dominates weighted \svdstack & For all $c_i \le 1$ \\
\hline
\Cref{remark:stack_outperform_svd} & \stacksvd can outperform weighted \svdstack & \begin{tabular}[c]{@{}l@{}}Many low SNRs\\ $\beta_1 = 0$, $\|\theta\|_2^4 > \|c\|_1$\end{tabular} \\
\hline
\Cref{remark:svd_outperform_stack} & \svdstack can outperform binary \stacksvd & \begin{tabular}[c]{@{}l@{}}Imbalanced SNRs\\ $\beta_2 > 0$ and $\|\theta\|_2^4 < \|c\|_1$\end{tabular} \\
\hline
\end{tabular}
\caption{Summary of the relative performances of different methods. Weighted \stacksvd uniformly dominates all other methods (binary \stacksvd, \stacksvd, weighted \svdstack, and \svdstack). Among these four no one method uniformly outperforms the others. Comparisons displayed pictorially in \Cref{fig:tikz_comparison}.}
\label{tab:tikz_comparison_summary}
\vspace{-.5cm}
\end{table}

We begin with the following remark, showing that for certain instances unweighted \stacksvd can outperform optimally weighted \svdstack (and by extension unweighted \svdstack).
We corroborate this result via simulation in \Cref{fig:fig3_interesting}a.
\begin{remark}\label{remark:stack_outperform_svd}
Unweighted \stacksvd can outperform optimally weighted \svdstack, particularly when many tables exhibit weak signal strengths below the detection threshold. 
\end{remark}
\begin{proof}[Justification of \Cref{remark:stack_outperform_svd}]
From \Cref{thm:simple_thm1}, unweighted \stacksvd can outperform even optimally weighted \svdstack by utilizing many tables with $\theta_i \le c_i^{1/4}$.
Concretely, consider $\theta_i=\theta_0 = 1$, $c_i = c_0 = 1$ for all $i$.
All tables have $\beta_i=0$, and so \svdstack and optimally weighted \svdstack fall below the detection threshold, with performance $0$.
Unweighted \stacksvd, however, has a performance of $\frac{M^2-M}{M(M+1)} = 1-\frac{2}{M+1}$, which is positive for $M>1$ and goes to 1 as $M\to \infty$.
This characterizes a broad class of instances, where many tables may have signal strength below the detectability threshold marginally, but when analyzed together \stacksvd is able to aggregate strength across \textit{all} tables to obtain good performance.
\end{proof}

While \Cref{prop:single_table} precludes \svdstack outperforming \stacksvd when all tables have the same $\theta_i$ and $c_i$, \svdstack can in fact outperform \stacksvd once tables have unequal $\theta_i$.
Intuitively, this occurs when there are low SNR tables that dilute the signal of the high SNR ones, causing \stacksvd to fail while even unweighted \svdstack essentially ignores the low SNR ones.
We formalize this in the following remark, showing that  unweighted \svdstack can actually outperform even optimally binary-weighted \stacksvd.

\begin{remark}
\label{remark:svd_outperform_stack}
    Unweighted \svdstack can outperform binary-weighted \stacksvd when tables have highly imbalanced signal strengths. 
\end{remark}

\begin{proof}[Justification of \Cref{remark:svd_outperform_stack}]
    We begin by proving the easier claim, that unweighted \svdstack can outperform unweighted \stacksvd, and then extend our discussion to  binary-weighted \stacksvd.

\medskip \noindent \textbf{Improvement over unweighted \stacksvd:}
the detection threshold of \svdstack can be lower than that of \stacksvd in the presence of ``nuisance'' tables, when
\begin{align*}
    \beta_2 > 0 \quad \text{ and } \quad 
    \|\theta\|_2^4 < \|c\|_1,
\end{align*}
where the first condition is for \svdstack to be above its detection threshold, and the second condition is for \stacksvd to fall below its detection threshold.
This formalizes the intuition that \stacksvd fails either by the addition of many low signal tables ($\theta_i=0$, $c_i>0$), or by the addition of a small number of tables with large $c_i$.
As a concrete example, consider $c_i=c_0 $ for all $i$, with $M > 2$, where $\theta_1 = \theta_2 = \theta_0>c_0^{1/4}$ and $\theta_3, \ldots, \theta_M = 0$, with $\beta_0$ defined appropriately.
Then, for $M>4\theta_0^4/c_0$, the asymptotic performance of \stacksvd is 0, while the performance of \svdstack is $\frac{2 \beta_0^2}{1+ \beta_0^2}$.

\medskip \noindent \textbf{Improvement over binary-weighted \stacksvd:}
\svdstack can outperform binary-weighted \stacksvd (defined in \Cref{cor.2}), especially in cases where one table with low SNR (large $c_i$ and small $\theta_i$) dilutes the signal of the other tables.
As a concrete example, consider $c_1=c_2=1$, with $c_3$ large.
Here, $\theta_1=\theta_2=2$, with $\theta_3 = c_3^{1/4}+\epsilon$ with $\epsilon$ some small positive constant (think $\epsilon\to0$), yielding $\beta_1^2=\beta_2^2 = .75$, with $\beta_3$ vanishingly small.
Then, the performance of \svdstack is $\frac{2\beta_1^2}{1+\beta_1^2} = 6/7$.
However, binary-weighted \stacksvd has $\|\theta\|_2^2 = \sqrt{c_3} + 8$, and $\|c\|_1 = c_3+2$.
This yields a performance of: $\frac{(c_3^{1/2}+8)^2 - c_3-2}{(c_3^{1/2}+8)(c_3^{1/2}+9)} = \frac{16 \sqrt{c_3} + 62}{c_3 + 17 \sqrt{c_3} + 72} \to 0$ as $c_3 \to \infty$.

\medskip \noindent\textbf{Improvement over all binary weightings of \stacksvd:}
For completeness, we additionally consider other subset selection rules for binary-weighted \stacksvd.
Concretely, beyond $w_i = \mathds{1}\{\theta_i^4 > c_i\}$, one can generally consider any subset $\CS \subseteq [M]$ of the tables to perform unweighted \stacksvd on.
As we show, \svdstack can outperform all binary weightings of \stacksvd.
This can happen when tables have similar $\beta_i$'s, but largely different $c_i$'s.
As a concrete example, consider the case where $\theta_1= \sqrt{5}$ with $c_1=1$, and $\theta_2 = 4$ with $c_2=38.4$ where $\beta_1^2=\beta_2^2=0.8$.
Here, binary \stacksvd has a performance of $0.869$, which decreases to 0.8 if either of the two tables is used on its own.
Weighted \svdstack on the other hand has a performance of $0.889$.
This implies that there exist examples where \svdstack outperforms \textit{any} binary weighting of \stacksvd, not just the naive one of $w_i = \mathds{1}\{\beta_i > 0\}$. \end{proof}

As shown by the above examples, the key to \svdstack outperforming \stacksvd is the presence of tables with large $c_i$ and relatively small $\theta_i$, to dilute the signal of the other tables.
This is in fact necessary, as shown in the following proposition.

\begin{prop} \label{thm:stacksvd_binary_optimal_svd_stack}
    \stacksvd with weights $w_i = \mathds{1}\{\beta_i > 0\}$ dominates optimally-weighted \svdstack when $c_i \le 1$ for all $i$, yielding strict improvement when $\beta_{2}>0$.
\end{prop}
This proposition (proof in \Cref{app:binary_weighting}) indicates that, whenever we are able to detect and discard tables with signals below the detectability threshold, \stacksvd outperforms weighted \svdstack, as long as the tables are not too large.

Finally, we show that if we improve the weighting of \stacksvd beyond binary, and instead optimize over \textit{all} weightings, that optimally weighted \stacksvd dominates all other methods without any restrictions on $c$.

\begin{thm}
    \label{prop:dominance}
    Optimally weighted \stacksvd dominates unweighted \stacksvd and optimally weighted \svdstack, providing strict improvement above its recovery threshold when $\theta_i^2/c_i$ is not constant across $i$, and when at least two $\theta_i$ are nonzero (respectively).
\end{thm}

We defer the proof details to \Cref{sec:dominance_opt_weighted_stacksvd}, and provide a sketch below.
Above the threshold of detectability, the performance of optimally weighted \stacksvd (\Cref{thm:stacksvd_weighted}) is the unique root of the function $f$ defined below:
\begin{equation*} 
    f(x) = \sum_{i=1}^M \theta_i^4 \frac{1-x}{c_i + x\theta_i^2} -1.
\end{equation*}
$f$ is a decreasing function of $x$, so we show that $f(x)>0$ for $x= S/(S+1)$, the performance of optimally weighted \svdstack (\Cref{thm:svdstack_weighted}).

This improvement in performance by optimally weighted \stacksvd can be dramatic for certain problem instances.
In fact, optimally weighted \stacksvd can have asymptotic performance going to 1, even when all other methods fall below the threshold of detectability.
\svdstack fails when $\theta_i^4 \le c_i$ for all $i$, and \stacksvd fails when $\|\theta\|^4 \le \|c\|_1$.
This can occur when $c_i$ is very large for one table with small $\theta_i$, where the remaining tables all have weak signal.
As a concrete example, consider $\theta_1=0$, $c_1$ to be large, and $\theta_2 = \hdots = \theta_{M}=1$ and $c_2 = \hdots = c_{M} = 1$.
Then,  $\|\theta\|^4 = (M-1)^2$, and so for $c_1>(M-1)^2 - (M-1)$ unweighted \stacksvd is unable to recover $v$.
Additionally, all tables fall below the marginal detection threshold, so \svdstack fails as well.
Thus, optimally weighted \stacksvd has nonzero asymptotic performance, but all other methods fail.
We numerically validate this surprising behavior in \Cref{fig:fig3_interesting}c.

\Cref{prop:dominance} shows that optimally weighted \stacksvd dominates both optimally weighted \svdstack and unweighted \stacksvd.
Initially, one might think that optimally binary-weighted \stacksvd may be able to bridge this gap, as we showed in \Cref{thm:stacksvd_binary_optimal_svd_stack}.
However, as we show in the following proposition, optimal weighting is indeed critical: we can construct examples where optimally-weighted \svdstack and optimally binary-weighted \stacksvd fall below the threshold of detectability, but optimally weighted \stacksvd has performance going to 1.

\begin{prop}
    \label{prop:binarystacksvd_inadmissable}
    For any $\eps\in(0,1)$, there exists a problem instance of size $M= \lceil e^{-\gamma} \exp(2/\eps)\rceil$ where optimally weighted \stacksvd has asymptotic performance greater than $1-\eps$, while optimally binary-weighted \stacksvd and optimally weighted \svdstack both fall below their recovery thresholds.
    This additionally implies that unweighted \stacksvd and \svdstack are also below their recovery thresholds.
\end{prop}
Notationally, $\gamma \approx .577$ denotes the Euler-Mascheroni constant \cite{dence2009survey}.
We construct such an example by taking $\theta_i=\theta_0=1$ for all $i$, and $c_i = 2i-1$.
This satisfies the necessary requirements, as firstly each table is below the threshold of detectability marginally.
Secondly, the optimal binary weighting must select a prefix of tables (as the tables are sorted by decreasing SNR), where each prefix is carefully balanced to sit exactly at the threshold of detectability.
We defer additional details to \Cref{app:binary_weighting_opt}, and numerically validate the increasing performance of optimally weighted \stacksvd in \Cref{fig:binary_stack_svd_inadmissable}.

\section{Experiments} \label{sec:experiments}
We provide extensive simulations on synthetic data, demonstrating the close correspondence across myriad settings between our theoretical results and the empirical performance of weighted and unweighted \stacksvd and \svdstack.
Across all simulations we set $d = 1000$ and $n_i = d c_i$ (rounded to the nearest integer), unless otherwise specified.
The entries of $E_i$ are i.i.d. $\CN(0,1/d)$, and $v$ and all $u_i$ are random unit vectors.
We assume knowledge of the true $\{\theta_i\}$ in our construction of the weights for weighted \stacksvd and \svdstack. 
To begin, we validate our theoretical predictions across a wide range of $\theta_i$ and $c_i$.
In particular, we set $M = 5$, sampled $c_i^{1/4} \iid \text{Expo}(1) + 0.1$ (exponentially distributed) and $\theta_i = c_i^{1/4} \exp(W)$, where $W \sim \mathcal{N}(\mu, 1/100)$.
For each choice of $\mu \in \{0,.4,.7\}$, we randomly generated 100 such vectors $c,\theta$.
For each fixed $c$ and $\theta$, we simulated $10$ replicates and applied each of the considered methods to obtain an estimate of $|\langle \hat{v}, v \rangle|^2$. We then compared this to the theoretical prediction for each method to obtain a bias and standard error.
\Cref{fig:asymptotic_bias}a-b shows that, in general, all methods concentrate closely to the theoretical value, with the exception of \svdstack which exhibits a small positive bias that decreases as $\mu$ increases.
Moreover, the variance of \svdstack is typically larger than \stacksvd, especially when $\mu$ is closer to $0$.
This increased variance could be the result of larger fluctuations in the individual singular vectors that are propagated to the final result.
As we show in our subsequent figures, our theoretical predictions are asymptotically accurate as $d \to \infty$.

\begin{figure}[h]
    \centering
    \includegraphics[scale=0.6]{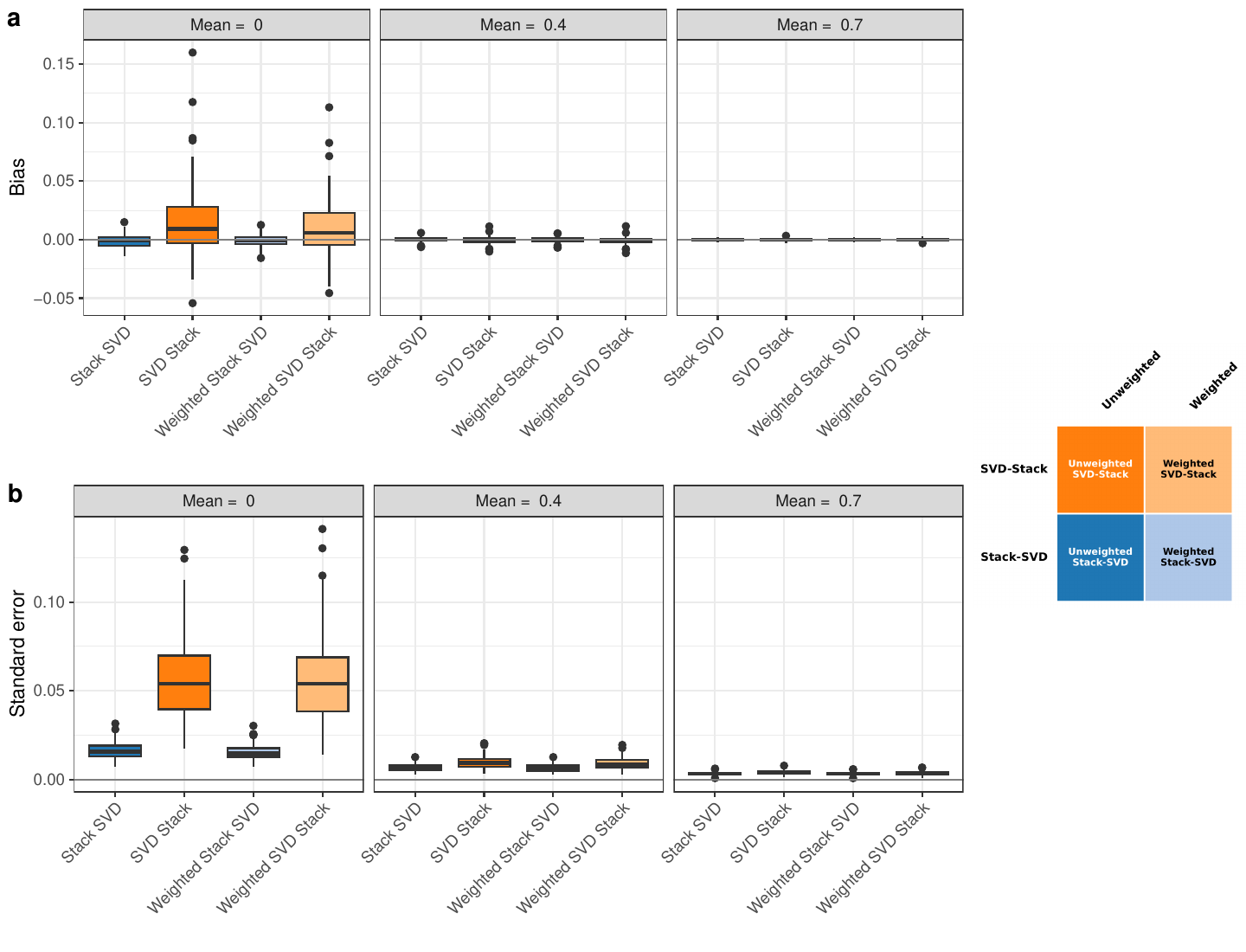}
    \vspace{-.7cm}
    \caption{For each mean $\mu \in \{0,0.4,0.7\}$, $100$ values of $c$ and $\theta$ were sampled according to the procedure specified in the main text.
    For each fixed value of $c$ and $\theta$, $10$ replicates were generated and $|\langle \hat{v}, v\rangle|^2$ was compared to its theoretical prediction to obtain estimates of the bias (\textbf{a}) and standard error (\textbf{b}).}
    \label{fig:asymptotic_bias}
        \vspace{-.5cm}
\end{figure}

Next, we examine the convergence behavior of different methods in three representative scenarios, shown in \Cref{fig:2_rate}. 
In each case, the empirical performance of the different methods (solid lines) rapidly approaches the corresponding theoretical predictions (dashed lines). 
We evaluate the rate of convergence as $d \to \infty$ by considering settings with $M=2$ and three different signal configurations: $\theta=[\theta_1,\theta_2]$ is slightly below the phase transition (\Cref{fig:2_rate}a), slightly above it (\Cref{fig:2_rate}b), and significantly above it (\Cref{fig:2_rate}c).
We find that the convergence to the theoretical value is faster when $\theta$ is farther from the threshold. 
Interestingly, in \Cref{fig:2_rate}a where both tables are below the marginal threshold of detectability ($\theta = [.98,.76]$), we expect zero performance for both weighted and unweighted \svdstack.
However in practice these methods still exhibit moderate performance when $d$ is small, which only gradually converge to the theoretical limit of 0 as $d\to\infty$, showing the slow convergence from \Cref{fig:asymptotic_bias}.

\begin{figure}[h]
    \centering
    \includegraphics[scale=0.5]{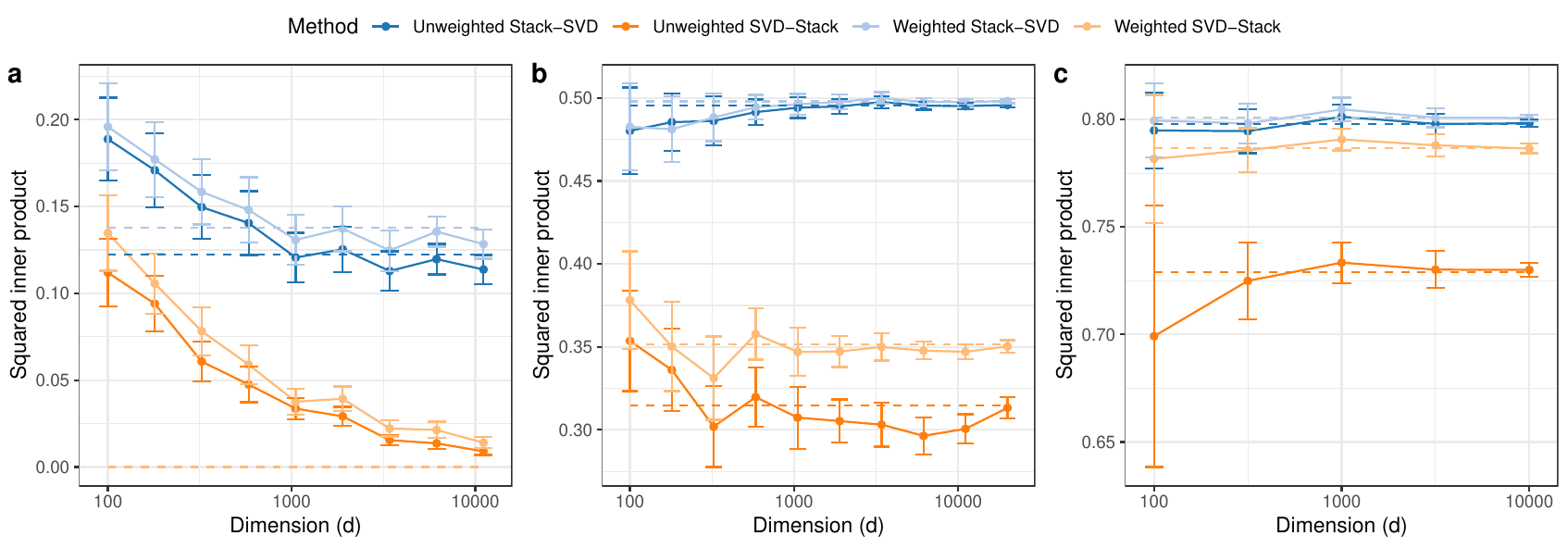}
    \vspace{-.4cm}
    \caption{Empirical results converge to theoretical predictions for $M=2$ tables with $c_1=c_2=1$ and differing signal strengths, varying $d$. 
    \textbf{a.} $\theta=[.98,.76]$, matrices are close to but below the threshold of detectability, converging slowly.
    \textbf{b.} $\theta=[1.2,1.05]$, matrices are slightly above the marginal threshold for detectability, converging at moderate $d$.
    \textbf{c.} $\theta=[2,1.3]$, all matrices are substantially above the threshold, yielding rapid convergence.
    Error bars represent $\pm 1.96$ times the standard error of the mean across $100$ simulations.}
    \label{fig:2_rate}
    \vspace{-.5cm}
\end{figure}

In \Cref{fig:fig3_interesting} we focus on some specific examples and empirically verify our key theoretical insights regarding the pros and cons of different methods.
First, we show a simple case where all matrices have $\theta_i=0.7$ with $c_i=1$, and $d=10000$ (\Cref{fig:fig3_interesting}a).
Each table is marginally under the threshold of detection, with $\theta_i<1$, and so \svdstack falls below the threshold of detectability.
Weighted and unweighted \stacksvd perform the same, where the threshold of detectability is at $M>4$.
We see that even though each table is marginally uninformative, together they can achieve strong performance, corroborating our results in \Cref{thm:simple_thm1}.
Next, we consider $\theta_i=1.2$ for $i=1,2$, and 0 otherwise, with $c_i=1$ for all $i$ (\Cref{fig:fig3_interesting}b) with $d=10000$.
Here, \svdstack (weighted and unweighted) have constant performance, as the signal in the direction of $v$ is the only consistent signal identified, and so adding a constant number of uninformative tables does not change the performance.
Weighted \stacksvd applies a weight of 0 to tables with $\theta_i=0$, and so maintains high, constant performance.
Unweighted \stacksvd, on the other hand, has its signal diluted by the pure noise tables, and so quickly fails. Appealing to \Cref{prop:stacksvd_general}, \stacksvd falls below the threshold of detectability for $M \ge 9$.
Finally, \Cref{fig:fig3_interesting}c highlights the dominance of weighted \stacksvd.
Here, $\theta = [.95,.95,0]$ and $c=[1,1,2]$, providing an asymptotic performance of $\approx .25$ for weighted \stacksvd, and 0 for all other methods.
We see with increasing $d$ the convergence of the empirical performance to the predicted asymptotic limit.

\begin{figure}[h]
    \centering
    \includegraphics[scale=0.5]{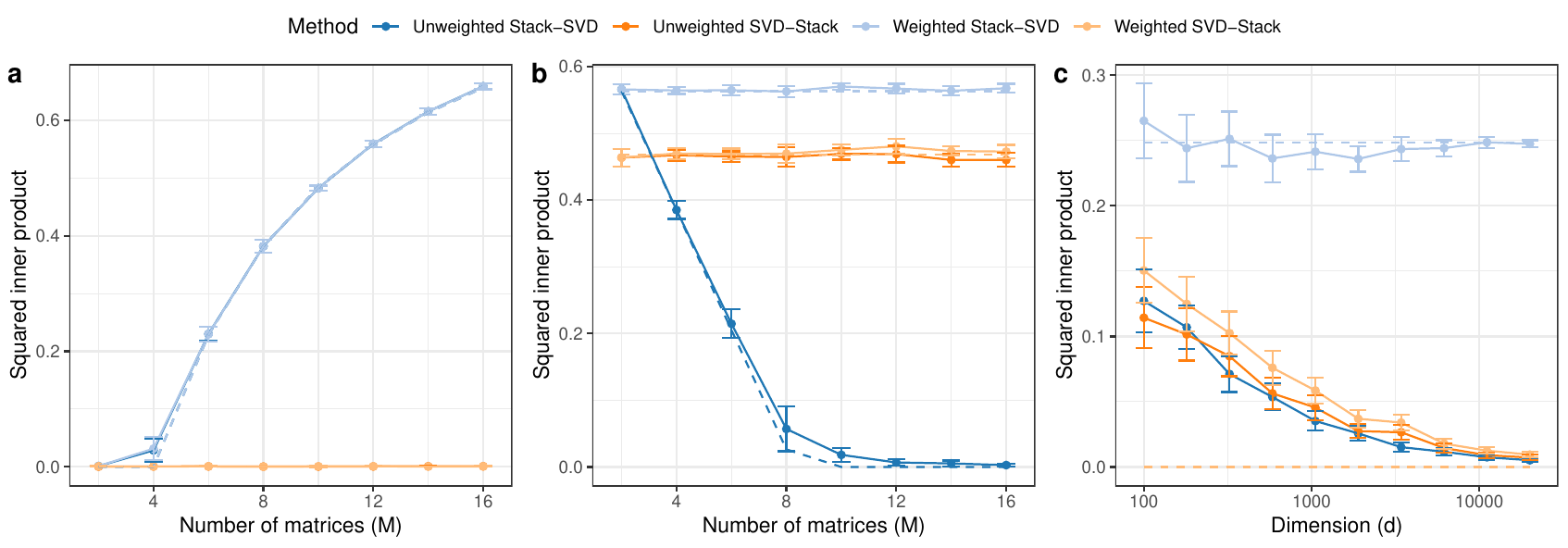}
    \caption{Simulations highlighting several illustrative examples. Details in text.
    \textbf{a.} $\theta_i=0.7$ with $c_i=1$. Each matrix is below the threshold of detectability, and so \svdstack has a performance of 0, while \stacksvd aggregates power across matrices to obtain strong performance.
    \textbf{b.} Two tables have strong signal while the rest have 0. \svdstack and weighted \stacksvd have constant performance while unweighted \stacksvd has vanishing performance as $M$ increases.
    \textbf{c.} $\theta=[.95,.95,0]$ with $c=[1,1,2]$, where unweighted \stacksvd, and both weighted and unweighted \svdstack, fall below the recovery threshold, while weighted \stacksvd yields positive performance.}
    \label{fig:fig3_interesting}
    \vspace{-.5cm}
\end{figure}

{
\section{General rank signal with unaligned subspaces} \label{sec:general_rankr_unaligned}
Thus far, our analysis has focused on the rank $1$ setting. In this section, we extend these results to a more complex setting where the right singular subspace of each signal matrix lies within a common subspace of dimension $r \ll d$. All of the above results appear as special cases of this more general model. By allowing for partially aligned singular subspaces, a more nuanced comparison between \stacksvd and \svdstack emerges. 

We define $V \in \mathbb{O}(d,r)$ to be an orthonormal basis for the rank $r$ subspace spanned by the union of the right singular subspaces across each of the signal matrices.
\begin{assum}\label{assum:unaligned}
    The matrices $\{X_i \}$ are generated as below, 
    where $\Theta_i \in \R^{r_i \times r_i}_{>0}$ is diagonal, $U_i \in \mathbb{O}(n_i, r_i)$, and $R_i \in \mathbb{O}(r,r_i)$.
    Additionally, the noise conditions of \Cref{assum:general_noise} and RMT scaling limits of \eqref{eq:RMT_limit} are satisfied, and $\{R_i\}$ span $\R^r$: \looseness=-1
    \begin{align*}
        X_i = U_i \Theta_i (V R_i)^{\top} + E_i\\ \operatorname{Rank}\Big(\sum_i R_i R_i^\top \Big) = r
    \end{align*}
\end{assum}

Throughout, we will write $\tilde{r} := \sum_{i=1}^M r_i$ to denote the number of spiked singular values across all of the data matrices. 
Note that we make no assumption on the ordering of singular values in the $\Theta_i$ matrix.
For simplicity we will assume for the analysis of \svdstack that, for each $i$, the diagonal entries of $\Theta_i$ are distinct. 
This general framework will also allow us to recover the fully aligned results of \Cref{sec:rank_r} by taking $R_i = I_r$ for each $i$. 

\medskip 

In this general rank $r$ setting, we define the estimate $\Vsvdstack$ generated by \svdstack as the $d \times r$ matrix containing the top $r$ right singular vectors of $\tilde{V}$. The estimate $\Vstacksvd$ generated by \stacksvd is defined as the $d \times r$ matrix containing top $r$ right singular vectors of $X_{\text{stack}}$. We assess the performance of each method in terms of the subspace similarity metric $\| \hat{V}^{\top} V \|_F^2$.
All proofs are deferred to \Cref{sec:unaligned_rank_r}.

\subsection{Unweighted \stacksvd and \svdstack  in the unaligned setting} \label{sec:subsec_unweighted_unaligned123}
Define the $\tilde{r} \times \tilde{r}$ matrix $A_{\beta, R}$ with diagonal entries $1$ and off-diagonal entries:
\begin{equation}
    \left[A_{\beta, R} \right]_{(i,j), (i', j')} = \beta_{ij} \beta_{i' j'}  (R_i)_j^{\top} (R_{i'})_{j'}, \quad (i,j) \neq (i', j')
\end{equation}
where $i \in [M]$ indexes the table and $j\in [r_i]$ indexes the column of $VR_i$. As before, this matrix is
the pointwise limit of $\tilde{V} \tilde{V}^{\top}$, and the performance of \svdstack is completely determined by its eigenstructure. 
The following definitions are required to state the limiting performance.
When clear from context, we overload notation and define $Q$ and $\Lambda$ as matrices rather than operators.

\begin{itemize}
    \item $\Lambda = \Lambda_r(A_{\beta, R}) \in \R_{>0}^{r \times r}$ is the diagonal matrix containing the top $r$ eigenvalues of $A_{\beta, R}$.
    \item $Q=Q_r(A_{\beta, R}) \in \R^{\tilde{r} \times r}$ is a matrix with columns containing eigenvectors corresponding to the top $r$ eigenvalues of $A_{\beta, R}$, such that $Q\Lambda Q^\top$ is the best rank $r$ approximation to $A_{\beta,R}$ in Frobenius norm.
    \item $\theta_{ij}$ denotes the signal strength of the $j$-th component in matrix $i$. The performance $\beta_{ij}$ is defined as in \Cref{prop:single_table}. $\beta \in \R^{\tilde{r}}$ denotes the vector with entries $\beta_{ij}$. 
    \item $R_{\text{stack}} \in \R^{\tilde{r} \times r}$ is the vertical concatenation of $R_i^\top$, i.e.
    \begin{equation*}
        R_\text{stack} = \begin{bmatrix}
            R_1^\top \\ \vdots \\ R_M^\top
        \end{bmatrix}
    \end{equation*}
\end{itemize}

\begin{prop}\label{prop:general_rank_unweighted_svdstack}
Assume that for each data matrix $X_i$, $\Theta_i$ has $r_i$ distinct singular values. Further assume that $\lambda_r(A_{\beta,R}) - \lambda_{r+1}(A_{\beta,R}) > 0$, where $\lambda_{\tilde{r} + 1}(A_{\beta,R}) := -\infty$. Then the output of \svdstack satisfies: 
    \begin{align*}
    \| \hat{V}_{\svdstack}^{\top} V \|_F^2 &\ip \left\| \Lambda^{-1/2} Q^{\top} \diag(\beta) R_{\text{stack}} \right\|_F^2. 
    \end{align*}
\end{prop}

To see that the above proposition recovers \Cref{thm:svd_stack_general} in the setting of an exactly shared rank $1$ subspace, note that by taking $r_i=1$, with $R_i = \begin{bmatrix}1\end{bmatrix}$ for all $1 \leq i \leq M$:
\begin{equation*}
    \Lambda^{-1/2} Q^{\top} \text{diag}(\beta) R_{\text{stack}} = \frac{v_{\max}(A_{\beta})^{\top} \beta}{\sqrt{\lambda_{\max}(A_{\beta})}}.
\end{equation*}

\medskip 

Next we turn to the performance of unweighted \stacksvd. The stacked matrix can be written as:
\begin{equation*}
    X_\text{stack} = \begin{bmatrix} X_1 \\ \vdots \\ X_M \end{bmatrix}
        = \begin{bmatrix} U_1 \Theta_1 R_1^\top V^\top + E_1 \\ \vdots \\ U_M \Theta_M R_M^\top V^\top + E_M \end{bmatrix}
\end{equation*}
The population signal component of the unweighted empirical covariance matrix
$X_\text{stack}^\top X_\text{stack} = \sum_i X_i^\top X_i$ is
\begin{equation*}\label{eq:signal_cov}
  S = \sum_{i=1}^M V_i \Theta_i^2 V_i^\top
    = V \underbrace{\left(\sum_{i=1}^M R_i \Theta_i^2 R_i^\top\right)}_{C} V^\top
    = VCV^\top = (VQ) \Lambda (VQ)^\top.
\end{equation*}
The \emph{core matrix} $C \in \mathbb{R}^{r \times r}$ has eigendecomposition
$C = Q\Lambda Q^\top$ with eigenvalues $\lambda_1 \geq \cdots \geq \lambda_r > 0$.
If the $\lambda_j$ are distinct, the population (true) right singular vectors of $X_\text{stack}$ are $\{Vq_j\}_{j=1}^r$ with signal strengths $\{\sqrt{\lambda_j}\}_{j=1}^r$. 
The performance of \stacksvd in this unaligned setting can be stated entirely in terms of the spectrum of $C$, which depends only on the $\Theta_i$ and $R_i$.

\begin{prop} \label{prop:stacksvd_subspace}
Under \Cref{assum:unaligned} 
    \begin{equation*}
        \left\| \Vstacksvd^{\top} V \right\|_F^2 \ip \sum_{j=1}^r 
        \frac{\lambda_j^2(C) - \|c\|_1}{\lambda_j^2(C) + \lambda_j(C)} \mathds{1}\left\{ \lambda_j^2(C) > \|c\|_1\right\}.
    \end{equation*}
\end{prop}
This argument follows similarly to that of \Cref{prop:stacksvd_general}, 
as $\E\left[X_\text{stack}^\top\right] \E \left[ X_\text{stack} \right] = VCV^\top$. Note that \Cref{prop:stacksvd_subspace} applies even when $C$ contains repeated eigenvalues.
As before, \Cref{prop:stacksvd_subspace} can also be used to recover the special case of \Cref{prop:stacksvd_general}. When $r=r_i=1$, and $R_i = 1$, the core matrix $C$ reduces exactly to $\| \theta \|_2^2$. \looseness=-1

\medskip 

Our results show that \svdstack generally performs better when there is weak alignment between the columns of $V R_i$. We illustrate this via a simple example with $M = 2$ and partially aligned subspaces:
\begin{equation} \label{eq:psi_equation}
R_1 = \begin{pmatrix} 1 \\ 0 \end{pmatrix}, \quad R_2 = \begin{pmatrix} \sin(\psi) \\ \cos(\psi) \end{pmatrix}, \quad \theta_1 = \theta_2 = \theta \geq 1, \quad c_1 = c_2 = 1, \quad \psi \in [0, \pi/2).
\end{equation}
In terms of $X_1$ and $X_2$, the above implies
\begin{equation}
    \E X_1 = \theta u_1 v_1^{\top}, \quad \E X_2 = \theta u_2(\sin(\psi) v_1 + \cos(\psi) v_2)^{\top}.
\end{equation}
Letting $s := \sin(\psi)$, an application of \Cref{prop:general_rank_unweighted_svdstack,prop:stacksvd_subspace} shows that

\begin{align*}
    \left\| \Vsvdstack^{\top} V \right\|_F^2 &\ip \frac{\beta^2(1+s)}{1 + s \beta^2} + \frac{\beta^2(1-s)}{1 - s \beta^2} \numberthis \label{eq:perf_rankr_ex_svdstack}\\
    \left\| \Vstacksvd^{\top} V \right\|_F^2 &\ip \frac{\theta^4(1 + s)^2 - 2c}{\theta^4(1 + s)^2 + \theta^2 (1+s)} \mathds{1}(\theta^4 (1 + s)^2 > 2c) \numberthis\label{eq:perf_rankr_ex_stacksvd} \\ &+ \frac{\theta^4(1 - s)^2 - 2c}{\theta^4(1 - s)^2 + \theta^2 (1-s)} \mathds{1}(\theta^4 (1 - s)^2 > 2c).
\end{align*}

We make several observations regarding \Cref{eq:perf_rankr_ex_svdstack,eq:perf_rankr_ex_stacksvd}, which we confirm via simulation (\Cref{fig:unaligned_sim}).
Proofs are deferred to \Cref{app:unaligned_example_proofs}.
First, we observe that \svdstack always dominates when $s$ is sufficiently small.
Evaluating the performance of both methods when $s=0$, we see that \svdstack analyzes each table separately to attain a performance of $2\beta^2(\theta,c)$. \stacksvd on the other hand has the noise (additional rows) from the other table \textit{dilute} the signal, and attains a performance of only $2\beta^2(\theta,2c)$, which is strictly worse.
Analyzing the other endpoint at $s=1$ (which by continuity describes $s \uparrow 1$), we see by \Cref{thm:simple_thm1} that \stacksvd outperforms \svdstack.
Aside from the indicators in \Cref{eq:perf_rankr_ex_stacksvd}, both performances are continuous functions of $s$, and so \svdstack outperforms \stacksvd up until some $s'>0$, potentially for all $s$ except those very close to 1, as shown in \Cref{fig:unaligned_sim}c.

Although the performance of \svdstack is monotonically nonincreasing as a function of $s$, \stacksvd exhibits two possible kinks induced by the indicators in \Cref{eq:perf_rankr_ex_stacksvd}. 
When $\theta^4 > 2c$, there is a kink at $s \opt = 1 - \sqrt{2c}/\theta^2$ corresponding to the detection threshold of $v_2$. The performance decreases for $s < s\opt$ because the loss of signal in $v_2$ outweighs the gain of signal in $v_1$. On the other hand, for $s > s \opt$ only the increasing signal in $v_1$ contributes as $v_2$ is undetectable. 

Studying the cases with weaker signal, if instead $c < \theta^4 \le 2c$, component 2 is never detected, and there is a kink at $s_{\dagger}= \frac{\sqrt{2c}}{\theta^2} -1$, above which component 1 is detected (see \Cref{fig:unaligned_sim}, $\theta = 1.15$). In the even lower signal regime where  $c/2 < \theta^4 \le c$, \svdstack attains a performance of $0$ while \stacksvd will have nonzero performance after $s_{\dagger}$. Finally, if $\theta^4 \le c/2$, no method will have nonzero performance.

\begin{figure}
    \centering
    \includegraphics[scale=0.6]{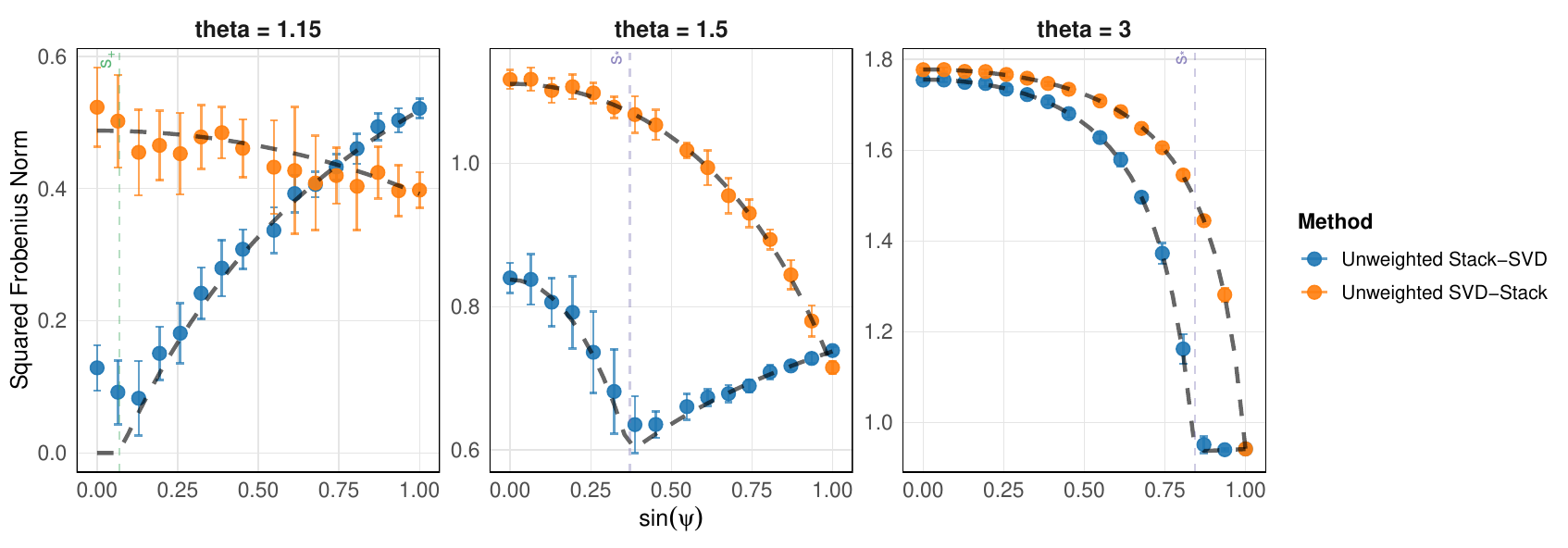}
    \caption{Simulations from \Cref{eq:psi_equation} with different values of $\theta$. $n = d = 3000$ and $10$ repetitions per value of $\sin(\psi)$. The black dashed lines represent the theoretical predictions. Vertical dashed lines indicate kinks $s\opt$, $s_{\dagger}$ in performance for \stacksvd.}
    \label{fig:unaligned_sim}
\end{figure}

Having fully characterized the performance of \stacksvd and \svdstack in the unweighted setting, we now turn our attention to their weighted counterparts, as in the fully-shared setting.

\subsection{Weighted \svdstack  in the unaligned setting}
For \svdstack in this setting, we define $\Vsvdstack(W)$ as the top $r$ right singular vectors of $W \tilde{V} \in \R^{\tilde{r} \times d}$, where $W \in \R^{\tilde{r} \times \tilde{r}}$. As before, we define an optimal weighting $W\opt$ as a choice of $W$ such that the limit (in probability) of $\| \hat{V}(W)^{\top} V\|_F^2$ is maximized (and constant)
\footnote{
We note that some choices of $W$ fail to produce the required eigengap in the $W A_{\beta,R} W^{\top}$ matrix as discussed in the proof of \Cref{thm:gen_rank_weight_svdstack}. We do not consider such $W$ as in these cases $\|\hat{V}(W)^{\top} V\|_F$ does not converge in probability to a constant.}.
The following result shows that an asymptotically optimal weighting $W\opt$ is a diagonal matrix with entries identical to the optimal weights from the rank $1$ setting (\Cref{thm:svdstack_weighted}). 

\begin{thm}\label{thm:gen_rank_weight_svdstack}
Suppose that $\beta_{ij} > 0$ for all $i,j$.
Then, under \Cref{assum:unaligned},\svdstack can be optimally weighted with $W\opt = \diag(1/\sqrt{1 - \beta_{ij}^2})$. 
Specifically,
\begin{equation*}
    \left\| \Vsvdstack(W\opt)^{\top} V \right\|_F^2 \ip L\opt := r - \sum_{\ell = 1}^r \lambda_{\tilde{r} + 1 - \ell} (A_{\beta, R}^{-1/2} (W\opt)^{-2} A_{\beta, R}^{-1/2})
\end{equation*}
\end{thm}
We emphasize that \Cref{thm:gen_rank_weight_svdstack} shows that the optimal weighting depends only on $\theta_{ij}$ and $c_i$ and \textit{does not depend on $R_i$}. Despite allowing for non-diagonal weightings $W \in \R^{\tilde{r} \times \tilde{r}}$ that could ``rotate'' the vectors from the initial SVD, our result shows that this is not necessary to obtain the maximal performance. This makes \svdstack an attractive choice in practical settings, as the $R_i$ are difficult to estimate without prior knowledge of $V$. 
The fact that this weighting is independent of $R_i$ makes it intuitive and easy to implement. \looseness=-1

\subsection{Weighted \stacksvd in the unaligned setting} \label{sec:weighted_stacksvd_unaligned}
For \stacksvd in the general rank $r$ setting, there are three plausible weighting schemes. In the following we discuss each of these options, ordered by the increasing number of assumptions necessary. Except in the case of exactly aligned singular subspaces, we cannot provide closed-form optimal weightings for \stacksvd in the most general setting. 

\paragraph*{Binary weighting} First, binary-weighted \stacksvd assigns each $X_i$ a weight $w_i \in \{0,1\}$. Arguing similarly to \Cref{cor.2}, one can easily derive the performance of \stacksvd in terms of $c_i$ and $\lambda_i(C)$, the latter of which can be estimated from the data (provided sufficiently strong signal, as discussed in \Cref{sec:practitioners_guide}). The optimal choice of binary weighting can be found by comparing all $2^M-1$ options, which is possible when $M$ is modest (as is the case in most of the considered applications). Without additional knowledge of the $R_i$, binary weighted \stacksvd is the most practical choice.

\paragraph*{General (single) weighting} Alternatively, we can consider using a single weight per matrix. However, while we can compute the asymptotic performance as a function of the weighting vector and the rotation matrices, it does not yield a closed form optimal solution and requires knowledge of the $R_i$ (\Cref{prop:gen_rank_stacksvd_singleweight} in \Cref{sec:insufficiency_of_global_weights_stacksvd}).
Nevertheless, through a more careful analysis of the example in \Cref{fig:unaligned_sim} (\Cref{eq:psi_equation} with $\sin(\psi)=0$), we are able to prove the following Proposition (details in \Cref{sec:appendix_insufficiency_single_weights_stacksvd}):

\begin{prop} \label{prop:singleweight_suboptimality}
    In the general rank-$r$ setting (under \Cref{assum:unaligned}), unweighted \svdstack can outperform optimally weighted \stacksvd with a single weight per table.
\end{prop}

\paragraph*{Component-specific weighting}
Finally, in the case of exactly shared right singular subspaces (the $R_i$ are all square and permutations of the identity matrix) one can employ a component-specific weighting as was done in weighted PCA \cite{hong2023optimally}. In particular, we assemble for each $j=1,\hdots,r$ the matrix $X_\text{stack}^{(j)}$ based on the optimal weightings $w_{ij}$ for table $i$ and component $j$.
For each table $X_\text{stack}^{(j)}$, we denote its $j$-th largest right singular vector as $\hat{v}_{j, \stacksvd} := v_j(X_\text{stack}^{(j)})$, and define $\hat{V}_\stacksvd = 
    \begin{bmatrix}
        \hat{v}_{1, \stacksvd} & \ldots & \hat{v}_{r, \stacksvd}
    \end{bmatrix}$, which simply forms $\{\hat{v}_{j, \stacksvd}\}$ into a $d \times r$ matrix.
Note that the columns of $\hat{V}_\stacksvd$ will be asymptotically orthogonal. 
In the case where $\theta_{ij}$ follow a common ordering across tables this algorithm is simple to analyze, and yields our results in \Cref{sec:single_cell}.
We defer the more general result (\Cref{thm:rank_r_stacksvd}) and its analysis to \Cref{sec:fullySharedRankr_stacksvd}.
}

{
\section{Practitioners guide}\label{sec:practitioners_guide}

Our theoretical results provide a nuanced picture of the algorithmic landscape of data integration. As we have seen, while weighted \stacksvd is uniformly better when all matrices share the same components or programs, when this assumption is violated and each matrix only spans some subspace of the overall subspace, weighted \svdstack can dominate.
This assumption can be hard to test in the presence of noise, and low signal components.
Thus, if practitioners believe that the subspaces are aligned and the model holds, then weighted \stacksvd should be used. 
This allows for maximal detection of low signal strength components. However, if there is potential for the subspaces to be unaligned, but the signal strength is strong for each component, then there is minimal loss in using weighted \svdstack.}

\subsection{Estimating unknown signal strength} \label{sec:estimating_theta}
In our construction of optimal weights, and instance-wise comparison of methods, we have assumed that all $\theta_i$ were known.
However, in practice, it must often be estimated from the data. If $\theta_i^4 > c_i$, a consistent estimator of $\theta_i$ can be constructed by correcting the leading singular value for the bias due to the high-dimensional noise \cite{10.3150/19-BEJ1129,liu2023asymptotic, hong2023optimally}. However, if $\theta_i$ is below this threshold of detectability, it is a priori unclear if it is possible to estimate $\theta_i$ to any degree of accuracy. In particular,
just looking at $X_i$ on its own will not yield any information about $\theta_i$, as the singular value for the signal spike is absorbed into the bulk spectrum \cite{cai2020limiting,dornemann2025tracy}.
As we show, however, by leveraging information from other tables with strong signal we are still able to estimate $\theta_i$.

Our approach for estimating $\theta_i$ that are marginally below the threshold of detectability only requires the existence of at least one table with sufficiently strong signal strength.
Without loss of generality we assume that $X_1$ has $\theta_1^4 > c_1$, and show how to estimate $\theta_2$ despite the possibility that $\theta_2^4 < c_2$.
This is accomplished by performing an SVD of $X_1$, yielding $\hat{v}_1$ and $\sigma_1(X_1)$, where $|\langle v, \hat{v}_1 \rangle|^2 \pto \beta_1^2$. Then, we can estimate $\hat{\beta}_1$ consistently just from $X_1$ by processing $\sigma_1(X_1)$; see \Cref{app:theta_est_proof} for its explicit expression.
Using this, we can estimate $\theta_2$ as:
\begin{equation} \label{eq:theta_estimation}
    \hat{\theta}_2 = \frac{1}{\hat{\beta}_1} \sqrt{\|X_2 \hat{v}_1\|_2^2 - c_2}.
\end{equation}
We prove that this estimator is consistent in the following proposition.

\begin{prop} \label{thm:theta_est}
    Consider two tables $X_1,X_2$ following \Cref{assum:general_noise,assum:main}, where $\theta_1^4 > c_1$.
    Then, $\hat{\theta}_2$ in \Cref{eq:theta_estimation} is a consistent estimator of $\theta_2$, i.e. $\hat{\theta}_2 \pto \theta_2$.
\end{prop}
We defer the proof of this claim to \Cref{app:theta_est_proof}.
Essentially, this estimation is possible because while the singular value corresponding to $v$ falls in the bulk, there is still $\theta_2$ ``excess signal strength'' in this direction.
Thus, $\|X_2 v\|_2^2 \ip c_2 + \theta_2^2$, where for a fixed vector $y$ not aligned with $v$, independent of $X_2$, we would have that $\|X_2 y\|_2^2 \ip c_2$.
Leveraging the identifiable signal in $X_1$, we can successfully estimate $\theta_2$.
As we show via numerical simulations (\Cref{fig:mse_theta2}), this estimator is quite accurate in practice. {
The estimator requires that at least one table has signal strength exceeding the detection threshold. Otherwise, it is not possible to estimate the optimal weights and thus unweighted or binary weighted \stacksvd should be used (additional discussion in \Cref{sec:binarystacksvd_disc}).}

{
\subsection{Estimating noise variance} \label{sec:estimating_noise_variance}

In our model, we assumed that the noise was i.i.d., with per entry variance $\sigma^2=1/d$, mean zero, and finite fourth moment. Although the fourth moment assumption holds under most reasonable noise models, on real data the true noise scale $\sigma$ is unknown and may need to be estimated. In some cases (e.g., count data), estimation of $\sigma$ can be avoided by transforming the data to satisfy unit variance \cite{bartlett1947use}. Otherwise, we recommend the approach of \cite{josse2012selecting}, which estimates $\sigma$ using the Frobenius norm of the residual matrix:
\begin{equation*}
    \hat{\sigma}_i^2 = \frac{\|X_i - \hat{U}_i \hat{\Theta}_i \hat{V}_i^\top\|_F^2}{n_i d - r(n_i + d - r)}.
\end{equation*}
Here, we subtract out our estimated low rank structure in the numerator, and correct for the effective degrees of freedom in the denominator.
Note that simpler methods such as taking $\hat{\sigma}_i^2 = \|X_i\|_F^2 / (n_i d)$ are also asymptotically consistent \cite{hong2023optimally}.
After estimating $\sigma_i^2$ for table $i$, we divide every entry of $X_i$ by $\hat{\sigma}_i\sqrt{d}$ to obtain the appropriate scaling our method assumes.
}

\section{Application to single-cell data} \label{sec:single_cell}

In this section we demonstrate the practical utility of our theoretical analysis through simulations based on single-cell RNA sequencing (scRNA-seq) data.
Recent scRNA-seq methodologies have aimed to address the problem of \textit{ambient} RNA, contaminating transcripts present in solution that are not associated with a particular cell \cite{young2020soupx, fleming2023unsupervised}.
We hypothesized that the weighted versions of \stacksvd and \svdstack could be used to address this problem by assigning lower weights to samples with higher contamination of ambient RNA.

\begin{figure}[t]
    \centering
    \includegraphics[scale=0.6]{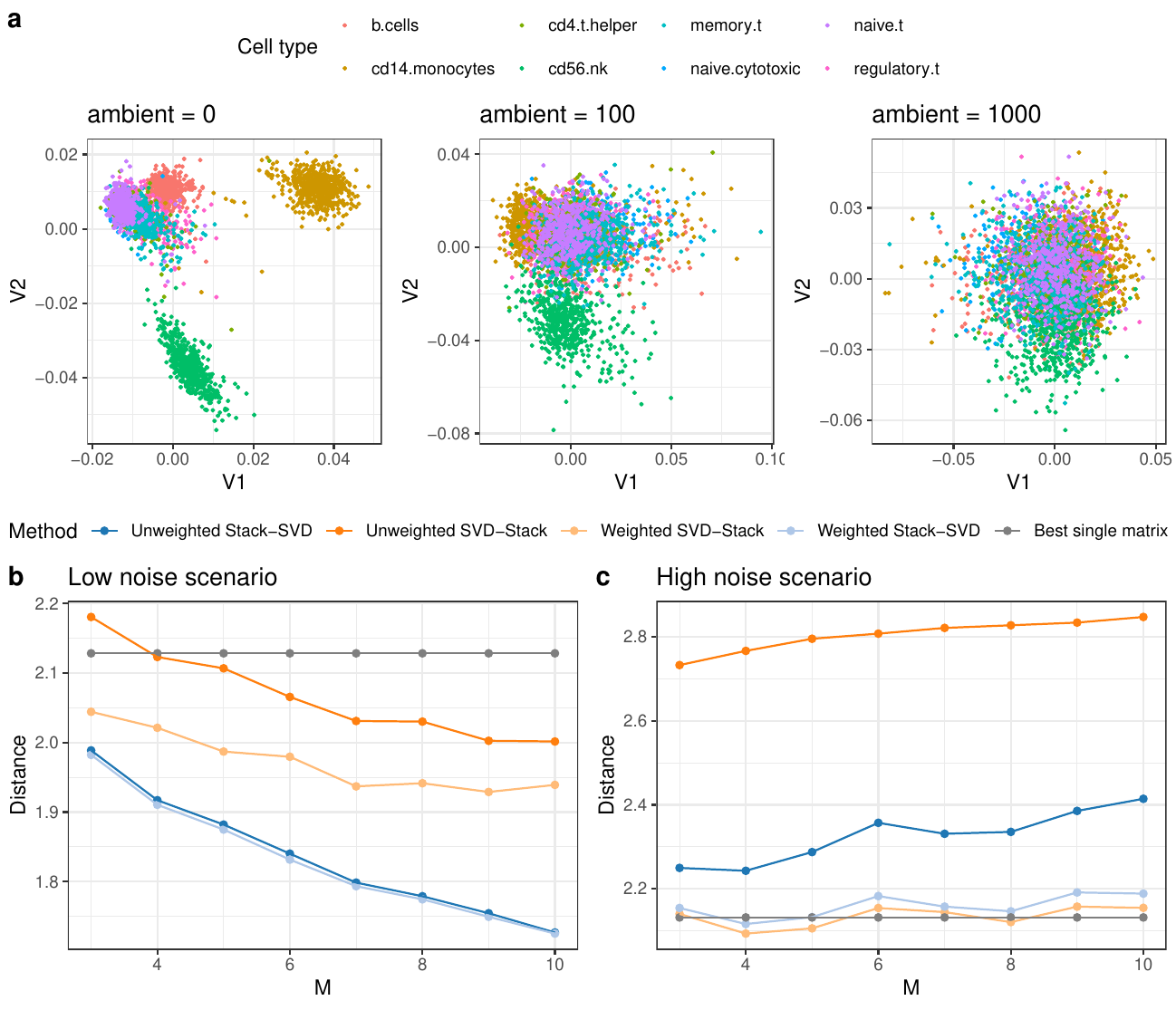}
    \vspace{-.2cm}
    \caption{\textbf{a.} The top two left singular vectors obtained from the (transformed) count matrices for three different levels of (artificially) introduced ambient RNA. The color of each cell corresponds to the cell-type assigned by \cite{zheng2017massively}. \textbf{b.} and \textbf{c.} The low and high noise scenarios were constructed using the procedure described in the main text. The performance of each method was assessed by computing the distance between projection matrices $\| \Pi_{\hat{V}} - \Pi_{V_{\text{true}}} \|_F$.}
    \label{fig:ambient_rna}
    \vspace{-.4cm}
\end{figure}

To test this, we conducted a semi-synthetic analysis based on a real scRNA-seq dataset consisting of a mixture of 8 known cell types \cite{zheng2017massively}.
The data $Y \in \R^{n \times d}$ consists of $n = 3798$ labeled cells and $d = 1572$ genes.
The entry $Y_{ij} \in \Z_{\geq 0}$ encodes the number of unique molecular identifier counts (UMI) corresponding to gene $j$ in cell $i$. 

Existing scRNA-seq research suggests that entries of $Y$ can be modeled as following a Poisson distribution \cite{sarkar2021separating}. Under this assumption, $X = 2\sqrt{Y/d}$ has variance approximately $1/d$ \citep{bartlett1947use}. Centering the columns of $X$ and performing an SVD yielded $r = 10$ components with singular value above the $1 + \sqrt{c}$ threshold (Fig \ref{fig:tenx_sv}).
We randomly split cells into a training and validation set to obtain $Y_{\text{train}}$ and $Y_{\text{test}}$, from which we can apply the variance stabilizing transformation to obtain $X_{\text{train}}$ and $X_{\text{test}}$.
From this, we defined $V_{\text{true}}$ to be the leading right singular vectors of $X_{\text{test}}$.
To test the performance of different data integration methods in the presence of ambient RNA, we randomly partitioned $Y_{\text{train}}$ row-wise to create $M$ datasets $Y_1, \ldots, Y_M$ and simulated the presence of ambient RNA by adding independent $\text{Poisson}(\lambda_i)$ noise to each entry of $Y_i$.
To stabilize the variance we set $X_i = 2 \sqrt{Y_i/d}$ as before.
For visualization, as $\lambda_i$ increases from $0$ to $1000$ with $M = 3$, the separation between the cell types visible in the leading two left singular vectors decreases (Fig \ref{fig:ambient_rna}).

We assessed the performance of each method by computing $\| \Pi_{\hat{V}} - \Pi_{V_{\text{true}}} \|_F$, where $\Pi_{(\cdot)}$ denotes a projection onto the column space of the given matrix. 
We compared both the weighted and unweighted versions of \stacksvd and \svdstack and considered the setting where the weights are estimated from the available data. {
We used the variant of weighted \stacksvd that assumes completely aligned singular subspaces}.
In addition, we also compared the ``lowest-noise only'' method which only uses the SVD of the $X_i$ corresponding to the lowest amount of ambient RNA (although in practice this is typically not known).
As discussed, we partitioned the training data to produce $10$ smaller datasets and varied $M$ by choosing the number of smaller datasets to include in the analysis.
We considered two choices for the amount of ambient RNA $\lambda_i$.
In the ``low noise scenario'', $\lambda_1 = 10$ and $\lambda_i = 50$ for $i \geq 2$, indicating a scenario where all data matrices provide some signal.
In the ``high noise scenario'', $\lambda_1$ is also set to $10$ but $\lambda_i  = 1000$ for $i \geq 2$ which washes out all relevant biological signal from the remaining matrices.
In the low noise scenario (\Cref{fig:ambient_rna}b), the weighted and unweighted \stacksvd methods outperform the \svdstack methods, showing their increased power in the regime where all matrices have relevant signal.
In the high noise scenario (\Cref{fig:ambient_rna}c), the performance of both unweighted \stacksvd and \svdstack degrades as $M$ increases, showing that both methods are negatively affected by the presence of null tables.
The weighted methods, however, are able to downweight the noisy tables and preserve the relevant signal as $M$ increases, retaining essentially constant performance.

\section{Discussion}\label{sec:discussion_future_work}

In this work, we have provided a rigorous random matrix theory analysis of two popular families of data integration strategies, \stacksvd and \svdstack, for inferring shared latent structure across multiple high-dimensional datasets.
Our results derive optimal weighting procedures, as well as closed-form expressions for their asymptotic performances, in terms of the inner product between the estimated and true shared components, of both unweighted and optimally weighted procedures.
In particular, we have shown that when signal strengths vary across data sources, weighting can substantially improve estimation accuracy.
{Conducting a careful instance-dependent analysis, we showed that optimally weighted \stacksvd strictly dominates even optimal weighting of the commonly used \svdstack when the signal is fully shared.
We showed how our framework naturally extends to the case of multiple shared (or partially shared) components, where the picture becomes more complex, and \stacksvd no longer uniformly dominates once dataset-specific signals are misaligned.
To help manage this complexity, we developed a practitioners guide, to assist with the practical implementation of our theoretical recommendations.
This will help broaden our method's applicability to complex, multimodal datasets.}

These findings not only clarify the theoretical underpinnings of these two approaches but also offer practical guidance for designing data integration methods in high-dimensional genomic and multimodal applications.
We go beyond the coarse minimax analysis to characterize the instance optimality of different methods, additionally allowing for $n_i$ that differ by constant factors which are lost in minimax analyses.
This work provides justification for the ubiquitous construction of large data atlases particularly in genomics \cite{weinstein2013cancer,gtex2015genotype,regev2017human,the2022tabula}, proving that, at least in the linear setting with fully shared latent structure, integrating data first and then performing inference is better than first performing inference and only then integrating the results.

There are several exciting avenues of future research building on this.
First, our analysis assumed that $M$ is finite, whereas the behavior of the estimators could change when $M$ grows in tandem with $n$ and $d$.
{Further, in this work we assumed that noise was homoskedastic, or could be corrected by a variance stabilizing transform. Recent works tackle more general noise models through e.g. biwhitening approaches \cite{landa2022biwhitening, chardes2025random}. Extending our analysis to such noise models, or studying the interplay between correcting for homoskedasticity and downstream inference, are a valuable direction of future work. Additionally, given the generality of our work to the partially aligned setting, it would be interesting to contrast our results with those of the widely used joint and individual variation explained (JIVE) model \citep{lock2013joint}.}

{
In the rank $1$ exactly aligned setting, our results show that weighted \stacksvd achieves a uniformly higher performance than weighted \svdstack. A natural question for future work is to determine whether weighted \stacksvd remains optimal among a broader class of data integration procedures, including e.g. Procrustes-based methods \cite{ma2024principled}, or whether it already attains fundamental information theoretic lower bounds \cite{perry2018optimality}. One possible notion of optimality is to characterize the largest asymptotic performance $f\opt(\theta, c)$ that is achievable by an estimator uniformly over all sequences of singular vectors $u_i^{(d)}$ and $v^{(d)}$. A complementary direction is to consider structured parameter classes. For example, imposing assumptions such as sparsity on the singular vectors permits estimators that exploit this structure to outperform classical spectral methods \cite{vu2013minimax,feng2022unifying}, as we show in a simple example in \Cref{sec:sparse_maxrow}. Extending such approaches to the multi-table setting is a promising direction for future work.
 }

\bibliographystyle{imsart-number}
\bibliography{refs}
\clearpage
\newpage

\begin{appendix}

\renewcommand{\thefigure}{S\arabic{figure}}
\renewcommand{\thetable}{S\arabic{table}}
\renewcommand{\theequation}{S\arabic{equation}}
\setcounter{figure}{0}    
\setcounter{table}{0}     
\setcounter{equation}{0}  
\addtocontents{toc}{\protect\setcounter{tocdepth}{2}}

\begin{center}
{\LARGE  \textbf{Supplementary Material}}
\end{center}
\vspace{1em}

\setcounter{tocdepth}{3}
\addtocontents{toc}{\protect\setcounter{tocdepth}{3}}
\begingroup
\renewcommand{\contentsname}{Supplementary Table of Contents}
\tableofcontents
\endgroup

\clearpage
\newpage

\section{Proofs for \svdstack}\label{app:weightedSVDStackProof}

We compute the asymptotic performance of \svdstack by leveraging the single-matrix limits in \Cref{prop:single_table}.
This analysis focuses on the stacked matrix of per-table estimates $\tilde{V}$.
Classical tools make it difficult to directly analyze the $d$-dimensional right singular vector $\hat{v}_\svdstack$, as $d \to \infty$, so we instead study the $M$-dimensional left singular vector of $\tilde{V}$.

\subsection{Preliminaries}

To analyze \svdstack (starting with the unweighted case), we begin by optimizing over vectors in the range of $\tilde{V}$.
We show via a quadratic program that for any vector in the range of $\tilde{V}$, its squared inner product with $v$ can be no more than $\frac{S}{S+1}$.
This upper bounds the performance of \svdstack, as the principal right singular vector output by \svdstack must be in this subspace.
Next, we show that with our proposed weighting scheme, we can successfully attain this performance.
This proves the optimality of our \svdstack weights, and the resulting performance.

\subsubsection{\Cref{lem:delocalization} and its proof}
We begin by deriving the following \textit{delocalization} result, which we state for $X_1,X_2$, but holds for each pair of tables $(i,j)$ with $i\neq j$.
    \begin{lemma} \label{lem:delocalization}
    Let $\hat{v}_i$ be the top right singular vector of $X_i= \theta_i u_i v^\top + E_i$, $X_i \in \R^{n_i \times d}$ for $i=1,2$, where $X_1$ and $X_2$ are independent. Then, under \Cref{assum:general_noise,assum:main},
    \begin{equation*}
    |\langle \hat{v}_1, \hat{v}_2 \rangle|^2 \pto \beta_1^2 \beta_2^2.
    \end{equation*}
    \end{lemma}
    
\begin{proof}
We prove this claim by analyzing the projection of $\hat{v}_2$ in the orthogonal complement of $v$.
We simplify the inner product as:
\begin{equation}
    \hat{v}_1^{\top} \hat{v}_2 = \hat{v}_1^{\top} v v^{\top} \hat{v}_2 + \hat{v}_1^{\top} (I - vv^{\top}) \hat{v}_2,
\end{equation}
where as before, we assume without loss of generality that $\langle \hat{v}_i, v \rangle \geq 0$ because \svdstack is invariant to sign.
Then $\hat{v}_1^{\top} vv^{\top} \hat{v}_2 \ip \beta_1 \beta_2$ and so the result follows by showing
\begin{equation}
    \hat{v}_1^{\top} (I - vv^{\top}) \hat{v}_2 \ip 0.
\end{equation}
We define $\hat{w}_2 = (I - vv^{\top}) \hat{v}_2 / \|(I - vv^{\top}) \hat{v}_2\|_2 \in S^{d-1}$ (provided that $\hat{v}_2 \notin \text{span}(v)$, otherwise $\hat{w}_2$ may be defined as an arbitrary vector in $\text{span}(v)^{\perp}$),  and note that it is sufficient to prove $\hat{v}_1^{\top} \hat{w}_2 \ip 0$ because $\|(I - vv^{\top}) \hat{v}_2 \|_2 \leq 1$ almost surely. Given $\varepsilon > 0$, we wish to show that

\begin{equation}
    \lim_{d \to \infty} \E \left[ \mathds{1}( |\hat{v}_1^{\top} \hat{w}_2| > \varepsilon) \right] =  \lim_{d \to \infty} \E \left[ \E \left( \mathds{1}( |\hat{v}_1^{\top} \hat{w}_2| > \varepsilon) | \hat{w}_2 \right) \right] = 0.
\end{equation}
Since $\hat{v}_1$ and $\hat{w}_2$ are independent, the inner conditional expectation can be defined as the random variable such that
\begin{equation}
    \omega \mapsto \mathbb{P}( |\hat{v}_1^{\top} \hat{w}_2(\omega)| > \varepsilon),
\end{equation}
for $\omega \in \Omega$ (the proof of this is given by Example 4.1.7 of \citep{durrett2019probability}). Applying Theorem 1 (part 2) of \citep{liu2023asymptotic} and noting that $\hat{w}_2(\omega) \in \text{span}(v)^{\perp}$ for all $\omega$ and $d$ yields that 
\begin{equation}
    \lim_{d \to \infty} \mathbb{P}( |\hat{v}_1^{\top} \hat{w}_2(\omega)| > \varepsilon) = 0.
\end{equation}
In other words, the random variable $\E \left( \mathds{1}( |\hat{v}_1^{\top} \hat{w}_2| > \varepsilon) | \hat{w}_2 \right)$ converges pointwise almost surely to $0$. The desired result then follows by applying the dominated convergence theorem. 
\end{proof}

\subsubsection{Entrywise convergence of eigenvector}

We study the covariance matrix $\tilde{V}\tilde{V}^\top$, whose principal eigenvector corresponds to $\hat{v}_\svdstack$.
Recall that $\tilde{V}$ is the unweighted matrix:
\begin{align*}
    \tilde{V} &= 
     \begin{bmatrix}
        \hat{v}_1^\top\\ \vdots \\ 
        \hat{v}_M^\top
    \end{bmatrix},\\
    \tilde{V}_w &=  \diag(w)\tilde{V}\\
    &=\begin{bmatrix}
        w_1\hat{v}_1^\top\\ \vdots \\ 
        w_M\hat{v}_M^\top
    \end{bmatrix}.
\end{align*}
Without loss of generality, we assume throughout that $ \tilde{V} v \ge 0$ entrywise.
Leveraging \Cref{lem:delocalization}, we can show that $\tilde{V}\tilde{V}^\top \pto A_\beta$, where the convergence is entry-wise.

\begin{align*}
A_\beta
&= \begin{bmatrix}
        1 & \beta_1\beta_2 & \ldots\\
        \beta_1 \beta_2 & 1 & \ldots\\
        \vdots & \vdots & \ddots
    \end{bmatrix}\\
    &= \beta\beta^\top +\diag(1-\beta^2).
\end{align*}

We now show that $v_{\text{max}}(\tilde{V}\tilde{V}^\top)$ converges entrywise to $v_{\text{max}}(A_\beta)$.

\begin{lemma} \label{lem:entrywise_conv_eigenvec}
    Suppose that $A_\beta$ has a unique largest eigenvalue $\lambda_1$ with $x =v_{\text{max}}(A_\beta)\in \R^M$, where without loss of generality $\langle x, \hat{x} \rangle \geq 0$. Then $\hat{x} = v_{\text{max}}(\tilde{V}\tilde{V}^\top)$ satisfies $(\hat{x})_m \pto x_m$ for all $m \in [M]$. 
\end{lemma}

\begin{proof}    
Let $\hat{u}_1 = v_{\text{max}}(\tilde{V}\tilde{V}^\top)$ and $u_1= v_{\text{max}}(A_\beta)$.
When $\beta_{2}>0$ the largest eigenvector of $A_\beta$ is unique up to its sign, so we select the eigenvector $\hat{u}_1$ such that $\langle \hat{u}_1, u_1 \rangle \geq 0$.
We then have by a variant of the Davis-Kahan theorem \cite{yu2015useful} that
\begin{equation}
    \| \hat{u}_1 - u_1 \|_2 \leq \frac{ 2^{3/2}\| \tilde{V}\tilde{V}^\top -  A_\beta\|_F}{\lambda_1(A_\beta) - \lambda_2(A_\beta)}.
\end{equation}
Since $\tilde{V}\tilde{V}^\top$ is converging entrywise to $A_\beta$ (\Cref{lem:delocalization}), the numerator is converging to 0, and so $\hat{u}_1 \pto u_1$ element-wise.
\end{proof}

To show that the condition $\beta_2 > 0$ implies that $\lambda_1(A_{\beta})$ is unique, we note that
$A_\beta$ is a rank one perturbation of a diagonal matrix $D=\diag(1-\beta^2)$.
Cauchy's interlacing theorem implies that $\lambda_1(A_\beta) \ge \lambda_1(D) \ge \lambda_2(A_\beta)$.
If $\beta_{2}>0$, then $\lambda_1(A_\beta)>1 \ge \lambda_1(D)$, and the largest eigenvalue of $A_\beta$ has multiplicity 1.
The largest eigenvector of $A_\beta$ can be derived using the Bunch-Nielsen-Sorensen formula \citep{bunch1978rank}, but to show that the largest eigenvalue is larger than 1 it suffices to take $x$ as a normalized indicator over the entries corresponding to $\beta_{1},\beta_{2}$, in which case the Rayleigh quotient $x^\top A_\beta x = 1+\beta_1\beta_2>1$. 
Thus, $\lambda_1 - \lambda_2 \ge \beta_{1}\beta_{2}$.

\subsubsection{Detectability threshold} \label{sec:svdstack_threshold}

Analyzing the detectability threshold of \svdstack, we partition our space into 3 possible options.

\vspace{.1cm}\noindent\textbf{Case 1:} for our primary analysis with \svdstack, our condition is that $\beta_{2}>0$, where multiple matrices have detectable signal.
Then, unweighted \svdstack works as written as $A_\beta$ has a unique largest eigenvalue, which has the stated squared correlation with $v$.

\vspace{.1cm}\noindent\textbf{Case 2:}
On the other extreme, when $\beta_{1}=0$ ($\beta$ is the all zeros vector), we show that unweighted (or weighted) \svdstack has an asymptotic performance of 0.
Since $\hat{v}_\svdstack$ lies within the row space of $\tilde{V}$, we may write 
\begin{equation*}
    \hat{v}_\svdstack = \tilde{u}_{1} \hat{v}_1 + \ldots + \tilde{u}_M \hat{v}_M,
\end{equation*}
where $\tilde{u}$ is the principal left singular vector of $\tilde{V}$, with $\|\tilde{u}\|=1$.

From the single-table results in \Cref{prop:single_table}, we know that $v^\top\hat{v}_i \pto \beta_i=0$ for all $i$.
Combining this with the fact that $|\tilde{u}_i| \leq 1$ for all $i$ yields:
\begin{equation*}
    | \langle \hat{v}_\svdstack, v \rangle | \leq \sum_{i=1}^M  |\langle \hat{v}_i, v \rangle| \pto 0.
\end{equation*}

\vspace{.1cm}\noindent\textbf{Case 3:}
The most analytically troublesome edge case is when $\beta_{1} > 0$ but $\beta_{2}=0$, as then $A_\beta = I$, and so asymptotically the signal is indistinguishable from the noise solely looking at $\hat{V}$.
In practice, since $\beta_i$ is directly computable for each table, the one table with $\beta_1>0$ can be identified and its estimated $\hat{v}_1$ used as $\hat{v}_\svdstack$, trivially yielding a performance of $\beta_1^2$ with this slight modification to the stated algorithm.
Analyzing unweighted \svdstack without this fix is analytically challenging, with the performance converging to a random variable rather than a constant.

\subsection{Unweighted \svdstack: proof of \Cref{thm:svd_stack_general}}\label{app:SVD_stack_proofs}
Leveraging \Cref{lem:entrywise_conv_eigenvec}, we are able to characterize the asymptotic performance of unweighted \svdstack.

\begin{proof}[Proof of \Cref{thm:svd_stack_general}]
In the unweighted setting, $\hat{v}_\svdstack = \tilde{V}^\top y$ for $y = u_{\text{max}}(\tilde{V})$, up to normalization.
This leads to:
\begin{align*}
    \hat{v}_\svdstack &= \frac{\tilde{V}^\top u_{\text{max}}(\tilde{V})}{\|\tilde{V}^\top u_{\text{max}}(\tilde{V})\|}.
\end{align*}

Examining the performance of \svdstack requires taking the inner product with $v$:
\begin{align*}
    \left| \langle v, \hat{v}_\svdstack \rangle \right|^2
    &= \frac{\left| \left\langle v, \tilde{V}^\top u_{\text{max}}(\tilde{V}) \right\rangle \right|^2}{\sigma_\text{max}^2(\tilde{V})}.
\end{align*}
The denominator is equivalently $\lambda_\text{max}(\tilde{V}\tilde{V}^\top) \pto \lambda_\text{max}(A_\beta)$ by Weyl's inequality.
The numerator simplifies by noting that $\tilde{V} v \pto \beta$:
\begin{align*}
    \langle v, \tilde{V}^\top u_{\text{max}}(\tilde{V})\rangle
    &= \langle \tilde{V} v, u_{\text{max}}(\tilde{V})\rangle\\
    &\pto \langle \beta, v_{\text{max}}(A_\beta) \rangle.
\end{align*}
The last line uses the continuous mapping theorem, as this is a finite sum of products of pairs of random variables which are each converging to constants.
Again applying the continuous mapping theorem to the ratio yields the desired result:
\begin{equation} \label{eq:svd_stack_power_x}
    \left| \langle v, \hat{v}_\svdstack \rangle \right|^2
    \pto \frac{\left| \langle \beta, v_{\text{max}}(A_\beta) \rangle\right|^2}{\lambda_\text{max}(A_\beta)}.
\end{equation}
\end{proof}

To prove \Cref{thm:simple_thm1}, we can now specialize this result to the case where $c_i = c_0$ and $\theta_i = \theta_0$ for all $i$.
\begin{proof}[Proof of \Cref{thm:simple_thm1}]
    When $\theta_0^4 > c_0$, the condition $\beta_{2} > 0$ is satisfied and the result follows by using \Cref{thm:svd_stack_general} and directly computing:
    \begin{align*}
        A_\beta &= \diag(1- \beta_0^2) + \beta_0^2 \b{1}\b{1}^\top,\\
        v_{\text{max}}(A_{\beta}) &= \frac{1}{\sqrt{M}} \b{1},\\
        \lambda_{\text{max}}(A_{\beta}) &= 1 + (M-1) \beta_0^2, \\
        \beta^\top v_{\text{max}}(A_{\beta}) &= \sqrt{M} \beta_0.
    \end{align*}
Thus the power of unweighted \svdstack when all $\theta_i$ and $c_i$ are equal and $\theta_0^4 > c_0$ is:
\begin{align*}
    \left| \langle v, \hat{v}_\svdstack \rangle \right|^2
    \pto \frac{M\beta_0^2}{1 + (M-1) \beta_0^2} = \frac{M\left(\theta_0^4 - c_0\right)}{M\theta_0^4 + \theta_0^2 - (M-1)c_0},
\end{align*}
which is equivalent to the stated result.
\end{proof}

\subsection{Optimally weighted \svdstack: Proof of \Cref{thm:svdstack_weighted}}
With weighting, the final performance expression simplifies for \svdstack.

{
\begin{proof}[Proof of \Cref{thm:svdstack_weighted}]
    Note that this result is a special case of the more general \Cref{thm:gen_rank_weight_svdstack}, although we include a self-contained proof here. Defining $W = \diag(w)$, we may write 
    \begin{equation*}
        (\hat{v}_{\svdstack}(w)^{\top} v)^2 = \frac{(v_{\max}(W \tilde{V} \tilde{V}^{\top} W)^{\top} W \tilde{V} v)^2}{\lambda_{\max}(W \tilde{V} \tilde{V}^{\top} W)}.
    \end{equation*}
    Provided that $W A_{\beta} W$ has a unique largest eigenvalue, we have by an argument similar to Lemma \ref{lem:entrywise_conv_eigenvec} that $\tilde{V} v \pto \beta$ (\Cref{prop:single_table}) and $\tilde{V}\tilde{V}^\top \pto A_\beta$, and consequently that $v_{\max}(W \tilde{V} \tilde{V}^{\top} W) \pto v_{\max}(W A_\beta W)$ (up to a sign flip) and $\lambda_{\max}(W \tilde{V} \tilde{V}^{\top} W)\pto \lambda_{\max}(W A_\beta W)$. Then, by the continuous mapping theorem, we have:
    \begin{equation}
        (\hat{v}_{\svdstack}(w)^{\top} v)^2 \ip \frac{(v_{\max}(W A_{\beta} W)^{\top} W \beta)^2}{\lambda_{\max}(W A_{\beta} W)} = (x^{\top} \beta)^2,
    \end{equation}
    where 
    \begin{equation}
        x = \frac{W^{\top} v_{\max}(W A_{\beta} W)}{\sqrt{\lambda_{\max}(W A_{\beta} W)}} \in \R^M.
    \end{equation}
    To make $x$ unique, we may assume that the first non-zero entry of $v_{\max}(W A_{\beta} W)$ is positive. Note also that $x^{\top} A_{\beta} x = 1$. 

    \medskip 

    The maximizer of $(x^{\top} \beta)^2$ subject to $x^{\top} A_{\beta} x = 1$ is $x\opt = \frac{A_{\beta}^{-1} \beta}{\sqrt{\beta^\top A_\beta^{-1} \beta}}$, which exists as $A_{\beta}$ is positive definite. This is due to Cauchy Schwarz on the $A_\beta$ inner product, or equivalently substituting $y = A_\beta^{1/2} x$. The above calculation therefore demonstrates that it suffices to identify a $w\opt$ such that the corresponding $x$ (as defined above) is $x\opt$ (provided that the $w\opt$ also yields an isolated maximum eigenvalue). 

    \medskip 

    We begin by simplifying $x\opt$. Defining $D = \text{diag}(1-\beta^2)$, we have that $A_\beta = D + \beta \beta^\top$:
    \begin{align*}
        A_\beta^{-1} &= D^{-1} - \frac{D^{-1} \beta \beta^\top D^{-1}}{1 + \beta^\top D^{-1} \beta}.
    \end{align*}
    Recall that $S= \sum_i \frac{\beta_i^2}{1-\beta_i^2}=\beta^\top D^{-1} \beta$, and so $\beta^\top A_\beta^{-1} \beta = \frac{S}{S+1}$.
    Thus, we see that $x\opt$ is:
    \begin{align*}
        x\opt &= \frac{1}{\sqrt{\beta^\top A_\beta^{-1} \beta}}A_\beta^{-1} \beta\\
        &=  \sqrt{\frac{S+1}{S}} \times \frac{1}{1 + \beta^\top D^{-1} \beta} D^{-1} \beta\\
        &= \frac{1}{\sqrt{S(S+1)}} \times \frac{\beta}{1-\beta^2} .
    \end{align*} 
This can be achieved by setting $w_i\opt = 1/\sqrt{1 - \beta_i^2}$, and computing
\begin{align*}
    A_{\beta,w\opt}
    &=W\opt A_{\beta} W\opt \\
    &= (w\opt\beta)(w\opt\beta)^\top + \text{diag}(\left(w\opt\right)^2(1-\beta^2))\\
    &= \left(\frac{\beta}{\sqrt{1-\beta^2}}\right)\left(\frac{\beta}{\sqrt{1-\beta^2}}\right)^\top + I,\\
    v_{\text{max}}(A_{\beta,w\opt}) &= \frac{1}{\sqrt{S}}\frac{\beta}{\sqrt{1-\beta^2}}, \\
    \lambda_{\max}(A_{\beta,w\opt}) &= S + 1.
\end{align*}
The fact that $W\opt A_{\beta} W\opt$ has a unique largest eigenvalue follows by noting $\beta_1 > 0 \Rightarrow S + 1 > 1$, and that the remaining eigenvalues are identically $1$. 
Thus,
\begin{align*}
    |\langle v,\hat{v}_\svdstack\rangle|^2 
    \pto
    \frac{\left(x^\top \beta\right)^2}{x^\top A_\beta x}&= \frac{\left(\beta^\top A_\beta^{-1} \beta \right)^2}{\beta^\top A_\beta^{-1}\beta}\\
    &= \beta^\top A_\beta^{-1}\beta\\
    &= \beta^\top  \left( \frac{1}{S+1} \times \frac{\beta}{1-\beta^2}\right)\\
    &= \frac{S}{S+1}. \numberthis
\end{align*} 
Observe that there are many weightings that can attain this same $\hat{v}_\svdstack$: concretely, taking 
$\tilde{w}_i\opt = \theta_i \sqrt{\frac{\theta_i^2+1}{\theta_i^2+c_i}} \mathds{1}\{\theta_i^4 > c_i\}$ yields the same results. This holds more generally for any weights $\tilde{w}$ agreeing with $1/\sqrt{1-\beta_i^2}$ on tables with $\beta_i>0$ and satisfying $|\tilde{w}_i| \le 1$ on the remaining tables, as then we get $A_{\beta,\tilde{w}} = zz^\top + D'$ with $z = \beta/\sqrt{1-\beta^2}$ and $D'_{ii} = \tilde{w}_i^2(1-\beta_i^2) \le 1$, equal to $1$ where $\beta_i>0$. Since $z$ is supported on the detectable tables, $z/\sqrt{S}$ is still the top eigenvector with eigenvalue $S+1 > 1 \ge D'_{ii}$, and the induced $x$ is unchanged (the $\tilde{w}_i$ on undetectable tables multiply zero entries of $z$)
\end{proof}
}

\blue{Recall that at the beginning of the proof we assumed that $WA_\beta W$ had a unique largest eigenvalue.
Here, we show that even if the limiting performance does not exist for such weightings $w$, we can upper bound the squared inner product with high probability, enabling us to characterize $w\opt$.
To this end, consider an arbitrary weight vector $w\neq 0$, and $W=\diag(w)$.
Here, $A_{\beta,d} = \tilde{V}\tilde{V}^\top$ is invertible with probability tending to one, since $A_{\beta,d} \pto A_\beta$ and $A_\beta$ is positive definite (by \Cref{lem:delocalization}).
Recall that every unit vector in the top right singular subspace of $W\tilde{V}$ (the output of weighted \svdstack) has the form $x=\tilde{V}^\top b$ for some $b\in \R^M$.
Observe that $1=\|x\| = b^\top A_{\beta,d} b$, and we further define $\beta_d = \tilde{V}v$, where $\beta_d \pto \beta$ (by \Cref{prop:single_table}).
Then, we have that for any such $x$:
\begin{align*}
    \langle x,v \rangle^2 
    &= \frac{\langle b, \tilde{V} v\rangle^2}{b^\top A_{\beta,d} b}\\
    &\le \beta_d^\top A_{\beta,d}^{-1} \beta_d, \numberthis \label{eq:nonuniqueweights}
\end{align*}
where in the last line we leveraged the Cauchy-Schwarz inequality in the $A_{\beta,d}$ inner product (i.e. substituting in $y= A_{\beta,d}^{1/2} b$).
Now, observe that the right hand side is independent of our choice of $w$. Hence, for every choice of weights $w\neq 0$, data-dependent or deterministic,
\begin{equation}
    \sup_{w\neq 0} \langle \hat{v}_\svdstack(w),v\rangle^2 \le \beta_d^\top A_{\beta,d}^{-1} \beta_d
\end{equation}
Since $\beta_d^\top A_{\beta,d}^{-1} \beta_d \pto \beta^\top A_\beta^{-1} \beta = \frac{S}{S+1}$, we have that for every $\eps >0$:
\begin{equation}
    \P\left(\sup_{w\neq 0} \langle \hat{v}_\svdstack(w),v\rangle^2 > \frac{S}{S+1} + \eps\right) \to 0.
\end{equation}
For the convergence in probability, we leveraged \Cref{lem:delocalization} and \Cref{prop:single_table}, and the fact that $A_\beta$ is positive definite by the conditions of the Theorem.
Since our chosen weighting $w\opt$ attains this asymptotic limit, it is optimal among all weightings, with no eigengap condition.
}

\section{Proofs for \stacksvd} \label{app:weightedStackSVDProof}

In this appendix, we provide proofs related to \stacksvd.
{
\subsection{Discussion regarding binary-weighted \stacksvd} \label{sec:binarystacksvd_disc}

Recall that binary-weighted \stacksvd assumes that $\theta_i$ are known, or are estimable, which requires at least one $\theta_i$ to be above the threshold of detectability. 
If all $\theta_i$ are unknown and fall below the marginal threshold of detectability, there are two potential options. The first is to simply take $w_i = 1$ for all tables, unweighted \stacksvd, which can yield good performance e.g. if $\theta_i=0.8, c_i=1$ for all $i$.
The second, which is possible if $M$ is small or moderate, is to test all $2^M-1$ possible binary weighted \stacksvd ($w_i \in \{0, 1\}$), and choose the weighting that yields the best performance in \Cref{cor.2} (which is estimable, as it depends only on the largest singular value of the stacked data matrix).
We note that either option will perform at least as well as \svdstack, which will always obtain a performance of $0$ in this setting.
}

\subsection{Proof of weighted \stacksvd (\Cref{thm:stacksvd_weighted})}

To begin, we note that $X_{\text{stack}}$ may be written as $\tilde{u}_0 v^{\top} + \Sigma^{1/2} E \in \R^{n \times d}$ with $n := \sum_{i=1}^M n_i$ and 

\begin{align*}
    \tilde{u}_0 &= \begin{bmatrix}
		\theta_1w_1 u_1\\
		\vdots\\
		\theta_Mw_M u_M
	\end{bmatrix}\in \R^{n}, \\ 
    \Sigma &= \text{diag}(\underbrace{w_1^2,...,w^2_1}_{n_1}, \underbrace{w^2_2,...,w^2_2,}_{n_2}...,\underbrace{w^2_M,...,w^2_M}_{n_M}), \\
    E &= \begin{bmatrix} E_1 \\ \vdots \\ E_M \end{bmatrix} \in \R^{n \times d}. 
\end{align*}

Our goal is to characterize the asymptotic performance $L(w) = |\langle \hat{v}_\stacksvd, v \rangle|^2$ where $\hat{v}_\stacksvd=v_{\text{max}}(X_{\text{stack}})$, and to find a maximizing choice of $w$.
The first part will prove \Cref{prop:stacksvd_general} as a special case by evaluating $L(1_M)$, where $1_M \in \R^M$ is the vector of all ones. 
Defining $R = \E[X_{\text{stack}}X_{\text{stack}}^\top] = \tilde{u}_0 \tilde{u}_0^{\top} + \Sigma$, the first four assumptions of \cite{liu2023asymptotic} are satisfied when 
\begin{equation}
    \sum_{i=1}^M \frac{c_iw_i^4}{(\gamma_1-w_i^2)^2} < 1 \label{eq:assumption4}
\end{equation}
and $\gamma_1 > \max_{1 \leq i \leq M} w_i^2$, where $\gamma_1$ is the largest eigenvalue of $R$.
When this condition is not met, all the singular values of $X_{\text{stack}}$ will lie within the support of the limiting spectral distribution, and no outlier singular values emerge.
Consequently, by the same argument used in \cite{liu2023asymptotic}, the leading right singular vector of $X_{\text{stack}}$ is asymptotically orthogonal to $v$, implying that $L(w) = 0$. 

\medskip 
Provided that \eqref{eq:assumption4} holds, $L(w)$ may be expressed in terms of the eigenvectors and eigenvalues of $R$.
The following result describes the eigenstructure of $R$.

\begin{lemma}\label{lem:secular_equation}
    At least $n_i - 1$ eigenvalues of $R$ are equal to $w_i^2$ for $i \in [M]$, and the remaining eigenvalues are given by the roots of the secular equation 
    \begin{equation*}
        f(\lambda) = 1 + \sum_{j=1}^M \frac{\theta_j^2 w_j^2}{w_j^2 - \lambda}.
    \end{equation*}
    Under condition \eqref{eq:assumption4}, the largest eigenvalue $\gamma_1$ is unique and a corresponding eigenvector $\xi_1$ is given by 
    \begin{equation*}
        \xi_1 \propto (\Sigma - \gamma_1 I)^{-1} \tilde{u}_0.
    \end{equation*}
\end{lemma}

\begin{proof}
    For $\lambda \notin \{w_1^2, \ldots, w_M^2 \}$, we have 
    \begin{align*}
        \det(\tilde{u}_0 \tilde{u}_0^{\top} + \Sigma - \lambda I) &= (1 + \tilde{u}_0^{\top} (\Sigma - \lambda I)^{-1} \tilde{u}_0) \det(\Sigma - \lambda I) \\
        &= \left( 1 + \sum_{j=1}^M \frac{\theta_j^2 w_j^2}{w_j^2 - \lambda} \right)\det(\Sigma - \lambda I),
    \end{align*}
    establishing by the matrix determinant lemma \cite{golub2013matrix} that the roots of $f$ are eigenvalues of $R$.
    Furthermore, note that any vector in $\text{span}(e_1, \ldots, e_{n_1}) \cap \text{span}(\tilde{u}_0)^{\perp}$ is an eigenvector with eigenvalue $w_1^2$.
    In particular, the dimension of the eigenspace corresponding to $w_1^2$ must be at least $n_1 - 1$, implying that the algebraic multiplicity of $w_1^2$ is at least $n_1 -1$. 

    \medskip 

    An eigenvector $\xi_1$ corresponding to $\gamma_1$ satisfies 
    \begin{equation*}
        (\tilde{u}_0 \tilde{u}_0^{\top} + \Sigma) \xi_1 = \gamma_1 \xi_1.
    \end{equation*}
    Under condition \eqref{eq:assumption4}, $\gamma_1 > \max_i w_i^2$ and thus $(\Sigma - \gamma_1 I)^{-1}$ exists. Rearranging the above gives 
    \begin{equation*}
        \xi_1 \propto (\Sigma - \gamma_1 I)^{-1} \tilde{u}_0 .
    \end{equation*}

    Finally, the uniqueness of $\gamma_1$ follows by the interlacing theorem of \citep{bunch1978rank}.
    
\end{proof}

In this setting, Theorem 2 of \cite{liu2023asymptotic} has computed the asymptotic inner product between $\hat{v}$ and an arbitrary sequence of deterministic unit vectors $b^{(d)}$: 
\begin{equation}
    b^\top\hat v\hat v^\top b = \eta_1\frac{b^\top v \tilde{u}_0^{\top} \xi_1\xi_1^\top \tilde{u}_0 v^{\top} b}{\gamma_1} + o_{\pr}(1) \label{eq:weighted_norm},
\end{equation}
where 
\begin{align*}
    \eta_1 &= 1 - \frac{1}{d} \sum_{i=2}^n \frac{\gamma_i^2}{(\gamma_i - \gamma_1)^2} \\
    &= 1 - \sum_{i=1}^M \frac{c_i w_i^4}{(\gamma_1 - w_i^2)^2} + O(1/d).
\end{align*}
Moreover, 
\begin{equation*}
    \xi_1=\frac{1}{C}\begin{bmatrix}
		\frac{\theta_1 w_1}{w^2_1-\gamma_1}	u_1\\
		\vdots\\
		\frac{\theta_M w_M}{w^2_M-\gamma_1}	u_M
	\end{bmatrix},\qquad \text{ where } C= \sqrt{\sum_{j=1}^M \frac{\theta_j^2w_j^2}{(w^2_j-\gamma_1)^2}}.
\end{equation*}
Taking $v$ as $b$ in \eqref{eq:weighted_norm} gives 
\begin{equation*}
    (\hat{v}^{\top} v)^2 \ip \frac{\left(1-\sum_{j=1}^M\frac{c_jw_j^4}{(w^2_j-\gamma_1)^2}\right)\left(\sum_{j=1}^M 	\frac{\theta_j^2 w_j^2}{w^2_j-\gamma_1}\right)^2}{\gamma_1\left(\sum_{j=1}^M \frac{\theta_j^2w_j^2}{(w^2_j-\gamma_1)^2}\right)} = \frac{1-\sum_{j=1}^M\frac{c_jw_j^4}{(w^2_j-\gamma_1)^2}}{\gamma_1\left(\sum_{j=1}^M \frac{\theta_j^2w_j^2}{(w^2_j-\gamma_1)^2}\right)} := L(w).
\end{equation*}
As mentioned, the above result contains unweighted \stacksvd as a special case, which we now formalize.

\begin{proof}[Proof of \Cref{prop:stacksvd_general}]
    Assuming \eqref{eq:assumption4}, we have
    \begin{equation*}
        L(1_M) = \frac{1 - \sum_{j=1}^M \frac{c_j}{(\gamma_1 - 1)^2}}{\gamma_1 \sum_{j=1}^M \frac{\theta_j^2}{(\gamma_1 - 1)^2}}.
    \end{equation*}
    Moreover, $\gamma_1 = \|\theta\|_2^2 + 1$, so the above becomes 
    \begin{equation*}
        \frac{\|\theta\|_2^4 - \|c\|_1}{(\|\theta\|_2^2 + 1) \| \theta \|_2^2}.
    \end{equation*}
    Because $L(1_M) > 0$ if and only if \eqref{eq:assumption4}, the phase transition is equivalently given by
    \begin{equation*}
        \| \theta \|_2^4 > \| c\|_1.
    \end{equation*}
\end{proof}

The following proof derives the optimal weighting and phase transition for \stacksvd.

\begin{proof}[Proof of \Cref{thm:stacksvd_weighted}]

We write the nonzero limiting performance, when the signal outlier exists, as
\begin{equation*}
    L(w)
    =
    \frac{
    1-\sum_{j=1}^M
    \frac{c_j w_j^4}{(\gamma_1-w_j^2)^2}
    }{
    \gamma_1
    \sum_{j=1}^M
    \frac{\theta_j^2 w_j^2}{(\gamma_1-w_j^2)^2}
    }.
\end{equation*}
If the numerator is nonpositive, then the limiting overlap is zero, by the second part of Lemma 6.3 of \cite{hong2023optimally}. Thus the optimal performance is positive if and only if the maximized value of the above expression is positive.

We first show that, without loss of optimality, weights corresponding to zero-signal tables may be set equal to zero.
Define
\begin{equation*}
    \jgz := \{j : \theta_j > 0\}.
\end{equation*}

Suppose that $j \notin \jgz$, so that $\theta_j=0$.
Then the corresponding term does not appear in the secular equation defining $\gamma_1$:
\begin{equation}
    1+\sum_{k=1}^M \frac{\theta_k^2 w_k^2}{w_k^2-\gamma_1}=0.
\end{equation}
Consequently, as long as $\gamma_1>\max_k w_k^2$, the value of $\gamma_1$ is independent of $w_j$.
Examining $L(w)$,
we see that the denominator is independent of $w_j$ when $\theta_j=0$, while the numerator contains the nonpositive term
\begin{equation*}
    -\frac{c_j w_j^4}{(\gamma_1-w_j^2)^2}.
\end{equation*}
Hence $L(w)$ is maximized by setting $w_j=0$.
Therefore an optimal weighting satisfies
\begin{equation*}
    w_j=0
    \qquad \text{for all } j\notin \jgz,
\end{equation*}
and it suffices to optimize over indices in $\jgz$.

For $j\in\jgz$, define
\begin{equation}\label{eq:var_change_wstacksvd}
    p_j
    =
    \frac{\theta_j^2 w_j^2}{\gamma_1-w_j^2}.
\end{equation}
Using the secular equation,
\begin{equation*}
    \sum_{j\in\jgz} p_j = 1,
\end{equation*}
so $p=(p_j)_{j\in\jgz}$ lies on the simplex.
Furthermore,
\begin{equation*}
    \frac{c_j w_j^4}{(\gamma_1-w_j^2)^2}
    =
    \frac{c_j}{\theta_j^4} p_j^2,
\end{equation*}
and
\begin{equation*}
    \gamma_1
    \sum_{j\in\jgz}
    \frac{\theta_j^2 w_j^2}{(\gamma_1-w_j^2)^2}
    =
    1+
    \sum_{j\in\jgz}
    \frac{1}{\theta_j^2} p_j^2.
\end{equation*}
Therefore
\begin{equation*}
    L(w)
    =
    \frac{
    1-\sum_{j\in\jgz}\alpha_j p_j^2
    }{
    1+\sum_{j\in\jgz}\beta_j p_j^2
    },
\end{equation*}
where
\begin{equation*}
    \alpha_j=\frac{c_j}{\theta_j^4},
    \qquad
    \beta_j=\frac{1}{\theta_j^2},
    \qquad
    j\in\jgz.
\end{equation*}



Now let $r$ be a candidate value. Then
\begin{equation}
    L(w) \geq r \iff \sum_{j \in \jgz} (\alpha_j + r \beta_j) p_j^2 \leq 1 - r.
\end{equation}

Over the simplex, the minimum of the left side is
\begin{equation}
    \frac{1}{\sum_{j \in \jgz} \frac{1}{\alpha_j + r \beta_j}}, \quad \text{ when } \quad p_j\opt \propto \frac{1}{\alpha_j + r \beta_j} \text{ for } j \in \jgz.
\end{equation}

So the optimal value $r\opt$ is the unique solution of
\begin{equation*}
    \sum_{j \in \jgz} \frac{1 - r\opt}{\alpha_j + r\opt \beta_j} = 1,
\end{equation*}
i.e.,
\begin{equation*}
    \sum_{j \in \jgz} \frac{(1 - r\opt) \theta_j^4}{c_j + r\opt \theta_j^2} = 1.
\end{equation*}

The left-hand side is strictly decreasing in $r$, so this is a one-dimensional root-finding problem.
Then, for all $j\in\jgz$,
\begin{equation}
    p_j\opt = \frac{(1 - r\opt) \theta_j^4}{c_j + r\opt \theta_j^2},
\end{equation}
and one choice of weights $w_j$ that attains this is, for $\gamma>0$ and $j\in\jgz$:
\begin{equation}
    \left(w_j\opt\right)^2 = \gamma \frac{(1 - r\opt) \theta_j^2}{c_j + \theta_j^2} \propto \frac{\theta_j^2}{c_j+\theta_j^2}.
\end{equation}

To establish the detectability threshold, note that the above calculation shows that the detectability condition in \eqref{eq:assumption4} (stated first) is equivalent to
\begin{align*}
\sum_{i=1}^M \frac{c_iw_i^4}{(\gamma_1-w_i^2)^2} < 1
    &\overset{(1)}{\iff} \sum_{j \in \jgz} \frac{c_j (p_j\opt)^2}{\theta_j^4} < 1\\
    &\overset{(2)}{\iff} r\opt > 0\\
    &\overset{(3)}{\iff} \sum_{j=1}^M \frac{\theta_j^4}{c_j} > 1 \numberthis
\end{align*}
these follow: 
\begin{enumerate}
    \item By definition of $p_j\opt$.
    \item $r\opt=L(w\opt)=\frac{1-\sum_j c_j(p_j\opt)^2/\theta_j^4}{1+\sum_j (p_j\opt)^2/\theta_j^2}$. The denominator is positive, so the sign of $r\opt$ is the same as the sign of $1-\sum_j c_j(p_j\opt)^2/\theta_j^4$. 
    \item $g(r)=\sum_j\frac{(1-r)\theta_j^4}{c_j+r\theta_j^2}$ is strictly decreasing in $r$, with $g(r\opt)=1$, so $r\opt>0\iff g(0)=\sum_j\theta_j^4/c_j>1$.
\end{enumerate}

\noindent
This yields the full result for optimally weighted \stacksvd, where above this detectability threshold,
evaluated at the optimizing $w_j$, 
\begin{equation}
    L(w) = \text{ the unique solution $x \in (0,1)$ of } \sum_{i=1}^M \theta_i^4 \frac{1-x}{c_i + x\theta_i^{2}} = 1.
\end{equation}
\end{proof}
\color{black}

\subsection{MLE interpretation of weighted \stacksvd} \label{app:wstacksvd_mle}

We begin by deriving the form of the weighted \stacksvd estimator:
\begin{align*}
\widehat v_{\mathrm{wStack}}(w)
&= v_{\max}\left(\begin{bmatrix}
w_1 X_1\\
\vdots\\
w_M X_M
\end{bmatrix}\right)\\
&\in \argmax_{\|v\|_2=1} v^\top\!\left(\sum_{i=1}^M w_i^2 X_i^\top X_i\right)v \numberthis \label{eq:weightedStackSVD_form}.
\end{align*}

Now, we take our fixed-effects model $X_i = \theta_i u_i v^\top + E_i$, and retain the same conditional rank-1 signal plus noise, but treat $u_i$ as a latent isotropic random vector, i.e. 
\begin{equation*}
    u_i \sim \mathcal{N}\left(0,\frac{1}{n_i}I_{n_i}\right).
\end{equation*}

\noindent
Marginalizing over \(u_i\), the rows of \(X_i\) are independent Gaussian vectors with covariance
\[
\Sigma_i(v)
=
\frac{1}{d}I_d+\frac{\theta_i^2}{n_i}vv^\top
=
\frac{1}{d}\left(I_d+\frac{\theta_i^2}{c_i}vv^\top\right).
\]
Thus, up to constants independent of \(v\), the marginal log-likelihood is
\[
\ell(v)
=
-\frac12\sum_{i=1}^M
\left(
n_i\log\det\Sigma_i(v)
+
\operatorname{tr}\!\left(\Sigma_i(v)^{-1}X_i^\top X_i\right)
\right).
\]
For \(\|v\|_2=1\), the determinant is
\[
\det \left(\Sigma_i(v)\right)
=
d^{-d}\left(1+\frac{\theta_i^2}{c_i}\right),
\]
and hence is independent of the direction of \(v\).  Moreover, by Sherman--Morrison,
\[
\Sigma_i(v)^{-1}
=
dI_d-\frac{d\theta_i^2}{c_i+\theta_i^2}vv^\top .
\]
Substituting this expression into the likelihood and discarding terms independent of \(v\), maximizing \(\ell(v)\) is equivalent to maximizing
\[
v^\top
\left(
\sum_{i=1}^M
\frac{\theta_i^2}{c_i+\theta_i^2}
X_i^\top X_i
\right)v
\qquad
\text{over } \|v\|_2=1 .
\]
Therefore the marginal MLE is
\[
\widehat v_{\mathrm{MLE}}
=
v_{\max}\!\left(
\sum_{i=1}^M
\frac{\theta_i^2}{c_i+\theta_i^2}
X_i^\top X_i
\right).
\]

This exactly coincides with weighted \stacksvd, as seen in \Cref{eq:weightedStackSVD_form}. Thus the optimal \stacksvd weights arise as the marginal maximum-likelihood weights under a random-effects version of the rank-one model.

\subsection{Suboptimality of SVD-based methods under structural assumptions}\label{sec:sparse_maxrow}
Here, we briefly discuss the suboptimality of SVD-based methods under structural assumptions.
We illustrate this by looking at a single table.
The performance $\beta_i^2$ in \Cref{prop:single_table} is shared by every method we
study, since on a single matrix \stacksvd, \svdstack, and their reweightings all reduce
to the top singular vector. 
This is the optimal estimator for $v$ (via Eckart-Young or \Cref{app:wstacksvd_mle}), however, only when $u_i$ is an unstructured
unit vector: under additional structural assumptions on $u$, a structure-aware estimator can strictly dominate the SVD.
In the extreme case of a $1$-sparse $u = e_k$, selecting the largest-norm row and returning its direction (which is the MLE here, and the analogue of
diagonal thresholding \cite{johnstone2009consistency}) attains
$\beta^2_{\mathrm{max-row}} = \theta^2/(\theta^2+1)$, where $\beta_{\mathrm{SVD}}^2 = \beta_{\mathrm{max-row}}^2 \max(1-c/\theta^4,0)$ for every $\theta>0$, paying no aspect-ratio penalty and exhibiting no detectability threshold.
Optimality is thus a property of the model-estimator pair, and the \stacksvd-versus-\svdstack comparison should be read within the unstructured model in which both operate; exploiting known structure such as sparsity \cite{vu2013minimax,feng2022unifying} is a separate, complementary direction.

\color{black}

\section{Proofs for relative-performance analysis}

\svdstack and \stacksvd, weighted and unweighted, yield several different choices for methods to use.
We have shown that there exist many scenarios where one can outperform another, visualizing this in \Cref{fig:tikz_comparison}.
In the end, however, weighted \stacksvd is uniformly most powerful.
In this appendix we prove the instance-dependent improvements of different methods.
These relationships are shown in \Cref{fig:tikz_comparison}, and tabulated in \Cref{tab:tikz_comparison_summary}.

\tikzset{
    dominance arrow/.style={
        line width=2pt,
        -triangle 45,
        color=orange
    }
}

\tikzset{
    defn arrow/.style={
        line width=2pt,
        -triangle 45,
        color=blue
    }
}

\tikzset{
    some arrow/.style={
        line width=2pt,
        -triangle 45,
        dashed,
        color=forestgreen
    }
}

\begin{figure}[h]
    \centering
\newcommand{\hspacing}{3} 
\newcommand{\vspacing}{3.5} 
\begin{tikzpicture}[
    font=\sffamily,
    node distance=1.2cm,
    >=Latex,
    align=center,
    every node/.style={minimum height=.75cm}
]
\node (wstackSVD) at (-\hspacing,{2*\vspacing}) {weighted\\\stacksvd};
\node (binarystackSVD) at (-\hspacing,{\vspacing}) {binary\\\stacksvd};
\node (stackSVD) at (-\hspacing,0) {\stacksvd};
\node (wsvdstack) at (\hspacing,{\vspacing}) {weighted\\\svdstack};
\node (svdstack) at (\hspacing,0) {\svdstack};

\draw[defn arrow] (wstackSVD) -- (binarystackSVD) node[midway, sloped, align=center] {For all $\theta,c$\\By defn};
\draw[defn arrow] (binarystackSVD) -- (stackSVD) node[midway, sloped, align=center] {For all $\theta,c$\\By defn};
\draw[defn arrow] (wsvdstack) -- (svdstack) node[midway, sloped, align=center] {For all $\theta,c$\\By defn};

\draw[dominance arrow] (wstackSVD) -- (wsvdstack) node[midway, sloped, align=center] {For all $\theta,c$\\\Cref{prop:dominance}};

\draw[dominance arrow, dashed] (binarystackSVD) to node[midway, sloped, align=center] {For all $c_i \le 1$\\\Cref{thm:stacksvd_binary_optimal_svd_stack}} (wsvdstack);

\draw[some arrow] (svdstack) to node[pos=.2, sloped, align=center] {Some $\theta,c$\\\Cref{remark:svd_outperform_stack}} (binarystackSVD);

\draw[some arrow] (stackSVD) to node[pos=.2, sloped, align=center] {Some $\theta,c$\\\Cref{remark:stack_outperform_svd}} (wsvdstack);
\end{tikzpicture}

\caption{Comparison of methods in terms of their performance under different conditions. 
Solid lines indicate uniform dominance, while dashed lines indicate improvement for some problem instances.
{\color{blue} Blue} indicates that a relationship is by definition, i.e. weighted \svdstack dominates \svdstack because the former optimizes over many weightings, including the unweighted one.
{\color{orange} Orange} indicates that for an easily-defined class of problem instances, we can prove that one method dominates another; e.g. binary \stacksvd dominates weighted \svdstack when all $c_i \le 1$.
{\color{forestgreen} Green} indicates that there are examples where one method outperforms another (given by algebraic expressions).
Arrows can be followed to chain dependencies together, i.e. weighted \svdstack outperforms \stacksvd on some instances, as weighted \svdstack uniformly dominates \svdstack, which outperforms binary \stacksvd on certain instances (\Cref{remark:svd_outperform_stack}) which uniformly dominates \stacksvd.
} \label{fig:tikz_comparison}
\end{figure}

\subsection{Binary-weighted \stacksvd: proof of \Cref{thm:stacksvd_binary_optimal_svd_stack}} \label{app:binary_weighting}

Here, we prove that a simple binary weighting of \stacksvd, with $w_i = \mathds{1}\left\{ \theta_i^4 > c_i\right\}$, yields improved performance over optimally weighted \svdstack, as long as $c_i \le 1$ for all $i$. 
Note that this is not even necessarily the optimal threshold for binary weights (e.g. consider where all $\theta_i = \sqrt{c_0}$), but it is still enough to outperform optimally-weighted \svdstack.

\begin{proof}[Proof of Proposition~\ref{thm:stacksvd_binary_optimal_svd_stack}]
We want to show that binary-weighted \stacksvd has a strictly better asymptotic performance than optimally-weighted \svdstack when $c_i \le 1$ for all $i$ and $M \ge 2$. 
This requires proving the inequality:
\begin{equation} \label{eq:binaryweighting_real_ineq}
    1 - \left( 1 + \sum_{i=1}^M \frac{\theta_i^4 - c_i}{\theta_i^2 + c_i}\right)^{-1} < \frac{\left(\sum_{j=1}^M \theta_j^2\right)^2 - \sum_{j=1}^M c_j}{\sum_j \theta_j^2 \left(\sum_{j=1}^M \theta_j^2 + 1\right)},
\end{equation}
where the LHS is the performance of optimally-weighted \svdstack, and the RHS is the performance of binary-weighted \stacksvd, assuming without loss of generality that all tables have $\theta_i^4 \ge c_i$, as otherwise this table will not be included by binary-weighted \stacksvd, and will not impact the performance of \svdstack (as it has $\beta_i = 0$).
Now, observe that the function $f(z) = z/(1-z)$ is monotonically increasing over the interval $(0,1)$.
Thus, proving our stated inequality can equivalently be done by first applying the transformation $f(z)$ to each side:
\begin{equation} \label{eq:binaryweighting_simplify_general}
    \sum_{i=1}^M \frac{\theta_i^4 - c_i}{\theta_i^2 + c_i} < \frac{\left(\sum_{j=1}^M \theta_j^2\right)^2 - \sum_{j=1}^M c_j}{\sum_{j=1}^M \theta_j^2 + \sum_{j=1}^M c_j}.
\end{equation}
Let $x_i = \theta_i^2$ and $T = \sum_i x_i$. 
Without loss of generality, we assume that for all tables included in the analysis $\theta_i^4 > c_i$, as otherwise this table will not factor into the performance of either method, which implies $x_i^2 > c_i$.

First, we decompose both sides of the inequality. For the left-hand side (LHS), each term can be written as:
\[
\frac{x_i^2 - c_i}{x_i + c_i} = \frac{(x_i - c_i)(x_i + c_i) + c_i^2 - c_i}{x_i + c_i} = (x_i - c_i) + \frac{c_i(c_i - 1)}{x_i + c_i}.
\]
Summing over all $i$ yields:
\[
\text{LHS} = \sum_i (x_i - c_i) + \sum_i \frac{c_i(c_i - 1)}{x_i + c_i} = \left(T - \sum_j c_j\right) + \sum_i \frac{c_i(c_i - 1)}{x_i + c_i}.
\]
Similarly, the right-hand side (RHS) can be decomposed as:
\begin{align*}
\text{RHS} &= \frac{T^2 - \sum_j c_j}{T + \sum_j c_j} \\
&= \frac{\left(T - \sum_j c_j\right)\left(T + \sum_j c_j\right) + \left(\sum_j c_j\right)^2 - \sum_j c_j}{T + \sum_j c_j} \\
&= \left(T - \sum_j c_j\right) + \frac{\left(\sum_j c_j\right)^2 - \sum_j c_j}{T + \sum_j c_j}.
\end{align*}
After canceling the common term $\left(T - \sum_j c_j\right)$, the original inequality is equivalently:
\[
\sum_i \frac{c_i(c_i - 1)}{x_i + c_i} < \frac{\left(\sum_j c_j\right)^2 - \sum_j c_j}{T + \sum_j c_j}.
\]
Multiplying by $-1$ reverses the inequality we wish to show:
\[
\sum_i \frac{c_i(1 - c_i)}{x_i + c_i} > \frac{\sum_j c_j - \left(\sum_j c_j\right)^2}{T + \sum_j c_j} .
\]

If $\sum_j c_j > 1$, the RHS is negative, while the LHS is a sum of nonnegative terms.
Thus, the inequality holds and is strict.

If $\sum_j c_j = 1$, then since $M \ge 2$ we must have that $c_i < 1$ for all $i$, and so the LHS is positive, while the RHS is zero. 
Thus, the inequality holds and is strict.

If $\sum_j c_j < 1$, both sides are positive. Let $y_i = x_i + c_i > 0$. Cross-multiplying gives:
\[
\left(\sum_j y_j\right) \left(\sum_i \frac{c_i(1 - c_i)}{y_i}\right) > \sum_j c_j - \left(\sum_j c_j\right)^2.
\]
Expanding the LHS yields:
\begin{align*}
\text{LHS} &= \sum_i \frac{y_i(c_i - c_i^2)}{y_i} + \sum_{i \ne j} \frac{y_j(c_i - c_i^2)}{y_i} \\
&= \left(\sum_i c_i - \sum_i c_i^2\right) + \sum_{i<j} \left( \frac{y_j}{y_i}(c_i - c_i^2) + \frac{y_i}{y_j}(c_j - c_j^2) \right).
\end{align*}
The RHS can be expanded as:
\begin{align*}
\text{RHS} &= \sum_j c_j - \left(\sum_j c_j\right)^2 \\
&= \sum_j c_j - \left(\sum_j c_j^2 + 2\sum_{j<k} c_j c_k \right) \\
&= \left(\sum_j c_j - \sum_j c_j^2\right) - 2\sum_{j<k} c_j c_k.
\end{align*}
Canceling the common term $\left(\sum c_i - \sum c_i^2\right)$, the inequality reduces to proving:
\[
\sum_{i<j} \left( \frac{y_j}{y_i}(c_i - c_i^2) + \frac{y_i}{y_j}(c_j - c_j^2) \right) > -2\sum_{j<k} c_j c_k.
\]
This final inequality is trivially true (and strict) for $M \ge 2$. The left-hand side is a sum of nonnegative terms, since $y_i, y_j > 0$ and $c_k(1-c_k) \ge 0$ for $c_k \in (0,1]$.
The right-hand side is strictly negative. 
This completes the proof.
Thus, having proved \Cref{eq:binaryweighting_simplify_general}, we have proved our desired result of \Cref{eq:binaryweighting_real_ineq}, showing that binary-weighted \stacksvd outperforms optimally weighted \svdstack when all $c_i \le 1$.
\end{proof}

We make a related conjecture, that if $c_i=c_0$ are all equal, or $\theta_i=\theta_0$ are all equal, then binary-weighted \stacksvd outperforms optimally weighted \svdstack.

\subsection{Dominance of optimally weighted \stacksvd} \label{sec:dominance_opt_weighted_stacksvd}

Having shown that binary \stacksvd dominates optimally weighted \svdstack when $c_i\le 1$, we now prove \Cref{prop:dominance}, which shows that \textbf{for all} $\theta,c$, optimally weighted \stacksvd performs at least as well as weighted \svdstack.

\begin{proof}[Proof of \Cref{prop:dominance}]
    
Recall that the expression for the asymptotic performance of \stacksvd is given as the unique solution $x \in (0,1)$ of the following equation:
\begin{equation*}
    f(x) = \sum_{i=1}^M \theta_i^4 \frac{1-x}{c_i + x\theta_i^2} = 1.
\end{equation*}
In the interval $(0,1)$, $f$ is monotonically decreasing.
Since the asymptotic performance of \svdstack is $S/(S+1)$, to show that \stacksvd dominates \svdstack, it suffices to show that $f(S/(S+1)) \ge 1$.
{Note that weighted \svdstack's recovery threshold is more stringent than weighted \stacksvd's, as: 
\begin{equation*}
    \sum_i \frac{\theta_i^4}{c_i} \ge \max_i \frac{\theta_i^4}{c_i} >1.
\end{equation*}
thus, below weighted \svdstack's recovery threshold, both methods have a performance of 0. When $S=0$ but $\sum_i \frac{\theta_i^4}{c_i} >1$, weighted \stacksvd has positive performance while weighted \svdstack falls below the recovery threshold.
Thus, we are left to consider the case where $S>0$.
We assume without loss of generality that all $\theta_i>0$, as otherwise they will be assigned a weight of 0. }
Recall that $\beta_i^2 = \frac{\theta_i^4 - c_i}{\theta_i^4 + \theta_i^2} \mathds{1}\{ \theta_i^4 > c_i \}$.
The one inequality we use is that for $x,y,z>0$ with $z\ge x$:
\begin{equation} \label{eq:helpfulIneq}
    \frac{x+y}{y+z} \ge \frac{x}{z},
\end{equation}
where the inequality is strict if $z > x$.
This follows by cross-multiplying and simplifying.
We now leverage this result to prove the following inequality holds for all $i$, with strict inequality if $\theta_i^4 > c_i$ for at least two $i$:
\begin{equation} \label{eq:dominance_termwise_ineq}
    \frac{\theta_i^4}{c_i + S(\theta_i^2+c_i)} 
    \ge \frac{\frac{\theta_i^4 - c_i}{\theta_i^2 + c_i}}{S} \mathds{1}\left\{ \theta_i^4 > c_i \right\}
    =\frac{\frac{\beta_i^2}{1-\beta_i^2}}{S}.
\end{equation}
We prove this by studying the two cases.
When $\theta_i^4 \le c_i$, the inequality is trivially true (and strict) as the LHS is positive, and the RHS is 0.
When $\theta_i^4 > c_i$, we apply the stated inequality:
\begin{align*}
    \frac{\theta_i^4}{c_i + S(\theta_i^2+c_i)}
    &= \frac{ \frac{\theta_i^4 - c_i}{\theta_i^2 + c_i} + \frac{c_i}{\theta_i^2 + c_i}}{\frac{c_i}{\theta_i^2 + c_i} + S}\\
    & \ge \frac{\frac{\theta_i^4 - c_i}{\theta_i^2 + c_i}}{S}.
\end{align*}
This is strict for the $i$-th term if $S> \frac{\theta_i^4 - c_i}{\theta_i^2 + c_i}$.
Combining these inequalities proves \eqref{eq:dominance_termwise_ineq}, and the summed inequality is strict when $S>0$ and at least two $\theta_i$ are nonzero.
Evaluating $f$ now allows us to see that:
\begin{align*}
    f\left(\frac{S}{S+1}\right) &= \sum_{i=1}^M \theta_i^4 \frac{1-\frac{S}{S+1}}{c_i + \theta_i^{2}\frac{S}{S+1}}\\
    &= \sum_{i=1}^M  \frac{\theta_i^4}{c_i(S+1) + S\theta_i^{2}}\\
    &= \sum_{i=1}^M \frac{\theta_i^4}{c_i + S(\theta_i^2+c_i)} \\
    & \ge \sum_{i=1}^M \frac{\frac{\theta_i^4 - c_i}{\theta_i^2 + c_i}}{S} \mathds{1}\{ \theta_i^4 > c_i \}\\
    &= \frac{1}{S}\sum_{i=1}^M \frac{\beta_i^2}{1-\beta_i^2}\\
    &= 1. \numberthis
\end{align*}
Where we apply the inequality in \Cref{eq:dominance_termwise_ineq} term-wise, and see that it is strict unless $S=0$ or only one $\theta_i$ is nonzero.
Combined with the monotonicity of $f$, this implies that the asymptotic performance of optimally-weighted \stacksvd is at least as good as that of optimally weighted \svdstack, and is strictly better if at least two tables contain nonzero signal (and we are above weighted-\stacksvd's threshold).
\end{proof}

\subsection{Optimally weighted \stacksvd performance goes to 1 while all others undetectable}\label{app:binary_weighting_opt}

In this section, we show that even optimally binary-weighted \stacksvd is inadmissible, and can fall below the threshold of detectability, while optimally weighted \stacksvd has performance going to 1.

\begin{proof}[Proof of \Cref{prop:binarystacksvd_inadmissable}]
This happens when:
\begin{align*}
    \frac{\left(\sum_{i\in S} \theta_i^2\right)^2}{\sum_{i\in S} c_i} &\le 1 \quad \forall \ S \subseteq [M],\\
    \max_i \frac{\theta_i^4}{c_i} &\le 1.
\end{align*}
The first condition is that any binary weighting of \stacksvd, i.e. selecting any subset $S$ of tables, is below the threshold of detectability.
The second, to rule out (un)weighted \svdstack, is that each table marginally is below the threshold of detectability. Note that this is subsumed by the first condition, taking $\CS = \{i\}$ for all $i$. 

We observe that by taking the tables to have constant $\theta_i$ and increasing $c_i$, or constant $c_i$ and decreasing $\theta_i$, the maximizing $S$ will simply be a prefix of this list, as seen by a swapping argument, simplifying the combinatorial nature of the constraint.
Then, we can compute $c_i$ (respectively $\theta_i$) such that for any prefix of the list, the inequality is met with equality, as this maximizes the performance of optimally weighted \stacksvd.

The first option we have is to take $c_i=1$ for all $i$.
Our constraint implies that we must have $(\sum_{i=1}^M \theta_i^2)^2 \le M$, where equality is attained by taking $\theta_j^2 = \sqrt{j} - \sqrt{j-1}$ for $j=1,\ldots,M$.
The asymptotic performance of this scheme is difficult to compute, so we instead consider the second case taking $\theta_i=1$ for all $i$.
In this case, $\sum_{i=1}^M \theta_i^2 = M$, and so we require $\sum_{i=1}^M c_i \ge M^2$.
To sit at this boundary, we take $c_j = j^2 - (j-1)^2 = 2j-1$ for $j=1,\ldots,M$.
We can then evaluate the performance of optimally weighted \stacksvd on this instance, as the root of the following equation:
\begin{equation*}
    f(x) := \sum_{i=1}^M \theta_i^4 \frac{1-x}{c_i + x\theta_i^{2}} = 1.
\end{equation*}
Our goal is to identify how large $M$ needs to be in order for the performance to be greater than $1-\eps$.
Since this function is monotonically decreasing in $x$, we can solve for $x$ such that $1\le f(x)$.
We substitute in $x=1-\eps$ for $\eps\in(0,1)$, and simplify, to obtain:

\begin{align*}
    1 &\le \sum_{i=1}^M \theta_i^4 \frac{1-x}{c_i + x\theta_i^{2}}\\
    &= \sum_{i=1}^M \frac{\eps}{2i-1+1-\eps}\\
    &= \frac{\eps}{2} \sum_{i=1}^M \frac{1}{i - \eps/2}.
\end{align*}
We lower bound the RHS to show that this condition is met whenever $M$ is sufficiently large, using the integral comparison test for decreasing functions, and that $\eps<1$
\begin{align*}
    \sum_{i=1}^M \frac{1}{i - \eps/2} \ge \sum_{i=1}^M \frac{1}{i} \ge \ln M + \gamma,
\end{align*}
where $\gamma$ is the Euler-Mascheroni constant.
This implies that $f(x) \ge 1$ whenever
\begin{equation*}
    \frac{2}{\eps} \le \ln M + \gamma \le \sum_{i=1}^M \frac{1}{i} \le \sum_{i=1}^M \frac{1}{i - \eps/2} .
\end{equation*}
Rearranging gives us that the asymptotic performance of weighted \stacksvd is precisely characterized by the digamma function, and can be bounded as greater than $1-\eps$ whenever
\begin{equation}\label{eq:Mapprox}
    M \ge e^{-\gamma} \exp(2/\eps).
\end{equation}
Since $\gamma >0$, this can be loosened to $M\ge \exp(2/\eps)$.
\end{proof}

Simulations demonstrate how the performance of optimally weighted \stacksvd scales when $\theta_i=1$ for all $i$, and $c_i = 2i-1$ (\Cref{fig:binary_stack_svd_inadmissable}).
Dots indicate the asymptotic performance predictions for the instance constructed for that given $M$ (root finding for the equation $f$), while the solid blue line indicates the closed form expression based on the digamma function \Cref{eq:Mapprox}.

\begin{figure}[h]
    \centering
    \includegraphics[width=0.5\linewidth, clip, trim=0 0 1.6cm 1.4cm]{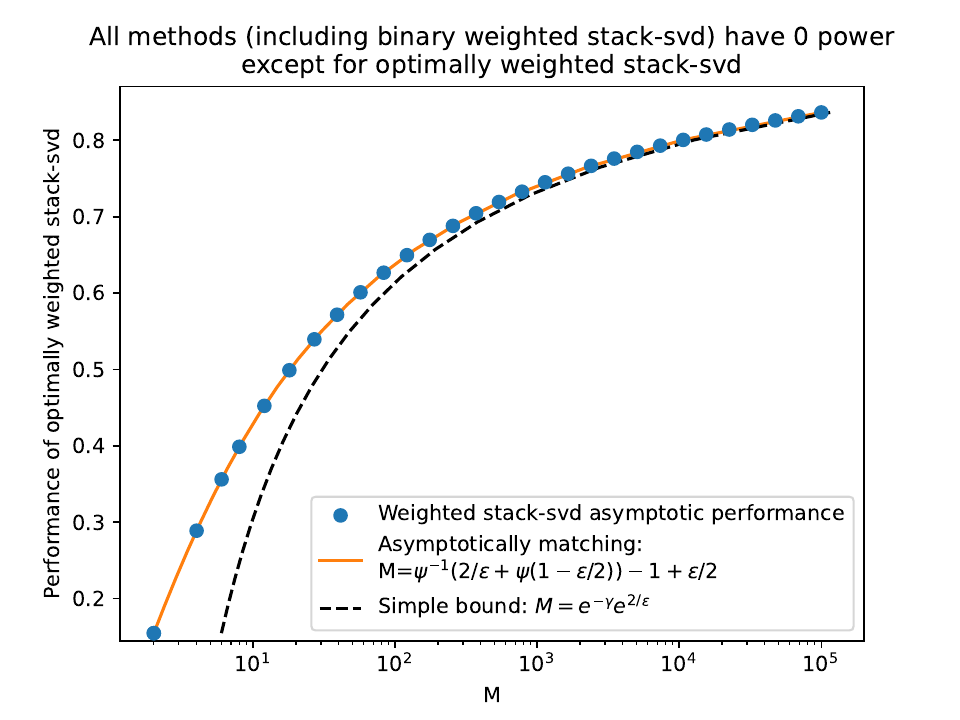}
        \vspace{-.3cm}
    \caption{Simulations demonstrating \Cref{prop:binarystacksvd_inadmissable}. Tables generated so that all methods, including optimally binary-weighted \stacksvd and optimally weighted \svdstack, fall below the detectability threshold, while optimally weighted \stacksvd has performance tending to 1. $\theta_i=1$ for all $i$, with $c_i = 2i-1$. $\psi$ is the digamma function \cite{abramowitz1965handbook}.}
    \label{fig:binary_stack_svd_inadmissable}
        \vspace{-.5cm}
\end{figure}

{
\section{Proofs for the general rank-$r$ model} \label{sec:unaligned_rank_r}

In this appendix, we provide the proofs for the general rank-$r$ unaligned subspaces settings.

\subsection{Unaligned weighted \svdstack} 
We begin with a necessary lemma to establish the performance of \svdstack:

\begin{lemma}\label{lem:general_rank_delocalization}
\begin{equation*}
    \begin{split}
        \tilde{V} V &\ip \diag(\beta) R_{\text{stack}} := B_R \\ 
        \tilde{V} \tilde{V}^{\top} &\ip A_{\beta, R}.
    \end{split}
\end{equation*}
\end{lemma}

\begin{proof}
    Because \svdstack is invariant to the sign of the singular vector, we assume that $\hat v_{ij}^\top (VR_i)_j\geq 0$. Then for the first part, 
    \begin{equation*}
        \hat{v}_{ij}^{\top} v_k = \hat{v}_{ij}^{\top} (V R_i)_j (V R_i)_j^{\top} v_k + \hat{v}_{ij}^{\top} (I - (V R_i)_j(V R_i)_j^{\top}) v_k.
    \end{equation*}
    By an argument similar to Lemma \ref{lem:delocalization}, the second term is $o_{\pr}(1)$. Then see that $(V R_i)^{\top} V = R_i^{\top}$. 

    \medskip 
    
    For the second part with $i \neq i'$, write 
    \begin{align*}
        \hat{v}_{ij}^{\top} \hat{v}_{i' j'} &= \hat{v}_{ij}^{\top} (V R_i)_j (V R_i)_j^{\top} \hat{v}_{i' j'} + o_{\pr}(1) \\
        &= \hat{v}_{ij}^{\top} (V R_i)_j (V R_i)_j^{\top} (V R_{i'})_{j'} (V R_{i'})_{j'}^{\top} \hat{v}_{i' j'} + o_{\pr}(1) \\
        &= \beta_{ij} \beta_{i' j'} (V R_i)_j^{\top} (V R_{i'})_{j'} + o_{\pr}(1). \numberthis
    \end{align*}
    For $i = i'$, $j \neq j'$, $(\tilde{V} \tilde{V}^{\top})_{(i,j),(i',j')} = 0$ and if $(i,j) = (i',j')$, $(\tilde{V} \tilde{V}^{\top})_{(i,j),(i',j')} = 1$. 
\end{proof}

\begin{proof}[Proof of Proposition \ref{prop:general_rank_unweighted_svdstack}]
Recall $\tilde{V} \in \R^{\tilde{r} \times d}$ is the stacked matrix of per-table estimates.
    By definition,
    \begin{equation}
        \Vsvdstack^{\top} V  = \left(\Lambda_r^{-1/2}(\tilde{V}\tilde{V}^{\top})
        Q_r(\tilde{V} \tilde{V}^{\top})^{\top} \tilde{V} \right) V. 
    \end{equation}
    Note that the positive-definiteness of $A_{\beta, R}$ along with Weyl's inequality implies that $\lambda_r^{-1/2}(\tilde{V} \tilde{V}^{\top})$ exists with probability tending to $1$ as $d \to \infty$. In particular,
    \begin{equation}
       \Lambda_r^{-1/2}(\tilde{V}\tilde{V}^{\top}) \ip \Lambda_r^{-1/2}(A_{\beta,R}).
    \end{equation}
    To establish the convergence of the eigenvectors, we use the Davis-Kahan Theorem (Theorem 2 of \cite{yu2015useful}). Let $I_1, \ldots, I_m$ be contiguous blocks of $\{1, \ldots, r\}$ corresponding to repeated eigenvalues of $A_{\beta, R}$. Because of the eigengap between the $r$ and $r+1$-st eigenvalue, there are orthonormal matrices $O_{I_1}^{(d)}, \ldots, O_{I_m}^{(d)}$ and an $r \times r$ orthonormal matrix $O^{(d)} = \diag(O_{I_1}^{(d)}, \ldots, O_{I_m}^{(d)})$ such that 
    \begin{equation}
        \| Q_r(A_{\beta,R})  - Q_r(\tilde{V} \tilde{V}^{\top}) O^{(d)} \|_F \ip 0
    \end{equation}
    Combining the above, \Cref{lem:general_rank_delocalization}, and the fact that the block diagonal $O^{(d)}$ commutes with $\Lambda^{-1/2}(A_{\beta, R})$, 
    \begin{align*}
        \|\Vsvdstack^{\top} V\|_F^2 
        &= \left\| \left(\Lambda_r^{-1/2}(\tilde{V}\tilde{V}^{\top})
        Q_r(\tilde{V} \tilde{V}^{\top})^{\top} \tilde{V}\right) V\right\|_F^2\\
        &\pto \left\|\Lambda_r^{-1/2}(A_{\beta,R}) Q_r(A_{\beta, R})^{\top} B_R\right\|_F^2
    \end{align*}
    by the continuous mapping theorem.
\end{proof}

\begin{proof}[Proof of Theorem \ref{thm:gen_rank_weight_svdstack}]
We begin with a calculation similar to the proof of Proposition \ref{prop:general_rank_unweighted_svdstack}. 
Let $W \in \R^{\tilde{r} \times \tilde{r}}$ be an arbitrary weighting satisfying the eigengap condition:
\begin{equation}
    \lambda_{r}\left(W A_{\beta, R} W^{\top} \right) > \lambda_{r+1}\left(W A_{\beta, R} W^\top\right). \label{eq:weighted_eigengap}
\end{equation}
By definition of \svdstack,
\begin{align*}
    \Vsvdstack(W) &= \tilde{V}^{\top} W^{\top} Q_r\left(W \tilde{V} \tilde{V}^{\top} W^{\top}\right) \Lambda_r^{-1/2}\left(W \tilde{V} \tilde{V}^{\top} W^{\top}\right)\\
    &= \tilde{V}^{\top} X^{(d)}. \numberthis
\end{align*}
Then, by an argument similar to the proof of \Cref{prop:general_rank_unweighted_svdstack}:
\begin{align*}
    &X^{(d)} = W^{\top} Q_r\left(W \tilde{V} \tilde{V}^{\top} W^{\top}\right) \Lambda_r^{-1/2}\left(W \tilde{V} \tilde{V}^{\top} W^{\top}\right)\\
    &X = W^{\top} Q_r(W A_{\beta, R} W^{\top}) \Lambda_r^{-1/2}(W A_{\beta, R} W^{\top})\\
    &\left\| X^{(d)} O^{(d)} - X\right\|_F^2 \pto 0,
\end{align*}
where $X,X^{(d)} \in \R^{\tilde{r} \times r}$, and $O^{(d)} \in \mathbb{O}\left(r,r\right)$.
In particular, the limiting performance may be written in terms of $X$:
\begin{align*}
    \left\| \Vsvdstack^{\top}(W) V \right\|_F^2
    &= \left\| \left(X^{(d)}\right)^\top \tilde{V} V \right\|_F^2 \\
    & \ip \left\| X^{\top} B_R \right\|_F^2,
\end{align*}
where $B_R := \diag(\beta) R_{\text{stack}}$. Furthermore, 
\begin{equation}
    X^{\top} A_{\beta, R} X = I_r.
\end{equation}

The above observation allows us to state the two key steps of the proof. Recall that $D := \diag(1 - \beta_{ij}^2) \in \R^{\tilde{r} \times \tilde{r}}$.

\begin{enumerate}
    \item Find the maximum of $\| X^{\top} B_R\|_F^2$ subject to $X^{\top} A_{\beta, R} X = I_r$. The above calculation shows that the maximal performance of weighted \svdstack can not exceed this value.
    \item Show \svdstack with weighting $W \opt = D^{-1/2}$ attains the maximal performance. 
\end{enumerate}

We begin with Step 1, which follows immediately from the observation that 
\begin{equation}
    A_{\beta, R} = B_R B_R^{\top} + D.
\end{equation}
In particular, 
\begin{align}
\sup_{X: X^{\top} A_{\beta, R} X= I_r} \| X^{\top} B_R\|_F^2 &=
    \sup_{X: X^{\top} A_{\beta, R} X= I_r} \tr \left( X^{\top} B_R B_R^{\top} X \right) \\
    &= \sup_{X: X^{\top} A_{\beta, R} X= I_r}  \tr \left(X^{\top} (A_{\beta, R} - D ) X\right) \\
    &= \sup_{X: X^{\top} A_{\beta, R} X= I_r}  r - \tr(X^{\top} D X) \\
    &= r - \inf_{X: X^{\top} A_{\beta, R} X = I_r} \tr(X^{\top} D X) \\
    &= r - \sum_{\ell = 1}^r \lambda_{\tilde{r} + 1 - \ell}(A_{\beta, R}^{-1/2} D A_{\beta,R}^{-1/2}).\label{eq:gr_w_svd_stack_opt_perf}
\end{align}

Proceeding to Step 2, we will now show that the choice of $\tilde{W} = D^{-1/2}$ yields the above optimal value. First, it is necessary to check that this choice of $W$ satisfies the eigengap condition \eqref{eq:weighted_eigengap}. As 
\begin{equation}
    D^{-1/2} A_{\beta, R} D^{-1/2} = D^{-1/2} B_R B_R^{\top} D^{-1/2} + I_{\tilde{r}},
\end{equation}
it suffices to show that the positive semidefinite matrix $D^{-1/2} B_R B_R^{\top} D^{-1/2}$ has exactly $r$ non-zero eigenvalues. This follows from expanding
\begin{equation}
    D^{-1/2} B_R B_R^{\top} D^{-1/2} = D^{-1/2} \diag(\beta) R_{\text{stack}} R_{\text{stack}}^{\top} \diag(\beta) D^{-1/2}
\end{equation}
and noting that $\diag(\beta)$ is invertible by assumption, and that $R_{\text{stack}} R_{\text{stack}}^{\top}$ has rank $r$ by the assumption that rows of $R_{\text{stack}}$ span $\R^r$. 

\medskip 

We conclude by showing that this induced $X = D^{-1/2} Q(D^{-1/2} A_{\beta, R} D^{-1/2}) \Lambda_r^{-1/2}(D^{-1/2} A_{\beta, R}D^{-1/2})$ achieves the performance in \eqref{eq:gr_w_svd_stack_opt_perf}. Letting $U =  Q(D^{-1/2} A_{\beta, R} D^{-1/2})$ and $\Lambda = \Lambda_r(D^{-1/2} A_{\beta, R}D^{-1/2})$:

\begin{align*}
    \tr \left( \Lambda^{-1/2} U^{\top} D^{-1/2} B_R B_R^{\top} D^{-1/2} U \Lambda^{-1/2}\right) 
    &=     \tr \left( \Lambda^{-1/2} U^{\top} D^{-1/2}(A_{\beta, R} - D) D^{-1/2} U \Lambda^{-1/2}\right) \\
    &= \tr \left( I_r - \Lambda^{-1}\right) \\
    &= r - \sum_{\ell = 1}^r \frac{1}{\lambda_{\ell} (D^{-1/2} A_{\beta, R} D^{-1/2})} \\
    &= r - \sum_{\ell = 1}^r \lambda_{\tilde{r} + 1 - \ell} (D^{1/2} A_{\beta, R}^{-1} D^{1/2}) \\
    &= r - \sum_{\ell = 1}^r \lambda_{\tilde{r} + 1 - \ell} (A_{\beta, R}^{-1/2} D A_{\beta, R}^{-1/2}),
\end{align*}
where in the last line, we have used the fact that $D^{1/2} A_{\beta, R}^{-1} D^{1/2} = (D^{1/2} A_{\beta, R}^{-1/2})(D^{1/2} A_{\beta, R}^{-1/2})^{\top}$ has the same eigenvalues as $(D^{1/2} A_{\beta, R}^{-1/2})^{\top}(D^{1/2} A_{\beta, R}^{-1/2})$.
\end{proof}

\subsection{\stacksvd with a single weighting per table}\label{sec:insufficiency_of_global_weights_stacksvd}

We define $\Vstacksvd(w)$ to be the top $r$ eigenvectors of 
\begin{equation*}
    \sum_{i=1}^M w_i^2 X_i^{\top} X_i. 
\end{equation*}
Furthermore, define
\begin{equation*}
    A = \begin{bmatrix}
        &w_1 U_1 \Theta_1 R_1^{\top} \\ &\vdots \\ &w_M U_M \Theta_M R_M^{\top} 
    \end{bmatrix} \in \R^{n \times r}
\end{equation*}
As in the rank one fully aligned setting, the performance of weighted \stacksvd depends on the spectrum of
\begin{equation*}
    G = AA^{\top} + \Sigma,
\end{equation*}
which reflects $\E(X_{\text{stack}} X_{\text{stack}}^{\top})$. For convenience, we make the assumption that the eigenvalues of $G$ are distinct and each lie above the detectability threshold of \cite{liu2023asymptotic}. 
\begin{assum}\label{assum:gen_rank_stacksvd_eig_sep}
    The top $r$ eigenvalues of $G$, denoted $\gamma_1, \ldots, \gamma_r$, are distinct and satisfy 
    \begin{equation}
        \sup_{\ell \in [r]} \sum_{i=1}^M \frac{c_i w_i^4}{(\gamma_{\ell} - w_i^2)^2} < 1
    \end{equation}
    and $\gamma_r > \sup_{i \in [M]} w_i^2$.
\end{assum}

\begin{prop}\label{prop:gen_rank_stacksvd_singleweight}
    Under Assumption \ref{assum:gen_rank_stacksvd_eig_sep}, for $\ell = 1,\hdots,r$ the matrix 
    \begin{equation}
        \sum_{i=1}^M \left( \frac{w_i^2}{\gamma_{\ell} - w_i^2} \right)R_i \Theta_i^2 R_i^{\top}.
    \end{equation}
    has a unit eigenvector $z_{\ell}$ with eigenvalue $1$. The performance of single-table weighted \stacksvd satisfies
    \begin{equation}
        \|\Vstacksvd(w)^{\top} V\|_F^2 \ip \sum_{\ell = 1}^r \frac{1 - \sum_{i=1}^M \frac{c_i w_i^4}{(\gamma_{\ell} - w_i^2)^2}}{\gamma_{\ell} z_{\ell}^{\top} \left[ \sum_{i=1}^M \left( \frac{w_i^2}{(w_i^2 - \gamma_{\ell})^2} \right)R_i \Theta_i^2 R_i^{\top}  \right] z_{\ell}}
    \end{equation}
\end{prop}

\begin{proof}
    The structure of the proof is similar to \Cref{thm:stacksvd_weighted}. First, we expand on \Cref{lem:secular_equation} to obtain the secular equation defining the spiked eigenvalues. If $\gamma > \sup_{i \in [M]} w_i^2$, 
    \begin{align*}
        \det(G  - \gamma I_n) &= \det(AA^{\top} + \Sigma - \gamma I_n) \\
        &= \det(\Sigma - \gamma I) \det(I_r - A^{\top}(\gamma I_n - \Sigma)^{-1} A) \\
        &= \det(\Sigma - \gamma I) \det \left( I_r - \sum_{i=1}^M \frac{w_i^2}{\gamma - w_i^2} R_i \Theta_i^2 R_i^{\top} \right)
    \end{align*}
    Thus, $\gamma_{\ell}$, $1 \leq \ell \leq r$, satisfies
    \begin{equation*}
        \det \left( I_r - \sum_{i=1}^M \frac{w_i^2}{\gamma_{\ell} - w_i^2} R_i \Theta_i^2 R_i^{\top} \right) = 0,
    \end{equation*}
    which demonstrates the existence of $z_{\ell}$ in the statement of the proposition. A corresponding unit eigenvector $\xi_{\ell}$ satisfies 
    \begin{equation*}
        A^{\top} \xi_{\ell} = A^{\top} (\gamma_{\ell} - \Sigma)^{-1} A(A^{\top} \xi_{\ell}).
    \end{equation*}
    So $A^{\top} \xi_\ell \propto z_{\ell}$. Furthermore,
    \begin{equation*}
        \xi_{\ell} = \frac{(\gamma_{\ell} - \Sigma)^{-1} A z_{\ell}}{\|(\gamma_{\ell} - \Sigma)^{-1} A z_{\ell} \|_2} \Longrightarrow A^{\top} \xi_{\ell} = \frac{z_{\ell}}{\|(\gamma_{\ell} - \Sigma)^{-1} A z_{\ell} \|_2}.
    \end{equation*}

    \medskip 

    Now, the squared subspace overlap may be expanded as  
    \begin{equation*}
        \| \Vstacksvd(w)^{\top} V \|_F^2 = \sum_{\ell = 1}^r \sum_{k=1}^r (\hat{v}_{\ell}^{\top} v_k)^2,
    \end{equation*}
    where $\hat{v}_{\ell}$ denotes the $\ell$-th eigenvector and $v_k$ denotes the $k$-th column of $V$. Applying Theorem 1 of \cite{liu2023asymptotic} yields 
    \begin{equation*}
        \left(\hat{v}_\ell^\top v_k\right)^2 = \eta_{\ell} \frac{e_k^{\top} (A^{\top} \xi_{\ell})(A^{\top} \xi_{\ell})^{\top} e_k}{\gamma_{\ell}} + o_{\pr}(1)
    \end{equation*}
    where $e_k \in \R^r$ is the standard basis vector. By the same argument in the proof of \Cref{thm:stacksvd_weighted}
    \begin{align*}
        \eta_{\ell} &= 1 - \frac{1}{d} \sum_{i \neq \ell} \frac{\gamma_i^2}{(\gamma_i - \gamma_{\ell})^2} \\
        &= 1 - \sum_{i=1}^M \frac{c_i w_i^4}{(\gamma_{\ell} - w_i^2)^2} + O(1/d).
    \end{align*}
    Then, 
    \begin{equation*}
        \sum_{\ell = 1}^r \sum_{k=1}^r (\hat{v}_{\ell}^{\top} v_k)^2 = \sum_{\ell = 1}^r \left(1 - \sum_{i=1}^M \frac{c_i w_i^4}{(\gamma_{\ell} - w_i^2)^2}\right) \frac{\|A^{\top} \xi_{\ell}\|_2^2}{\gamma_{\ell}} + o_{\pr}(1).
    \end{equation*} 
    Finally, by leveraging the relationship for $A^{\top} \xi_{\ell}$ derived above,
    \begin{align*}
        \|A^{\top} \xi_{\ell}\|_2^2 &= \frac{1}{z_{\ell}^{\top} A^{\top} (\gamma_{\ell} - \Sigma)^{-2} A z_{\ell}} \\
        &= \frac{1}{z_{\ell}^{\top} \left[ \sum_{i=1}^M \left( \frac{w_i^2}{(w_i^2 - \gamma_{\ell})^2} \right)R_i \Theta_i^2 R_i^{\top}  \right] z_{\ell}}.
    \end{align*}
\end{proof}

\subsubsection{Leveraging \Cref{prop:gen_rank_stacksvd_singleweight} to show the insufficiency of single-weights for \stacksvd (Proof of \Cref{prop:singleweight_suboptimality})} \label{sec:appendix_insufficiency_single_weights_stacksvd}

We now apply the result to demonstrate that unweighted \svdstack can outperform optimally weighted \stacksvd with a single weight per table in the general rank setting. Take 
\begin{equation}
    R_1 = \begin{bmatrix}
        1 \\ 0
    \end{bmatrix}, \quad R_2 = \begin{bmatrix}
        0 \\ 1
    \end{bmatrix}, \Theta_1 = \Theta_2 = \diag(\theta_0), \quad \theta_0^4 > c_0, \quad c_1 = c_2 = c_0.
\end{equation}
To derive the performance of optimally weighted \svdstack (which coincides with unweighted \svdstack in this setting), we may apply Theorem \ref{thm:gen_rank_weight_svdstack} with
\begin{equation*}
    A_{\beta, R} = \begin{bmatrix}
        1 & 0 \\ 0 & 1
    \end{bmatrix}, \quad (W\opt)^{-2} = \begin{bmatrix}
        1 - \beta_0^2 & 0 \\ 0 & 1 - \beta_0^2
    \end{bmatrix}
\end{equation*}
and therefore $\|\Vsvdstack(W\opt)^\top V\|_F^2 \ip 2 - (1 - \beta_0^2) - (1 - \beta_0^2) = 2 \beta_0^2$. For \stacksvd, we first assume without loss of generality that $w_1 > w_2$ and that assumption \ref{assum:gen_rank_stacksvd_eig_sep} holds. Otherwise, at least one component is undetectable and the performance is bounded above by $\beta_0^2$. In this case, a simple application of \Cref{prop:gen_rank_stacksvd_singleweight} yields
\begin{equation*}
\begin{split}
        \gamma_{\ell} &= w_{\ell}^2(1+ \theta_0^2), \quad w_2^2(1+\theta_0^2) > w_1^2, \quad z_{\ell} = e_{\ell} \in \R^2 \\
        \| \Vstacksvd(w_1, w_2)^{\top} V\|_F^2 &\ip \frac{\theta_0^4 - c_0 - \frac{c_0 w_2^4 \theta_0^4}{(w_1^2(1 + \theta_0^2) - w_2^2)^2}}{\theta_0^2(1 + \theta_0^2)} + \frac{\theta_0^4 - c_0 - \frac{c_0 w_1^4 \theta_0^4}{(w_2^2(1 + \theta_0^2) - w_1^2)^2}}{\theta_0^2(1 + \theta_0^2)} < 2 \beta_0^2.
\end{split}
\end{equation*}
When $w_1 = w_2$, the performance is obtained by a result for unweighted \stacksvd, and is also strictly less than $2 \beta_0^2$. This proves \Cref{prop:singleweight_suboptimality}, showing that in the general rank-$r$ setting, there exists a problem instance such that unweighted \svdstack outperforms optimally weighted \stacksvd with a single weight per table.

\subsection{Proofs for unweighted \stacksvd and \svdstack  in the unaligned setting} \label{app:unaligned_example_proofs}

In this subsection, we provide proofs and analysis of the results shown in \Cref{fig:unaligned_sim}.

\subsubsection{Unweighted \svdstack in unaligned setting}
First, we observe that the performance of \svdstack in \Cref{eq:perf_rankr_ex_svdstack} 
is strictly decreasing on $(0, 1)$ for $\beta \in (0, 1)$:
\begin{equation}
    \frac{d}{ds} \left[ \frac{\beta^2 (1 + s)}{1 + s \beta^2} + \frac{\beta^2 (1 - s)}{1 - s \beta^2} \right] = \beta^2 (1 - \beta^2) \left[ \frac{1}{(1 + s \beta^2)^2} - \frac{1}{(1 - s \beta^2)^2} \right] < 0,
\end{equation}
with the degenerate cases $\beta = 0$ (identically $0$) and $\beta = 1$ (identically $2$).
This is intuitively due to the concavity of the performance for each component; we require $\beta>0$ in order for nontrivial recovery to be possible for \svdstack, and above this threshold, having more weight on component 2 is more beneficial than on component 1 (where $R_1$ already places all its weight).

\subsubsection{Unweighted \stacksvd in the unaligned setting}
Analyzing \stacksvd, we see that its performance is in fact \textit{non-monotone} in $s$. 
We see that our core matrix here is 
\begin{equation}
    C = \sum_i R_i \Theta_i^2 R_i^\top = \theta^2 \begin{bmatrix}
        1 + s^2 & s\sqrt{1-s^2}\\
        s\sqrt{1-s^2} & 1-s^2
    \end{bmatrix}.
\end{equation}
Write $\lambda_{\pm}(s) = \theta^2(1 \pm s)$ for the two eigenvalues of the core
matrix $C$ in this example, and define the single-spike performance function
\begin{equation}
    h(\lambda) = \frac{\lambda^2 - 2c}{\lambda(\lambda + 1)} \, \mathds{1}\left\{ \lambda > \sqrt{2c} \right\}, \qquad \lambda > 0,
\end{equation}
so that the \stacksvd limit in \eqref{eq:perf_rankr_ex_stacksvd} is
$F(s) = h(\lambda_{+}(s)) + h(\lambda_{-}(s))$ for $s \in [0, 1)$.
The indicators in \eqref{eq:perf_rankr_ex_stacksvd} switch when $\lambda_{\pm}(s) = \sqrt{2c}$.
In the strong-signal regime $\theta^4 > 2c$, both components are detected at
$s = 0$, and the only switch in $(0,1)$ is component $2$ dropping below the
detection threshold. This gives the kink location
\begin{equation}\label{eq:kink_location}
    s^{\star} = 1 - \frac{\sqrt{2c}}{\theta^2},
\end{equation}
which moves toward $1$ as $\theta$ grows. 

The behavior on each side of $s^{\star}$ follows from two properties of $h$ on $(0, \infty)$, obtained by direct differentiation:
\begin{align}
    h'(\lambda) &= \frac{\lambda^2 + 4c\lambda + 2c}{\lambda^2 (\lambda + 1)^2} > 0, \\
    h''(\lambda) &= \frac{-2 \left( \lambda^3 + 6c\lambda^2 + 6c\lambda + 2c \right)}{\lambda^3 (\lambda + 1)^3} < 0.
\end{align}
Thus $h$ is strictly increasing and strictly concave, so $h'$ is strictly decreasing.

\medskip

\noindent
\emph{Right of the kink} ($s^{\star} < s < 1$): only component $1$ survives, so
$F(s) = h(\lambda_{+}(s))$ and
\begin{equation}
    F'(s) = \theta^2 \, h'\left( \theta^2 (1 + s) \right) > 0,
\end{equation}
so $F$ is strictly increasing.

\medskip

\noindent
\emph{Left of the kink} ($0 < s < s^{\star}$): both terms are active, and since
$\lambda_{+}(s) > \lambda_{-}(s)$ with $h'$ strictly decreasing,
\begin{equation}
    F'(s) = \theta^2 \left[ h'(\lambda_{+}(s)) - h'(\lambda_{-}(s)) \right] < 0,
\end{equation}
so $F$ is strictly decreasing. Concavity of $h$ is the driver: the marginal gain
from strengthening the already-strong component $1$ is smaller than the marginal
loss from weakening component $2$.

The curve is continuous at $s^{\star}$ because $h(\sqrt{2c}) = 0$, so the overlap
of component $2$ vanishes continuously at the detection threshold; only the slope
jumps. The size of the jump is
\begin{equation}
    F'(s^{\star +}) - F'(s^{\star -}) = \theta^2 \, h'\left( \sqrt{2c} \right) = \frac{2 \theta^2}{1 + \sqrt{2c}} > 0.
\end{equation}
Hence in the regime $\theta^4 > 2c$, the function $F$ is strictly V-shaped with a
unique global minimizer at $s^{\star}$: strictly decreasing on $(0, s^{\star})$, strictly increasing on $(s^{\star}, 1)$. At the left endpoint the two core eigenvalues coincide and $F'(0) = 0$, so the curve is flat at $s = 0$.

\paragraph*{Weak-signal regime}
If instead $\theta^4 < 2c$, component $2$ is never detected on $[0, 1)$, and the switch in $(0,1)$ is component $1$ \emph{entering} detection at $s_{\dagger} = \sqrt{2c} / \theta^2 - 1$.
Then $F \equiv 0$ on $[0, s_{\dagger}]$ and $F$ is strictly increasing on $(s_{\dagger}, 1)$, so $F$ is monotone nondecreasing overall.
At the boundary $\theta^4 = 2c$ one has $F(0) = 0$ and $F$ strictly increasing on $(0, 1)$.
This is provided that $\theta^4>c/2$ (so $s_\dagger<1$); if $\theta^4\le c/2$ then $s_\dagger\ge1$ and $F\equiv0$ on all of [0,1).
}

\section{Rank-$r$ exactly aligned scenario}\label{sec:rank_r}
In the main text, we began with the rank-1 setting, then extended to the unaligned setting with per-table subspaces $VR_i$.
In this section, we show that our results simplify substantially when $r_i=r$ and $R_i = I_{r}$ for all $i$, i.e. all components are fully shared.
In this setting, our model is simply:
\begin{equation} \label{eq:rank_r_model}
    X_i = \sum_{j=1}^r \theta_{ij} u_{ij} v_j^\top + E_i.
\end{equation}
Here, $\{v_j\}_{1\le j\le r}$ are shared orthonormal vectors, and $\{u_{ij}\}_{1\le j\le r}$ are individual-specific loadings, which are all orthonormal.
We define the matrix $V \in \R^{d \times r}$ with columns $v_j$.
We make the following assumption to analyze these methods.

\begin{assum} \label{assum:rank_r}
   We assume that the matrices $\{X_i\}$ are generated according to the rank-$r$ model \eqref{eq:rank_r_model}, and satisfy the classical RMT scaling limits in \eqref{eq:RMT_limit}.
\end{assum}

We begin our discussion by assuming that the $\theta_i$ follow the same order for each table, which is for simplicity of exposition.
Theoretically, we relax this assumption to allow for essentially arbitrary $\theta_{ij}$ (unordered, some components can have equal value). In this section, we derive the performance of the rank-$r$ variants of \stacksvd and \svdstack using the optimal weights derived in \Cref{sec:optimal_weighting}.

\subsection{Rank-r \stacksvd} \label{sec:fullySharedRankr_stacksvd}

Our approach to extend weighted \stacksvd to the rank-$r$ setting is inspired by the weighted PCA model of \cite{hong2023optimally}.
Algorithmically, due to the differing weightings required for each component in this rank-$r$ setting, we assemble for each $j=1,\hdots,r$ the matrix $X_\text{stack}^{(j)}$ based on the optimal weightings $w_{ij}$ for table $i$ and component $j$.
For each table $X_\text{stack}^{(j)}$, we denote its $j$-th largest right singular vector as $\hat{v}_{j, \stacksvd} := v_j(X_\text{stack}^{(j)})$, and define $\hat{V}_\stacksvd = 
    \begin{bmatrix}
        \hat{v}_{1, \stacksvd} & \ldots & \hat{v}_{r, \stacksvd}
    \end{bmatrix}$, which simply forms $\{\hat{v}_{j, \stacksvd}\}$ into a $d \times r$ matrix.
Note that the columns of $\hat{V}_\stacksvd$ will be asymptotically orthogonal. 
To analyze the asymptotic performance of this method for the $j$-th component, we define analogously to \Cref{thm:stacksvd_weighted} the quantities $\gamma_j$.
Formally:

\begin{align}
    w_{ij} &= \frac{\theta_{ij}}{\sqrt{\theta_{ij}^2 + c_i}}, \label{eq:stacksvd_app_wij}\\
    X_\text{stack}^{(j)}
    &= \begin{bmatrix}
        w_{1j} X_1 \\
        \vdots \\
        w_{Mj} X_M 
    \end{bmatrix}, \label{eq:stacksvd_appXstack}\\
    \tilde{\theta}_{jk}^2 &= \sum_{i=1}^M \frac{\theta_{ij}^2\theta_{ik}^2}{\theta_{ij}^2 + c_i}, \label{eq:stacksvd_apptildTheta}\\
    \lambda_{jk} &= \text{ the largest root } \lambda > \max_{1 \leq i \leq M} w_{ij}^2 \text{ of } \sum_{i=1}^M \frac{\theta_{ik}^2 w_{ij}^2}{\lambda - w_{ij}^2} = 1,\label{eq:stacksvd_lambda}\\
        \gamma_j &= 
        \text{ the unique solution $x \in (0,1)$ of } \sum_{i=1}^M \theta_{ij}^4 \frac{1-x}{c_i + x \theta_{ij}^{2}} = 1 \label{eq:stacksvd_gammak}
\end{align}
where $\gamma_j=0$ if $\sum_i {\theta_{ij}^4 }/{c_i}\le1$, and $\lambda_{jk} := \max_i w_{ij}^2$ if that equation has no root.
\blue{
Here $\lambda_{jk}$ is the spiked eigenvalue of $\E[X_\text{stack}^{(j)}(X_\text{stack}^{(j)})^\top]$ that component $k$ carries. Under \Cref{assum:rank_r} the $r$ components decouple, and \eqref{eq:stacksvd_lambda} is \Cref{lem:secular_equation} with $\theta_{ik}$ and $w_{ij}$ in place of $\theta_i$ and $w_i$.
The $k$-th component is detectable under the $j$-th weighting when $\lambda_{jk}$ exists and \eqref{eq:assumption4} holds there, that is $\sum_{i} c_i w_{ij}^4/(\lambda_{jk}-w_{ij}^2)^2 < 1$.
The outlier that component $k$ produces sits at $\rho_{jk} = \lambda_{jk}(1 + \sum_i c_i w_{ij}^2/(\lambda_{jk}-w_{ij}^2))$, written $\rho$ as a function of $\lambda$ below.
The left side of \eqref{eq:assumption4} is decreasing in $\lambda$, from $+\infty$ at $\max_i w_{ij}^2$ to $0$, so it equals $1$ at a single point $\lambda_j^{*}$, and component $k$ is detectable exactly when $\lambda_{jk} > \lambda_j^{*}$.
Since $d\rho/d\lambda = 1 - \sum_i c_i w_{ij}^4/(\lambda-w_{ij}^2)^2$, the map $\lambda \mapsto \rho$ is increasing above $\lambda_j^{*}$.
The detectable components therefore carry their outliers in the order of the $\lambda_{jk}$, and every undetectable component has a smaller $\lambda_{jk}$ than every detectable one, so we may rank by $\lambda_{jk}$ alone over all $k \in [r]$.
}
This enables us to provide the following Corollary regarding rank-$r$ weighted \stacksvd, following a similar path to the existing weighted \stacksvd result in \Cref{thm:stacksvd_weighted}, leveraging \cite{liu2023asymptotic,hong2023optimally}.

\begin{cor}\label{thm:rank_r_stacksvd}
Under \Cref{assum:rank_r,assum:general_noise}, and assuming that $\lambda_{jj} \neq {\lambda}_{jk}$ for all weightings $j$, for all components $k\neq j$ (\Cref{eq:stacksvd_lambda}),
rank-$r$ weighted \stacksvd satisfies:
\begin{align*}
    \left(v_j^\top \hat{v}_{j,\stacksvd}\right)^2 &\pto \gamma_j \quad \text{ for $j=1,2,\hdots,r$},\\
    \left\| V^\top \hat{V}_\stacksvd \right\|_F^2 &\pto \sum_j \gamma_j.
\end{align*}
\end{cor}

This aggregate measure of similarity $\| V^\top \hat{V}_\stacksvd \|_F^2$ is simply the sum of the squared cosines of the principal angles when $\hat{V}$ is orthonormal.
However, for any matrix $\hat{V}_\stacksvd$ with unit norm columns, this subspace similarity measure is always bounded above by $r$.
We discuss the algorithmic and proof details below, and show how this extends to unordered $\theta_{ij}$.

We begin by briefly discussing our work in relation to \cite{hong2023optimally}, who study a weighted PCA setting that can be reformulated to mirror our model.
Converting their setting to our notation, for matrix $i$, the $k$-th component has signal strength $\theta_{i,k} = \lambda_k a_i$, allowing only one degree of freedom per matrix uniformly scaling the singular values.
We show that weighted \stacksvd trivially applies to the more general setting where $\theta_i$ are simply ordered in each table, with $r$ degrees of freedom per table.

Note that if the ordering of the signal strengths is not the same across tables, then more care is needed to ensure we select the right components.
For example, consider $\theta_1 = [2,1]$ and $\theta_2=[1,10]$, where $c_1=c_2=1$.
Then, analyzing the first component, $w_{11} = 2/\sqrt{5}$ and $w_{21} = 1/\sqrt{2}$.
However, when weighting according to this, the signal strength corresponding to $v_1$ is scaling with $\sum_i w_{i1}^2 \theta_{i1}^2 = 16/5+1/2 = 3.7$.
A corresponding calculation for the signal strength corresponding to $v_2$ under this weighting yields $4/5+100/2=50.8$.
This means that, even under the optimal weighting for the first component, the second component has stronger signal strength.
However, with knowledge of the $\theta_{ij}$, we can easily account for this with some additional notation and bookkeeping.

We show in \Cref{alg:rank_r_stacksvd} how this procedure works.
\blue{As can be seen, $\theta_{ij}$ do not need to be ordered; we simply require that $\lambda_{jj} \neq \lambda_{jk}$ for all weightings $j$, for any component $k\neq j$, so that $\ell_j$ in line \ref{alg_line:stacksvd_rank} of \Cref{alg:rank_r_stacksvd} is well-defined. The rank of $\tilde{\theta}_{jj}$ among $\{\tilde{\theta}_{jk}\}_k$ is not in general the rank of $\lambda_{jk}$: the noise of $X_\text{stack}^{(j)}$ is heteroscedastic across blocks, and \eqref{eq:stacksvd_lambda} weights block $i$ by $w_{ij}^2/(\lambda-w_{ij}^2)$, which grows with $w_{ij}$, while $\tilde{\theta}_{jk}^2 = \sum_i w_{ij}^2\theta_{ik}^2$ weights the blocks alike.}
Observe that if $\theta_{ij}$ follow the same ordering in each table, then the ordering is preserved for each weighting, and so $\ell_j = j$.

\begin{algorithm}[h!]
\caption{Rank-$r$ Weighted \stacksvd}\label{alg:rank_r_stacksvd}
\begin{algorithmic}[1]
\State \textbf{Input:} Data matrices $\{X_i\}_{i=1}^M$, where $X_i \in \mathbb{R}^{n_i \times d}$; signal strengths $\{\theta_{ij}\}_{i \in [M], j \in [r]}$
\State \textbf{Output:} An estimated basis for the shared subspace, $\hat{V}_\stacksvd \in \mathbb{R}^{d \times r}$.
\State \textbf{Initialize:} $\tilde{V}$ as an empty list of vectors.
\For{$j = 1, \dots, r$} \Comment{Estimate each component $v_j$ separately.}
    \State For each table $i\in[M]$, calculate the optimal weight for component $j$ as $w_{ij}$ (\Cref{eq:stacksvd_app_wij})
    \State For each component $k\in[r]$, compute its spiked eigenvalue under the $j$-th weighting $\lambda_{jk}$ (\Cref{eq:stacksvd_lambda})
    \State Construct the weighted stacked matrix $X_{\text{stack}}^{(j)}$ (\Cref{eq:stacksvd_appXstack})
    \State \blue{Compute the rank $\ell_j$ of $\lambda_{jj}$ amongst $\{\lambda_{jk}\}_{k=1}^r$ (see \eqref{eq:stacksvd_lambda}). If $\{\theta_{ij}\}$ carry the same order in every table, $\ell_j=j$. \label{alg_line:stacksvd_rank}}
    \State Compute $\hat{v}_{j,\stacksvd} \leftarrow v_{\ell_j}\left(X_{\text{stack}}^{(j)}\right)$ as the $\ell_j$-th right singular vector of $X_{\text{stack}}^{(j)}$
\EndFor
\State Combine the collected vectors into a single matrix: $\hat{V}_{\stacksvd} \leftarrow [\hat{v}_{1,\stacksvd}, \dots, \hat{v}_{r,\stacksvd}]$.
\State \textbf{return} $\hat{V}_\stacksvd$
\end{algorithmic}
\end{algorithm}

\subsection{Rank-r \svdstack} 

To study \svdstack in the rank $r$ setting, we analogously define $\beta_{ij}$ as the performance of the $j$-th component in the $i$-th table, and $S_j$ as the overall performance for the $j$-th component aggregated across all tables. Specifically: \looseness=-1
\[
\beta_{ij}^2 := \frac{\theta_{ij}^4 - c_i}{\theta_{ij}^4 + \theta_{ij}^2} \mathds{1} \left\{ \theta_{ij}^4 \geq c_i \right\},\qquad S_j := \sum_i \frac{\beta_{ij}^2}{1-\beta_{ij}^2}.
\]
As in the unaligned setting, $w_{ij}$ are given by inverse noise variance weighting, and $\tilde{V}$ is defined by their weighted concatenation.
Algorithmically, our estimator $\hat{V}_\svdstack \in \R^{d \times r}$ is generated by taking the first $r$ right singular vectors of $\tilde{V} \in \R^{Mr \times d}$. 
We can then provide the following guarantees, leveraging the results from \Cref{thm:gen_rank_weight_svdstack}.

\begin{cor}\label{thm:rank_r_svdstack} 
Under \Cref{assum:rank_r,assum:general_noise}, assuming that a) $\theta_{ij} \neq \theta_{ik}$ for all $i$, for all $j\neq k$, and b) $S_{j}\neq S_k$ for all $j\neq k$,
rank-$r$ weighted \svdstack satisfies:
\begin{align*}
    \left(v_j^\top \hat{v}_{j,\svdstack}\right)^2 &\pto \frac{S_j}{S_j+1} \quad \text{ for $j=1,2,\hdots,r$},\\
    \left\| V^\top \hat{V}_\svdstack \right\|_F^2 &\pto \sum_j \frac{S_j}{S_j+1}.
\end{align*}
\end{cor}

Since the vectors corresponding to different $v_k$ are asymptotically orthogonal, and $\theta_{ij}$ are distinct within a table, the analogous matrix $A_\beta$ corresponding to the limiting behavior of $\tilde{V} \tilde{V}^\top$ is block diagonal with $r$ blocks corresponding to the $r$ components.

We require that the $\beta_{ij}$ are unique within a table so that the limiting matrix $A_{\beta,R}$ is well defined.
The $S_j$ are assumed to be unique so that we can uniquely identify each component from the SVD of $\tilde{V}$, though this can be generalized by using a Rank Tracking procedure as in \Cref{alg:rank_r_stacksvd} and performing $r$ SVDs.
\color{black}

\section{Estimating unknown parameters}

As discussed, heretofore we have assumed that $\theta_i$ were all known.
Here, we discuss how to estimate these signal strength parameters from the observed data.

\subsection{Estimating $\theta_i$ below the threshold}\label{app:theta_est_proof}

Below we provide the proof of \Cref{thm:theta_est}, regarding the estimator in \Cref{eq:theta_estimation}:
\begin{equation*} 
    \hat{\theta}_2 = \frac{1}{\hat{\beta}_1} \sqrt{\|X_2 \hat{v}_1\|_2^2 - c_2}.
\end{equation*}
We compare the empirical performance of this estimator in \Cref{fig:mse_theta2}, finding that it performs well unless $\theta_1$ is very close to the detectability threshold.

\begin{figure}
    \centering
    \includegraphics[scale=0.6]{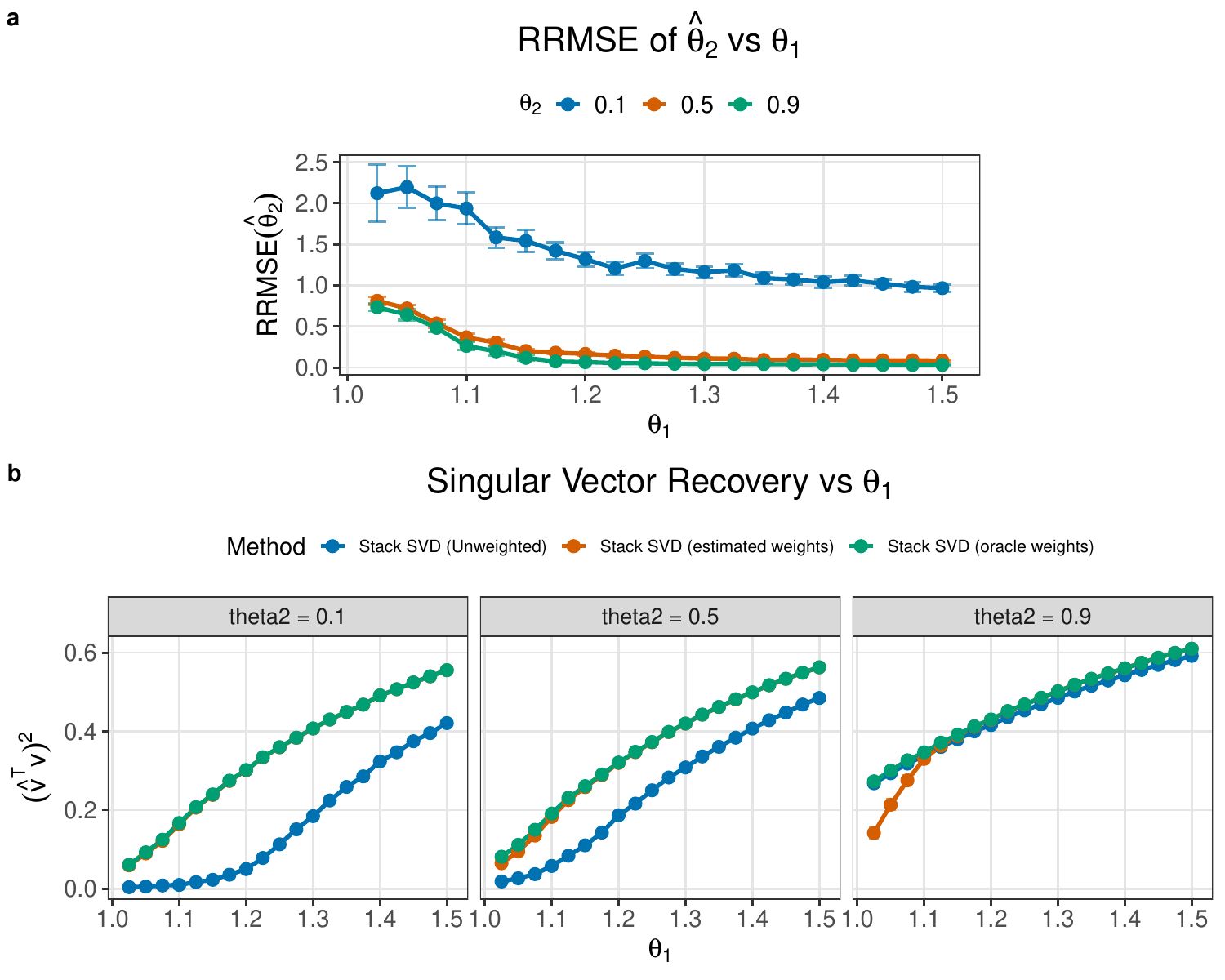}
    \caption{\textbf{a.} Simulations of the proposed estimator of $\theta_2$ when it falls below the detectability threshold. We set $n = d = 5000$ and varied $\theta_1 \in [1.025, 1.5]$, computing the root relative mean squared error (RRMSE) between $\theta_2$ and $\hat{\theta}_2$ (as defined in Proposition \ref{thm:theta_est}) across $B=25$ repetitions. Note that $\operatorname{RRMSE}(\widehat{\theta}_2)=\sqrt{B^{-1}\sum_{b=1}^B(\widehat{\theta}_{2,b}-\theta_2)^2/\theta_2^2}$. \textbf{b.} In the same setting, we estimated $\hat{v}$ using the oracle weights as well as the weights computed by plugging in $\hat{\theta}_1$ and $\hat{\theta}_2$.}
    \label{fig:mse_theta2}
\end{figure}

{
\begin{proof}[Proof of \Cref{thm:theta_est}] 
Since $\theta_1^4 > c_1$, we can construct an asymptotically consistent estimator $\hat{\theta}_1$ for $\theta_1$ as in \cite{liu2023asymptotic}, which corrects for the bias in $\sigma_1(X_1)$ due to the noise:
\begin{equation}
    \hat{\theta}_1 = \sqrt{\frac{\sigma_1^2(X_1) - (1+c_1) + \sqrt{(\sigma_1^2(X_1) - (1+c_1))^2 - 4c_1}}{2}}, \label{eq:theta_est_quadratic}
\end{equation}
where $\sigma_1(X_1)$ is the largest singular value of $X_1$.
This follows from \cite{liu2023asymptotic} Theorem 2 part 1, which for $A = \theta u v^\top$ studies $R = AA^\top +I = \theta^2 uu^\top + I$.
This theorem gives that the largest singular value of $X= A + E$, where $E$ satisfies \Cref{assum:general_noise}, satisfies $\sigma_1^2(X) \pto \phi(\gamma)$.
Here, $\gamma$ is the spiked (population) eigenvalue of $R = AA^\top + I$, $\gamma = \theta^2 + 1$, which yields
\begin{equation*}
    \sigma_1^2(X) \pto \phi(\gamma) = \theta^2 + 1 + c + \frac{c}{\theta^2}.
\end{equation*}

Solving this quadratic yields the stated estimator.
By the continuous mapping theorem, this estimator is consistent for $\theta_1$. 
This means that we can construct a consistent estimator $\hat{\beta}_1$ for $\beta_1$ as well, as $\beta_1^2 = \frac{\theta_1^4 - c_1}{\theta_1^4 + \theta_1^2}$.
Analyzing our estimator yields:
\begin{align*}
    \hat{\theta}_2 &= \frac{1}{\hat{\beta}_1} \sqrt{\|X_2 \hat{v}_1\|_2^2 - c_2}\\
    &\pto \theta_2.
\end{align*}
The key step in this proof is leveraging that:
\begin{equation*}
    \left(v^\top\hat{v}_1\right)^2 \pto \beta_1^2,
\end{equation*}
by \Cref{prop:single_table}, and that $\hat{v}_1$ is independent of $E_2$.
Then:
\begin{align*}
    \|X_2 \hat{v}_1\|_2^2 
    &= \left\| \langle \hat{v}_1,v \rangle \theta_2 u_2 + E_2 \hat{v}_1\right\|_2^2 \\
    &= \left( \langle \hat{v}_1,v \rangle\right)^2 \theta_2^2 + 2 \langle \hat{v}_1,v \rangle \theta_2 u_2^\top E_2 \hat{v}_1 + \left\|E_2 \hat{v}_1\right\|_2^2\\
    &\pto \theta_2^2 \beta_1^2 + c_2.
\end{align*}
In the last line, we analyze the 3 terms separately. 
The first term follows directly from \Cref{prop:single_table}.
The second term follows since $E_2$ is independent of $u_2$ and $\hat{v}_1$:
\begin{align*}
    u_2^\top E_2 \hat{v}_1 &= \sum_{i,j} u_{2,i} E_{2, (i,j)} \hat{v}_{1,j}.
\end{align*}
This has mean 0, as entries of $E_2$ are independent of each other and have mean 0 marginally.
We can then compute its variance (denoting $x = u_2, E=E_2,y=\hat{v}_1$):
\begin{align*}
    \text{Var}\left(u_2^\top E_2 \hat{v}_1\right)
    &= \E \left[\left(x^\top E y\right)^2\right]\\
    &= \sum_{ij}\E \left[ x_i^2 E_{ij}^2 y_j^2\right]\\
    &= \left(\sum_i x_i^2\right) \left( \sum_j y_j^2\right) \E \left[  E_{11}^2 \right]\\
    &= \E \left[  E_{11}^2 \right]\\
    &= 1/d.
\end{align*}
By applying the Chebyshev inequality, this gives us that $u_2^\top E_2 \hat{v}_1 \pto 0$. 

\medskip

The final term can be analyzed to show that $\left\|E_2 \hat{v}_1\right\|_2^2\pto c_2$ by leveraging \Cref{assum:general_noise}, that the entries of our noise matrix are drawn i.i.d. according to $z$.
\begin{lemma}\label{lem:noise_projection_concentration}
Let $E_2 \in \R^{n_2 \times d}$ satisfy \Cref{assum:general_noise}, and let
$a \in S^{d-1}$ be independent of $E_2$. Write
$\kappa := \E\left[(\sqrt{d} z)^4\right] \le C$. Then
\begin{equation}
    \E\left[ \|E_2 a\|_2^2 \right] = \frac{n_2}{d},
    \qquad
    \Var\left( \|E_2 a\|_2^2 \right)
    = \frac{2 n_2}{d^2} + \frac{(\kappa - 3)\, n_2}{d^2}\, \E\left[\|a\|_4^4\right].
\end{equation}
In particular $\Var(\|E_2 a\|_2^2) = O(1/d)$, and therefore
$\|E_2 a\|_2^2 \pto c_2$.
\end{lemma}

\begin{proof}
Condition on $a$, and set $W := \|E_2 a\|_2^2 = \sum_{k=1}^{n_2} g_k^2$ with
$g_k := \sum_{\ell=1}^d E_{2,(k,\ell)}\, a_\ell$. The entries of $E_2$ are i.i.d.
with $\E z = 0$, $\E z^2 = 1/d$, and $\E z^4 = \kappa/d^2$, and the $g_k$ are
independent across $k$ since they each represent a different row.

For the mean,
\begin{align}
    \E[g_k^2 \mid a] &= \sum_{\ell=1}^d a_\ell^2\, \E z^2
    = \frac{1}{d}\|a\|_2^2 = \frac{1}{d}\\
    \E[W \mid a] &= \frac{n_2}{d}.
\end{align}
Since this conditional mean does not depend on $a$, the law of total variance
gives $\Var(\E[W \mid a]) = 0$ and hence $\Var(W) = \E[\Var(W \mid a)]$.

For the conditional variance, we expand $g_k^4$. Independence and $\E z = 0$ leave
only the paired terms:
\begin{align}
    \E[g_k^4 \mid a]
    &= \sum_{\ell} a_\ell^4\, \E z^4
       + 3 \sum_{\ell \ne m} a_\ell^2 a_m^2\, (\E z^2)^2 \nonumber \\
    &= \frac{\kappa}{d^2}\|a\|_4^4
       + \frac{3}{d^2}\left(1 - \|a\|_4^4\right)
    = \frac{3}{d^2} + \frac{\kappa - 3}{d^2}\|a\|_4^4,
\end{align}
using $\sum_{\ell \ne m} a_\ell^2 a_m^2 = \|a\|_2^4 - \|a\|_4^4 = 1 - \|a\|_4^4$.
Subtracting $(\E[g_k^2 \mid a])^2 = 1/d^2$,
\begin{align}
    \Var(g_k^2 \mid a)
    &= \frac{2}{d^2} + \frac{\kappa - 3}{d^2}\|a\|_4^4\\
    \Var(W \mid a)
    &= \frac{2 n_2}{d^2} + \frac{(\kappa - 3)\, n_2}{d^2}\|a\|_4^4.
\end{align}
Taking expectation over $a$ yields the stated variance formula. Since
$\|a\|_4^4 \le \|a\|_2^4 = 1$ and $0 \le \kappa \le C$,
\begin{equation}
    \Var(W)
    \le \frac{(2 + |\kappa - 3|)\, n_2}{d^2}
    = \frac{2 + |\kappa - 3|}{d}\cdot\frac{n_2}{d}
    = O(1/d).
\end{equation}
Combined with
$\E[W] = n_2/d \to c_2$, Chebyshev's inequality gives $W \pto c_2$.
\end{proof}

Plugging these back into our expression for $\hat{\theta}_2$, and using the asymptotic consistency of $\hat{\beta}_1$ and the continuous mapping theorem, yields the desired result.
\end{proof}

For consistency of $\hat\theta_2$, we only require $\hat\beta_1 \pto \beta_1$,
which follows from the pointwise convergence $\sigma_1^2(X_1) \pto \phi(\gamma_1)$
(\cite{liu2023asymptotic}, Theorem 2) and the continuous mapping theorem, since
$\hat\beta_1$ is a continuous function of $\sigma_1^2(X_1)$ when $\theta_1^4 > c_1$.

If a rate is desired, note that above the BBP threshold the map
$\sigma_1^2(X_1) \mapsto \hat\beta_1$ is continuously differentiable with bounded
derivative in a neighborhood of $\phi(\gamma_1)$. Hence, one can obtain a rate using results from e.g. \cite{10.3150/19-BEJ1129}, with a constant that degrades as
$\theta_1 \downarrow c_1^{1/4}$.
}

\section{Supplementary Figures and Tables}
In this Appendix, we include figures that help elucidate points in the main text, but do not fit well flow-wise in that location. 

In \Cref{fig:cor_1_plot}, we show the setting of \Cref{thm:simple_thm1}, demonstrating the performance improvement afforded by unweighted \stacksvd when all tables are identical.

In \Cref{fig:expo_rate}, we show that the theoretical results hold even when the noise is i.i.d. exponentially distributed (centered, with mean 0, and scaled so that the variance is $1/d$).
This noise is both a) non-Gaussian, and b) non symmetric, but since it has mean 0, appropriate variance, and has bounded higher moments, it still works in practice, where the asymptotic theory predicted for Gaussian noise holds.
We simulate $d=2000$ with $c_i=1$ and $\theta_i = 1 + 2/(i+1)$ with increasing $M$ and exponential noise in \Cref{fig:expo_rate} (left).
Then, increasing $d$, we simulate $M=3$ with $c_i=c_0=1$ and $\theta=[1.7,1.6,1.5]$ in \Cref{fig:expo_rate} (right). 

In \Cref{fig:tenx_sv} we plot the spectrum of the single cell dataset we utilize \cite{zheng2017massively}. 
This validates the random matrix theory modeling of our matrix.

\begin{figure}[h]
    \centering
    \includegraphics[scale=.8]{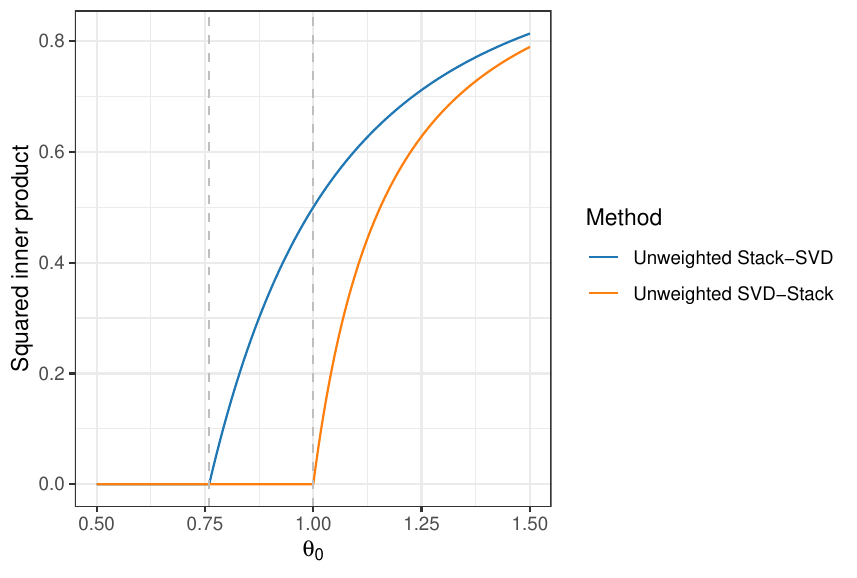}
    \caption{The performance of \stacksvd and \svdstack in the situation where $M = 3$, all $\theta_i = \theta_0$ and all $c_i = 1$, the setting of \Cref{thm:simple_thm1}. The dashed lines represent the phase transition points for \stacksvd ($M^{-1/4}$) and \svdstack ($1$).}
    \label{fig:cor_1_plot}
\end{figure}

\begin{figure}[h]
    \centering
    \includegraphics[scale=0.8]{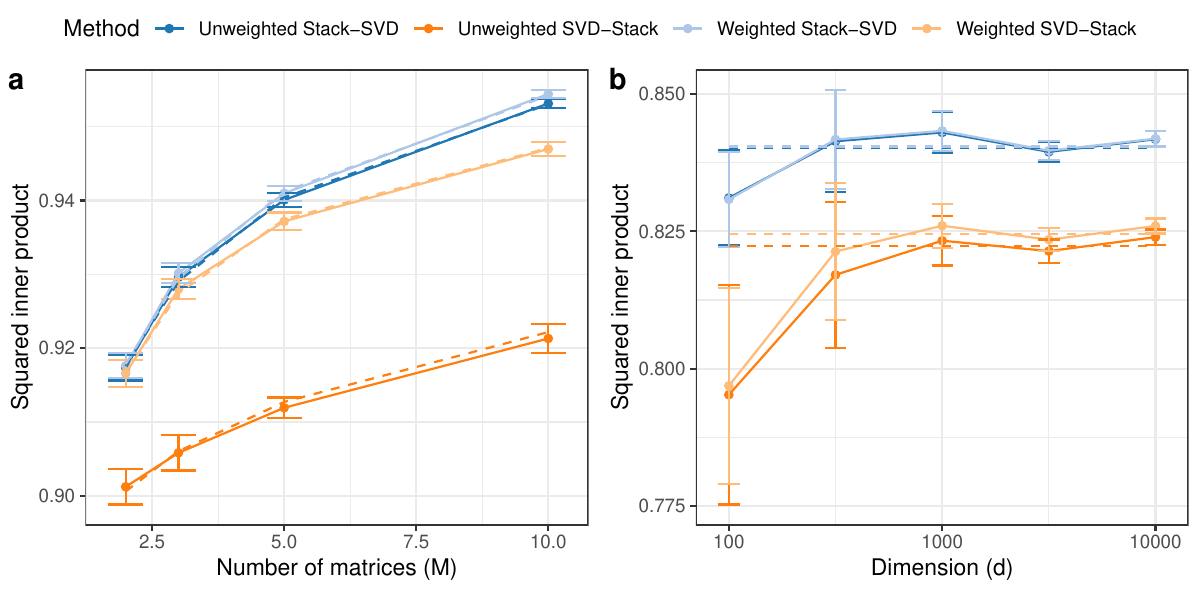}
    \caption{Theoretical results still hold for non-Gaussian, non-symmetric noise. Pictured, the noise is i.i.d. exponentially distributed (centered, with mean 0, and scaled so that the variance is $1/d$). }
    \label{fig:expo_rate}
\end{figure}

\begin{figure}[h]
    \centering
    \includegraphics[width=.7\linewidth]{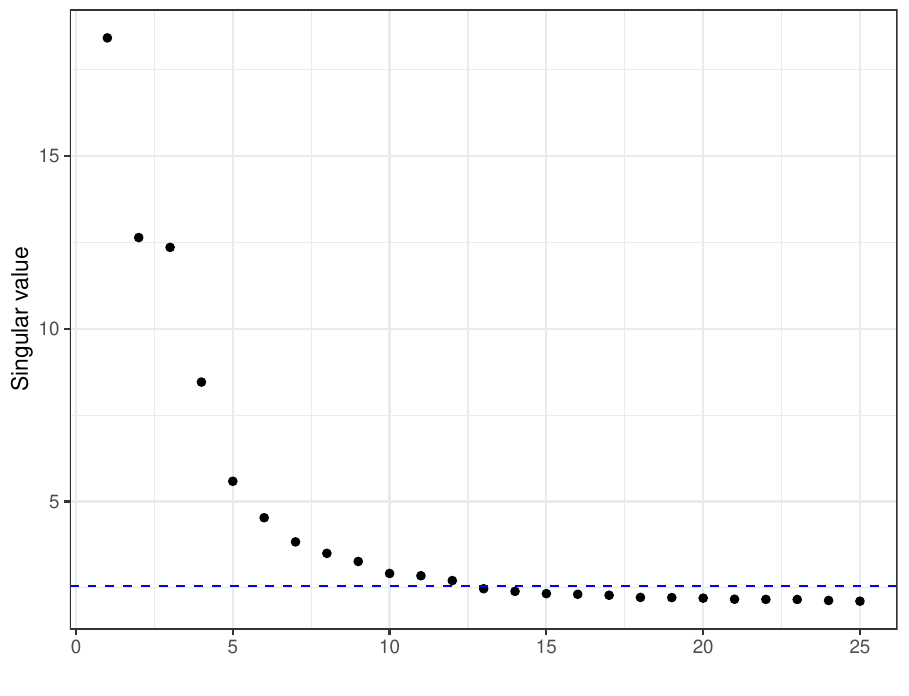}
    \caption{The leading singular values obtained from the transformed counts on the \cite{zheng2017massively} dataset. The blue dashed line indicates $1 + \sqrt{c}$.}
    \label{fig:tenx_sv}
\end{figure}

\section{Code for data analysis and simulations}
All code used to generate the results in this paper is publicly available on Github at \url{https://github.com/phillipnicol/stackedSVD}.
\blue{We used Large Language Models (chiefly Claude Fable 5 and GPT 5.6 Sol) to formalize the main results of this paper in Lean 4 \cite{moura2021lean,mathlib2020}, with no \texttt{sorry} placeholders or additional axioms. Running \texttt{\#print axioms} on each main theorem shows that it depends only on Lean's three standard axioms: propositional extensionality (\texttt{propext}), the axiom of choice (\texttt{Classical.choice}), and quotient soundness (\texttt{Quot.sound}). The authors have verified that the formal statements faithfully reflect the theorems as stated in the paper, and take full responsibility for the content of both.
}

\end{appendix}

\clearpage
\begin{funding}
TZB was supported by funding from the Eric and Wendy Schmidt Center at the Broad Institute of MIT and Harvard.
PBN was supported by the National Institutes of Health grant T32CA009337.
RAI was supported in part by funding from NIH Grants R35GM131802, and R01HG005220. RM was supported in part by funding from NSF Grant DMS-2515684.
\end{funding}

\end{document}